\documentclass{article}

 \usepackage[main, final]{neurips_2026}

\usepackage[utf8]{inputenc} 
\usepackage[T1]{fontenc}    

\usepackage{pifont}

\usepackage{enumitem}

\usepackage{subcaption}

\usepackage{tikz}        
\usepackage{xcolor}     

\newcommand{\circled}[2]{%
  \tikz[baseline=(char.base)]{
    \node[shape=circle,draw=#1,fill=#1!20,inner sep=1.5pt,minimum size=1.2em] (char) {\textbf{\small #2}};
  }%
}
\newcommand{\cOne}{\circled{blue}{1}}     
\newcommand{\cTwo}{\circled{orange}{2}}   
\newcommand{\cI}{\circled{teal}{i}}       
\newcommand{\cII}{\circled{purple}{ii}}   
\newcommand{\cIII}{\circled{red}{iii}}

\newcommand{\tauI}{\circled{cyan}{1}}
\newcommand{\tauII}{\circled{magenta}{2}}
\newcommand{\tauIII}{\circled{magenta!50!black}{3}}

\usepackage{hyperref}

\usepackage{booktabs,multirow}

\usepackage{wrapfig}
\usepackage{url}            
\usepackage{booktabs}       
\usepackage{amsfonts}       
\usepackage{nicefrac}       
\usepackage{microtype}      
\usepackage{xcolor}         

\usepackage{microtype}
\usepackage{graphicx}
\usepackage{subcaption}
\usepackage{booktabs}
\usepackage{hyperref}
\usepackage{booktabs,multirow}
\usepackage{wrapfig}

\usepackage{natbib}
\usepackage{amsmath}
\usepackage{amssymb}
\usepackage{mathtools}
\usepackage{amsthm}

\usepackage{algorithm}
\usepackage{algpseudocode}
\algrenewcommand\algorithmicrequire{\textbf{Input:}}
\algrenewcommand\algorithmicensure{\textbf{Output:}}
\algrenewcommand\algorithmiccomment[1]{\hfill$\triangleright$~#1}
\algnewcommand{\LineComment}[1]{\State \algorithmiccomment{#1}}

\usepackage{thm-restate}
\usepackage{xcolor,colortbl}
\usepackage[capitalize,noabbrev]{cleveref}
\usepackage[textsize=tiny]{todonotes}

\usepackage{xspace}

\newcommand{\best}[1]{\textcolor{red}{#1}}
\newcommand{\second}[1]{\textcolor{blue}{#1}}

\theoremstyle{plain}
\newtheorem{theorem}{Theorem}[section]

\theoremstyle{definition}

\theoremstyle{remark}
\newtheorem{remark}[theorem]{Remark}

\definecolor{scotpurple}{RGB}{89, 67, 156}

\newcommand{\scot}{\textbf{\textcolor{scotpurple}{SCOT}}\xspace}

\title{SCOT: Soft Correspondence via Optimal Transport for Multi-Source Cross-City Transfer}

\author{%
  Yuyao Wang$^{1}$ \quad
  Min Yang$^{2}$ \quad
  Meng Chen$^{2}$ \quad
  Weiming Huang$^{3}$ \quad
  Yilong Yin$^{2}$ \quad
  Yongshun Gong\thanks{Corresponding author.}\, $^{2}$ \\[0.4em]
  $^{1}$Department of Mathematics and Statistics, Boston University, Boston, MA, USA \\
  $^{2}$School of Software, Shandong University, Jinan, China \\
  $^{3}$School of Geography, University of Leeds, Leeds, UK \\
}

\begin{document}

\maketitle

\begin{abstract}
Cross-city transfer uses labeled data from well-instrumented cities to improve prediction in label-scarce ones. The task is difficult when cities have incompatible region partitions and no ground-truth correspondence. Existing methods either rely on heuristic anchors, which can create hubness, or align global distributions, which can over-mix heterogeneous regions. We propose \textbf{\textcolor{scotpurple}{SCOT}} (\textbf{\textcolor{scotpurple}{S}}oft \textbf{\textcolor{scotpurple}{C}}orrespondence via \textbf{\textcolor{scotpurple}{O}}ptimal \textbf{\textcolor{scotpurple}{T}}ransport), an OT-based framework for learning explicit, mass-controlled soft correspondences between unequal region sets. SCOT combines an OT-weighted contrastive loss, a one-sided cycle regularizer, and a shared prototype hub for multi-source transfer. Across real-world cities and tasks, SCOT consistently outperforms strong baselines, reducing MAE and MAPE by 5--50\%. Additional analyses show that the learned correspondences are selective, region-specific, and robust to heterogeneous settings. Code is available at \url{https://github.com/olivia3395/SCOT-Multi-Source-Cross-City-Transfer}.
\end{abstract}

\section{Introduction}

Many urban computing tasks build city-scale predictors from heterogeneous 
data (human mobility, POIs, remote sensing) and rely on high-quality 
region representations for downstream outcomes such as regional GDP, 
population, and carbon estimation 
\citep{an2025spatio,yang2025can,li2024dual}. Since reliable labels exist 
only for a few well-instrumented cities, cross-city transfer aims to 
learn region embeddings that generalize from labeled source cities to a 
label-scarce target 
\citep{li2025adaptive,jin2022selective,yang2023carpg,lu2022stgfsl,
wang2018regiontranas,jin2023transGTR,fang2022transfer,yao2019metast,
li2022localegn}. This is harder than standard domain adaptation: city 
pairs rarely share a natural region correspondence and often have 
unequal region counts (Fig.~\ref{fig:motivation}a), regions are graph 
nodes rather than i.i.d.\ samples, and only part of the semantics 
transfers (commuting corridors generalize while tourist districts are 
often city-specific), so alignment must be local and selective 
\citep{saito2018maximum,chen2024profiling,chen2025mgrl4re,
zhang2023spatial,chencross}.

\textbf{Where does the difficulty actually live?} Modern GNN encoders 
already produce expressive region embeddings, and within-city predictive 
quality is largely saturated by standard GAT-family architectures. What 
varies dramatically across methods is target-city performance under 
transfer---an asymmetry that points to the cross-city correspondence 
mechanism, not the encoder, as the locus of difficulty. We empirically 
confirm this in Section~\ref{sec:robustness}: backbone and readout 
substitutions shift performance only marginally, while different 
alignment designs cause far larger variation. These results suggest that, in this regime, methodological gains can be more efficiently obtained by investing effort in the alignment mechanism than in the encoder.

\textbf{Positioning relative to existing alignment paradigms.} Prior work typically aligns cities along one of two axes \citep{wei2021areatransfer,liu2024frequency,yang2025cross,zhang2025drawing,yang2025stda,yuan2025federated,zhang2025transfer}, each with distinct trade-offs. \emph{Distribution-matching} methods (e.g., MMD~\citep{gretton2012kernel}, adversarial alignment~\citep{ganin2016domain}) match embedding distributions without establishing explicit region-level correspondences, and even when augmented with region-level refinements, the distribution-level signal can encourage embeddings of heterogeneous cities to mix more than the data supports. \emph{Anchor and nearest-neighbor matching} sits at the opposite end, constructing correspondences directly via hard one-to-one assignments, which can be sensitive to unequal region counts and to hubness~\citep{lei2022dtignn,tang2022dastnet,zhao2023dual,bao2022storm,wang2021spatio,chencross}. Recent correspondence-learning approaches narrow this gap, though aspects of both behaviors can persist in cross-city transfer without targeted structural adaptation (Fig.~\ref{fig:tsne_xa2bj_scot_vs_core} illustrates this for CoRE~\citep{chencross}; quantitative diagnostics in Section~\ref{sec:diagnostics}). These two paradigms share a common underexplored ingredient: \emph{explicit, mass-controlled soft correspondences between unequal region sets}---a primitive well-studied in optimal transport but not yet carefully adapted to the structural demands of cross-city transfer. What is needed is an alignment mechanism that (i) learns region-level correspondences without ground-truth matching, and (ii) scales to multiple sources without source domination or gradient conflict.

\begin{figure}[t]
    \centering
    \begin{minipage}[t]{0.48\columnwidth}
        \centering
        \includegraphics[width=\linewidth]{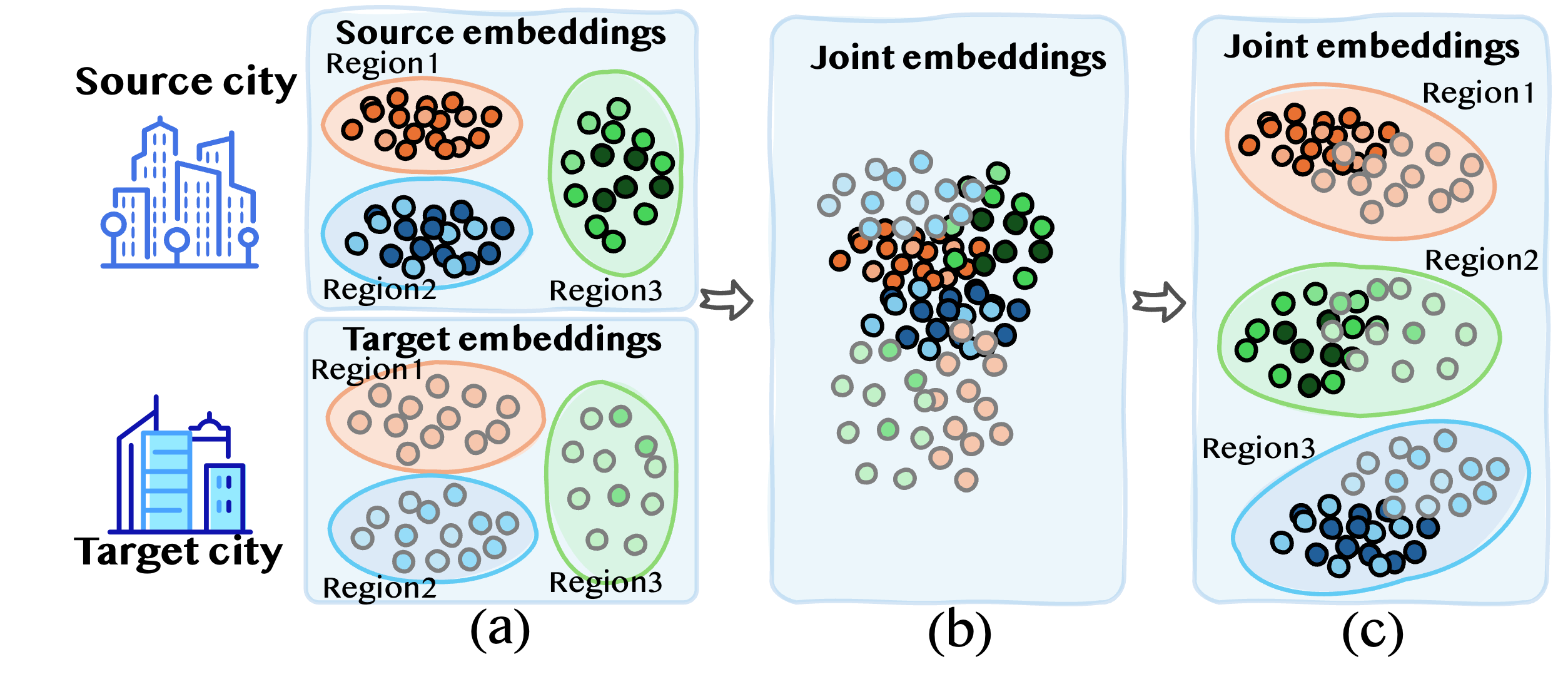}
        \caption{\textbf{Illustration of motivation.}}
        \label{fig:motivation}
    \end{minipage}
    \hfill
    \begin{minipage}[t]{0.48\columnwidth}
        \centering
        \includegraphics[width=\linewidth]{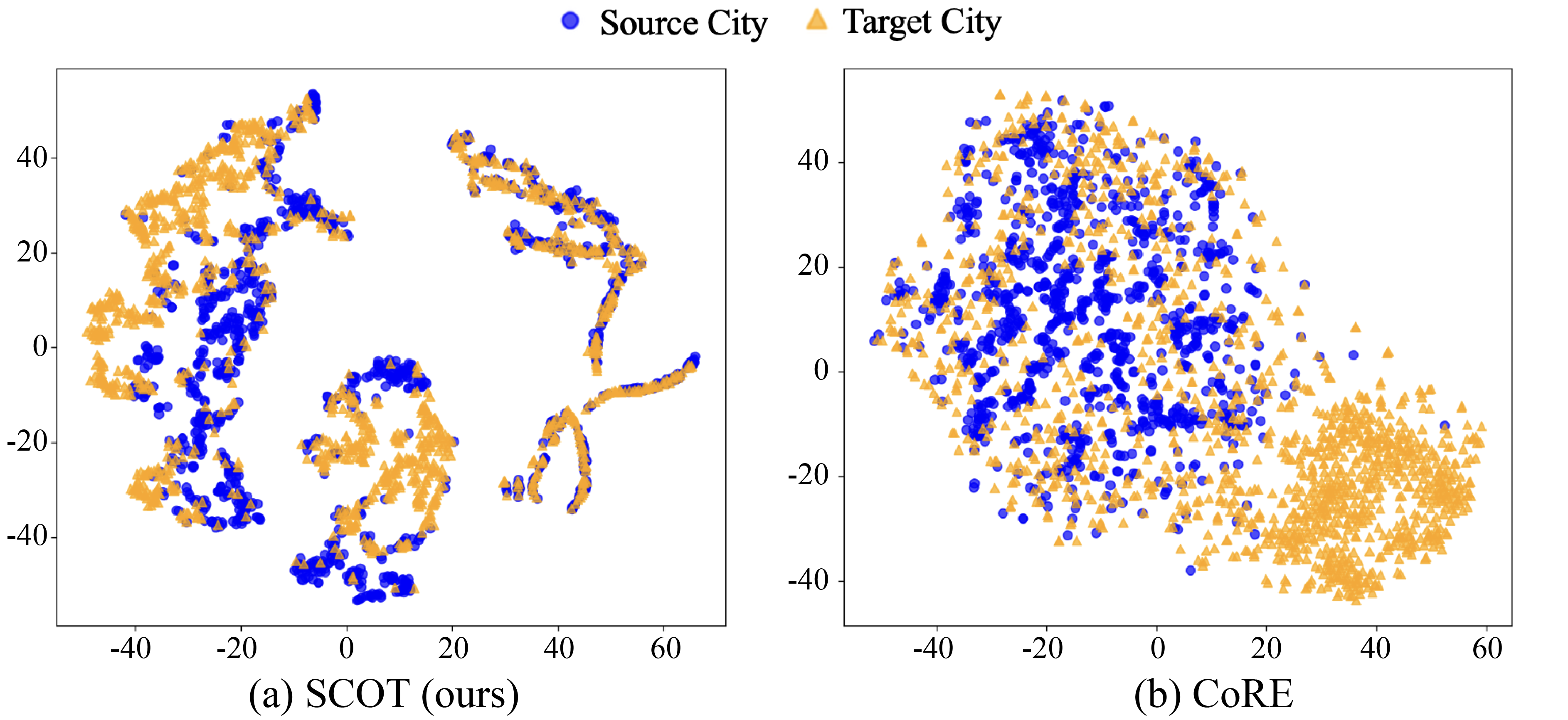}
        \caption{\textbf{t-SNE visualization for XA$\to$BJ.}}
        \label{fig:tsne_xa2bj_scot_vs_core}
    \end{minipage}
    \vspace{-1em}
\end{figure}

To address this unmet requirement, we propose \textbf{SCOT} 
(\textbf{S}emantic \textbf{C}orrespondence via \textbf{O}ptimal 
\textbf{T}ransport), which learns explicit soft correspondences for 
cross-city alignment (Fig.~\ref{fig:main}). \scot rests on two design 
choices, each a principled response to a specific failure mode 
diagnosed above.

\cOne~\textbf{Optimal transport for capacity-controlled soft 
matching.} OT compares distributions via a minimum-cost coupling between point sets~\citep{villani2021topics,villani2008optimal}, and its entropic form~\citep{cuturi2013sinkhorn} imposes marginal capacity constraints that address both failure modes at once: bounding per-region mass prevents the many-to-one collapse of anchor matching, while the resulting many-to-many soft coupling avoids the indiscriminate mixing of distribution-level objectives. Sinkhorn iterations~\citep{peyre2019computational} scale this to urban settings, leveraging OT's established role in domain adaptation and representation learning~\citep{courty2016optimal,chen2020graph,li2024optimal}.

\cTwo~\textbf{OT-guided contrastive sharpening for semantic 
discriminability.} Geometric closeness alone does not yield 
semantically discriminative embeddings. We therefore couple OT with 
a contrastive objective whose soft positives are defined by the 
transport coupling itself: target candidates are weighted by 
transported mass, focusing similarity on transport-supported pairs 
while inheriting OT's capacity 
control~\citep{genevay2018learning}. The result is a locally aligned 
yet non-collapsed embedding geometry that transfers more reliably 
(Fig.~\ref{fig:tsne_xa2bj_scot_vs_core}, left). A one-sided 
cycle reconstruction further stabilizes training under strong 
heterogeneity by enforcing transport-conditioned identity.
\paragraph{Contributions.}
\begin{itemize}

\item \cI~\textbf{Single-source alignment via OT soft correspondence.} 
We propose the \emph{first} cross-city alignment framework that jointly 
enforces capacity-controlled correspondence, semantic discriminability, 
and rectangular-geometry stability, instantiated via Sinkhorn-based 
entropic OT, an OT-weighted contrastive objective, and a one-sided 
cycle regularizer.

\item \cII~\textbf{Multi-source extension through a shared prototype 
hub.} We extend \scot to multi-source 
transfer via a shared prototype hub guided by a target-induced prior 
and balanced OT, coordinating heterogeneous sources without explicit source-selection heuristics and reducing the impact of conflicting per-source signals through hub-mediated aggregation.

\item \cIII~\textbf{Empirical validation.} On GDP, population, and CO$_2$ 
prediction across cities and directions, \scot outperforms strong 
baselines in both single- and multi-source settings, with multi-source 
\scot consistently surpassing the strongest single-source transfer. 
Controlled substitution of encoders and regressors further confirms 
that gains stem from alignment design rather than encoder capacity 
or readout flexibility.
\end{itemize}


\section{Related Work}
\label{app:related_work}

\subsection{Cross-City Transfer}

Cross-city transfer uses well-instrumented source cities to compensate for
label scarcity in a target city. Early work includes multimodal transfer for
urban inference (FLORAL)~\citep{wei2016floral}, match-then-transfer via
region correspondences (RegionTrans)~\citep{wang2018regiontranas}, and
meta-learning across cities (MetaST)~\citep{yao2019metast}. Later methods
address heterogeneity and negative transfer, e.g., by reweighting source
regions (CrossTReS)~\citep{jin2022selective}. Graph-based variants include
spatio-temporal few-shot learning (ST-GFSL)~\citep{lu2022stgfsl},
transferable graph structure learning (TransGTR)~\citep{jin2023transGTR},
parameter generation (CARPG)~\citep{yang2023carpg}, and joint embedding and
alignment (CoRE)~\citep{chencross}. Many of these methods assume comparable
spatial units such as grids. With irregular partitions, however, cities have
unequal region counts and no natural one-to-one correspondence, which calls
for local, selective alignment between region sets---the setting \scot
targets.

\subsection{Spatio-Temporal Representation Learning}

Spatio-temporal models learn urban embeddings for traffic forecasting, crowd
flow and OD estimation~\citep{mu2025gem,yuan2025spatio}, weather
prediction~\citep{gong2024spatio}, and regional attribute regression. On
regular grids, ST-ResNet~\citep{zhang2017deep}, DeepSTN+~\citep{lin2019deepstn+},
and UrbanFM~\citep{liang2019urbanfm} model citywide flows. On graphs,
STGNNs such as DCRNN~\citep{li2017diffusion}, STGCN~\citep{yu2017spatio},
Graph WaveNet~\citep{wu2019graph}, and GMAN~\citep{zheng2020gman} combine
spatial message passing with temporal modules, and recent work adds
self-supervised pretraining~\citep{liu2022contrastive,ji2023spatio}. These
models are designed for a single city or a fixed node set. Across cities
with different partitions, a stronger encoder alone does not ensure
transferability, so an explicit alignment mechanism is still needed.

\subsection{Optimal Transport in Deep Learning}

Optimal transport (OT) aligns distributions through a minimum-cost coupling;
entropic regularization and the Sinkhorn algorithm make it scalable and
differentiable~\citep{cuturi2013sinkhorn,peyre2019computational}. OT serves
as a geometry-aware loss~\citep{frogner2015learning,genevay2018learning,feydy2019interpolating}
and as an alignment tool in domain adaptation, from OT-DA~\citep{courty2016optimal}
to deep variants such as DeepJDOT~\citep{damodaran2018deepjdot}.

\paragraph{OT variants.}
Unbalanced OT replaces hard marginals with soft penalties~\citep{chizat2018scaling,fatras2021unbalanced},
partial OT transports only part of the mass~\citep{caffarelli2010free,chapel2020partial},
and quadratic regularization yields sparse couplings~\citep{blondel2018smooth}.
\scot uses balanced entropic OT because its strict marginals are what
prevent many-to-one concentration (Appendix~\ref{app:ot_variants}).

\paragraph{Graph-structural OT.}
Gromov--Wasserstein~\citep{memoli2011gromov,peyre2016gromov} and fused
GW~\citep{titouan2019optimal} align graphs by preserving intra-domain
distances. We instead apply OT to GAT embeddings, which already encode graph
structure; this avoids GW's non-convex optimization and higher cost
(Appendix~\ref{app:cost_variants}).

\section{Problem Setup}
\label{sec:problem}

We study cross-city transfer between a labeled source $\mathcal C_s$ and 
a label-scarce target $\mathcal C_t$, with region sets 
$V_s=\{1,\ldots,n_s\}$ and $V_t=\{1,\ldots,n_t\}$. Each city $c$ is 
equipped with an undirected spatial adjacency $\mathbf A_c$ and a 
row-stochastic mobility matrix $\mathbf M_c$ derived from OD trips:
\begin{equation}
M_{ij}=\frac{\mathrm{count}(i\!\to\! j)}{\sum_{k}\mathrm{count}(i\!\to\! k)},
\end{equation}
so $M_{i\cdot}$ encodes the destination distribution from region $i$. 
Our goal is to learn embeddings 
$\mathbf z_s\in\mathbb R^{n_s\times d}$, 
$\mathbf z_t\in\mathbb R^{n_t\times d}$ that preserve intra-city 
mobility and remain comparable across cities despite $n_s\neq n_t$, 
without any node correspondence (Section~\ref{sec:method}).

\section{Method}
\label{sec:method}

\begin{figure*}[t] 
    \centering 
    \includegraphics[width=1.0\textwidth]{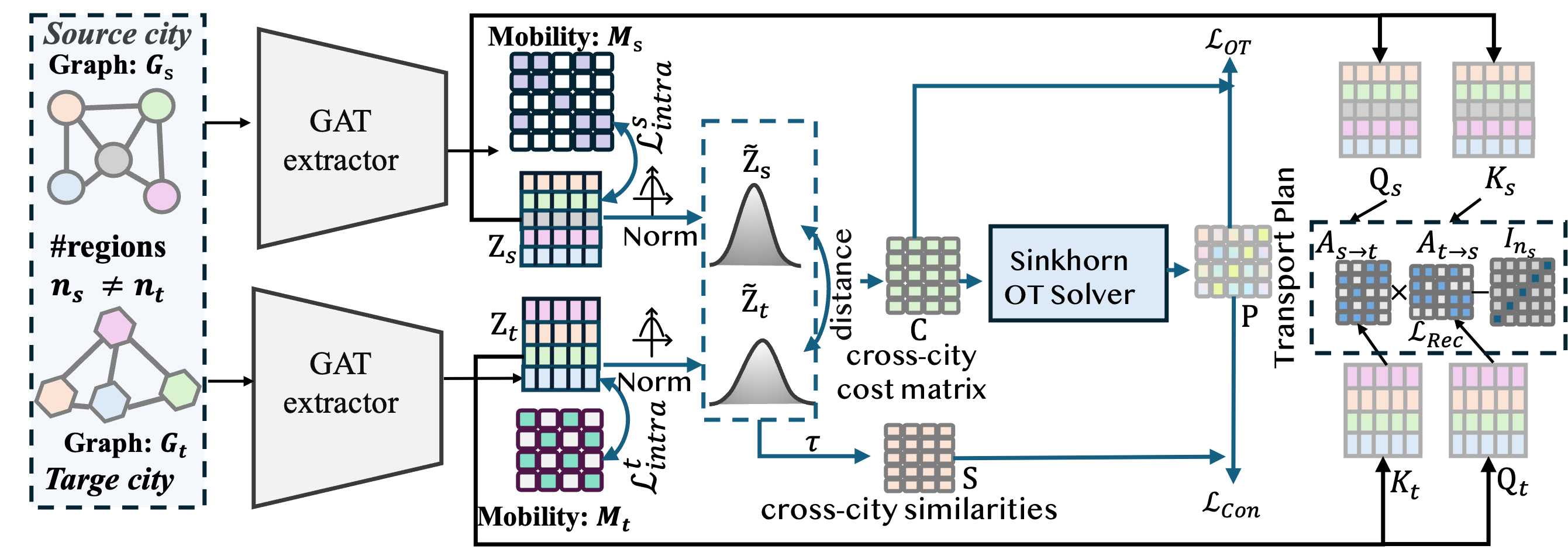} \caption{The pipeline of SCOT.} 
    \label{fig:main} 
    \vspace{-1.6em} 
\end{figure*}

\scot jointly learns mobility-preserving region embeddings and a soft
correspondence between two unequal region sets, without node-level
correspondence (Fig.~\ref{fig:main}, Alg.~\ref{alg:scot_single_sinkhorn}).

\paragraph{Encoder and intra-city objective.}
Because performance depends far more on the alignment design than on the
encoder (Section~\ref{sec:robustness}), we use a standard GAT
encoder~\citep{velivckovic2017graph}. For each city $c\in\{s,t\}$, learnable
initial embeddings $\mathbf H_c^{(0)}\in\mathbb R^{n_c\times d}$ pass through
$L$ GAT layers over the spatial adjacency $\mathbf A_c$, giving
$\mathbf z_c=\mathbf H_c^{(L)}$. We model the destination distribution of
each origin $i$ as
\begin{equation}
\hat{P}^{(c)}_{ij}=
\frac{\exp(\mathbf{z}_{c,i}^{\top}\mathbf{z}_{c,j})}
{\sum_{k=1}^{n_c}\exp(\mathbf{z}_{c,i}^{\top}\mathbf{z}_{c,k})},
\end{equation}
and fit it to the empirical transition matrix $\mathbf{M}_c$:
\begin{equation}
\mathcal{L}^{(c)}_{\mathrm{intra}}
=
-\sum_{i,j}(\mathbf{M}_c)_{ij}\log \hat{P}^{(c)}_{ij}.
\label{eq:lintra}
\end{equation}

\subsection{Alignment via Soft Transport-Guided Matching}
\label{sec:alignment}

Cross-city alignment has no ground-truth correspondence, unequal region
counts, and only partially shared semantics. One-to-one anchor matching is
fragile under unequal partitions, and distribution-level alignment can mix
regions that should stay distinct. We therefore model the correspondence as
an OT coupling $\mathbf{P}\in\mathbb{R}_+^{n_s\times n_t}$, whose marginal
constraints limit how much mass any region can send or receive. Unlike prior
uses of OT as a discrepancy, matching step, or
regularizer~\citep{courty2016optimal,damodaran2018deepjdot,genevay2018learning},
\scot feeds $\mathbf{P}$ directly into a contrastive loss, so the coupling
shapes representation learning.

\subsubsection{Sinkhorn-Based Soft Correspondence}

We $\ell_2$-normalize the embeddings,
$\tilde{\mathbf{z}}^{s}_{i}=\mathbf{z}^{s}_{i}/\lVert\mathbf{z}^{s}_{i}\rVert_2$ and
$\tilde{\mathbf{z}}^{t}_{j}=\mathbf{z}^{t}_{j}/\lVert\mathbf{z}^{t}_{j}\rVert_2$,
so that the cost reflects direction rather than city-specific scale, and set
\begin{equation}
C_{ij} = \big\lVert \tilde{\mathbf{z}}^{s}_{i} - \tilde{\mathbf{z}}^{t}_{j} \big\rVert_2,
\qquad
\mathbf{C}\in\mathbb{R}^{n_s\times n_t}.
\end{equation}
With Gibbs kernel $\mathbf{K}=\exp(-\mathbf{C}/\varepsilon)$, uniform
marginals $\mathbf{a}=\tfrac{1}{n_s}\mathbf{1}$ and
$\mathbf{b}=\tfrac{1}{n_t}\mathbf{1}$, and
$\mathbf{u}^{(0)}=\mathbf{v}^{(0)}=\mathbf{1}$, we run $T$ Sinkhorn
iterations~\citep{cuturi2013sinkhorn}:
\begin{equation}
\label{eq:sinkhorn_iter}
\mathbf{u}^{(k+1)}
=
\mathbf{a}\oslash\big(\mathbf{K}\mathbf{v}^{(k)}\big),
\qquad
\mathbf{v}^{(k+1)}
=
\mathbf{b}\oslash\big(\mathbf{K}^{\top}\mathbf{u}^{(k+1)}\big),
\end{equation}
where $\oslash$ is elementwise division, and obtain the differentiable
coupling
\begin{equation}
\label{eq:sinkhorn_P}
\mathbf{P}
= \mathrm{diag}\big(\mathbf{u}^{(T)}\big)\,\mathbf{K}\,
\mathrm{diag}\big(\mathbf{v}^{(T)}\big).
\end{equation}
Smaller $\varepsilon$ gives sharper but less stable couplings; we use
$\varepsilon=0.15$ (Section~\ref{sec:hp_sensitivity}). The OT loss is the
expected transport cost:
\begin{equation}
\label{eq:lot}
\mathcal{L}_{\mathrm{OT}}
= \frac{1}{\min(n_s,n_t)}
\sum_{i=1}^{n_s}\sum_{j=1}^{n_t} P_{ij} C_{ij}.
\end{equation}

\subsubsection{OT-Weighted Contrastive Alignment}
\label{sec:contrastive}

$\mathcal{L}_{\mathrm{OT}}$ makes coupled pairs close in absolute terms, but
a target region may be equally close to several sources. We therefore add a
contrastive loss that makes pairs supported by $\mathbf{P}$ more similar than
competing pairs. With similarities
$S_{ij}=\tilde{\mathbf{z}}^{s\top}_{i}\tilde{\mathbf{z}}^{t}_{j}/\tau$ and
row $i$ of $\mathbf{P}$ as soft positive weights,
\begin{equation}
\label{eq:lcon}
\mathcal{L}_{\mathrm{Con}}
=
-\frac{1}{n_s}\sum_{i=1}^{n_s}
\log
\frac{\sum_{j=1}^{n_t} P_{ij}\exp(S_{ij})}
{\sum_{j=1}^{n_t}\exp(S_{ij})}.
\end{equation}
This pulls each source region toward its high-mass targets and away from the
rest. The following transfer bound links $\mathcal{L}_{\mathrm{Con}}$ to
target error.

\begin{restatable}[Contrastive alignment as a transfer surrogate]
{proposition}{contrastivemaeprop}
\label{prop:contrastive_mae}
Let $\{u_i\}_{i=1}^{n_s}, \{v_j\}_{j=1}^{n_t}\subset\mathbb{S}^{d-1}$
be the normalized source and target embeddings, with marginals
$a\in\Delta^{n_s}$, $b\in\Delta^{n_t}$, and coupling
$P\in\mathbb{R}_+^{n_s\times n_t}$ satisfying
$P\mathbf{1}=a$ and $P^\top\mathbf{1}=b$.
Let $g$ be the unknown region-level task function and $h$ a source-trained
predictor applied to the target. Assume that $g$ and $h$ are
$L_g$- and $L_h$-Lipschitz, and define
\[
\mathcal{R}_s^a(h):=\sum_i a_i|h(u_i)-g(u_i)|,
\qquad
\mathcal{R}_t^b(h):=\sum_j b_j|h(v_j)-g(v_j)|.
\]
Then
\[
\mathcal{R}_t^b(h)
\le
\underbrace{\mathcal{R}_s^a(h)}_{\text{source risk}}
+
\underbrace{(L_h+L_g)\sqrt{2-2\underline{m}}}_{\text{transfer gap}},
\]
where
\[
\underline{m}
:=
\max\left\{
-1,\,
\tau\log n_t+\tau H(a)
-\tau\mathcal{L}_{\mathrm{Con}}(P)
-1-\frac{1}{2\tau}
\right\},
\qquad
H(a):=-\sum_i a_i\log a_i.
\]
\end{restatable}

With uniform marginals, the two risks are the source and target MAE. A
smaller $\mathcal{L}_{\mathrm{Con}}$ raises $\underline{m}$ and tightens the
transfer gap. The bound motivates the contrastive term; it does not
characterize the full \scot objective. The proof is in
Appendix~\ref{app:contrastive_mae_proof}, and an empirical check against
target MAE is in Appendix~\ref{app:theory_empirical_check}.

\subsubsection{Alignment Loss}

We combine the two terms as
\begin{equation}
\label{eq:align}
\mathcal{L}_{\mathrm{Align}}
=
\mathcal{L}_{\mathrm{OT}}
+
\eta\,\mathcal{L}_{\mathrm{Con}}.
\end{equation}
$\mathcal{L}_{\mathrm{OT}}$ controls global mass allocation, and
$\mathcal{L}_{\mathrm{Con}}$ sharpens local discrimination among candidates.

\subsection{Cycle Reconstruction Regularization}
\label{sec:cycle_regularization}

We regularize the correspondence with a one-sided cross-attention cycle. Its
goal is not an exact inverse between unequal region sets, but to limit
distortion when representations are mapped to the other city and back. With
shared projections $\mathbf{W}_q,\mathbf{W}_k\in\mathbb{R}^{d\times d}$, the
cross-attention maps are
\begin{equation}
\mathbf{A}_{s\to t}
=
\operatorname{softmax}\!\left(
\frac{(\mathbf{Z}_s\mathbf{W}_q^\top)(\mathbf{Z}_t\mathbf{W}_k^\top)^\top}{\sqrt d}
\right),
\qquad
\mathbf{A}_{t\to s}
=
\operatorname{softmax}\!\left(
\frac{(\mathbf{Z}_t\mathbf{W}_q^\top)(\mathbf{Z}_s\mathbf{W}_k^\top)^\top}{\sqrt d}
\right),
\end{equation}
and we penalize the source-returning cycle:
\begin{equation}
\mathcal{L}_{\mathrm{cyc}}
=
\left\|
\mathbf{A}_{s\to t}\mathbf{A}_{t\to s}
-
\mathbf{I}_{n_s}
\right\|_F^2.
\end{equation}
We cycle back to the source because labels and the downstream predictor live
there; the target-returning cycle does not improve target MAE
(Appendix~\ref{app:cycle_direction}). An entropy penalty discourages diffuse
forward attention:
\begin{equation}
\mathcal{R}_{\mathrm{ent}}
=
-\frac{1}{n_s}
\sum_{i,j}
A_{s\to t}(i,j)
\log\!\left(A_{s\to t}(i,j)+\delta\right),
\qquad
\mathcal{L}_{\mathrm{Rec}}
=
\mathcal{L}_{\mathrm{cyc}}
+
\beta\,\mathcal{R}_{\mathrm{ent}},
\end{equation}
with $\delta=10^{-8}$.

\subsection{Model Training}

\scot is trained end-to-end with
\begin{equation}
\label{eq:total}
\mathcal{L}_{\mathrm{Total}}
=
\underbrace{
\mathcal{L}^{(s)}_{\mathrm{intra}}
+
\mathcal{L}^{(t)}_{\mathrm{intra}}
}_{\text{intra-city reconstruction}}
+
\lambda_{\mathrm{align}}
\underbrace{
\mathcal{L}_{\mathrm{Align}}
}_{\text{cross-city alignment}}
+
\lambda_{\mathrm{rec}}
\underbrace{
\mathcal{L}_{\mathrm{Rec}}
}_{\text{cycle regularization}}.
\end{equation}
The intra-city terms preserve each city's mobility structure,
$\mathcal{L}_{\mathrm{Align}}$ learns the cross-city correspondence, and
$\mathcal{L}_{\mathrm{Rec}}$ regularizes it. All weights are fixed across
city pairs and tasks (Section~\ref{sec:experiments}).

\section{Multi-Source Hub Alignment}
\label{sec:ms-hub}

Multi-source transfer is fundamentally harder than the single-source case: 
different sources induce \emph{conflicting correspondences} to the same target, 
and independent pairwise alignments are easily destabilized or dominated by 
one source. We resolve both pathologies through a \textbf{shared semantic hub}---
a set of $K$ learnable prototypes that serve as intermediate anchors in a 
common alignment space (Fig.~\ref{fig:multi}, Alg.~\ref{alg:scot_multi_hub_ot_single_style}).

Instead of aligning each source to the target separately, all cities (sources 
and target) are aligned to the hub via balanced entropic OT, yielding a 
\emph{coordinated many-to-hub matching}. A target-induced prototype marginal 
controls hub capacity and steers prototypes toward target-relevant semantics, 
simultaneously stabilizing optimization and preventing source domination. 
The resulting transport plans $\{\Pi^{(m)}\}$ jointly supervise 
$\mathcal{L}^{m}_{\mathrm{OT}}$ and $\mathcal{L}^{m}_{\mathrm{Con}}$, enabling scalable alignment that avoids the source-domination and conflict pathologies observed under independent pairwise OT.

\paragraph{Balanced entropic OT to the hub.}
Let $\mathcal{S}=\{s_1,\dots,s_M\}$ denote the source cities and $t$ the 
target. For each city $m\in\mathcal{S}\cup\{t\}$, we $\ell_2$-normalize 
region embeddings and prototypes 
($\tilde{\mathbf{z}}^m_i=\mathbf{z}^m_i/\|\mathbf{z}^m_i\|_2$, 
$\tilde{\mathbf{a}}_k=\mathbf{a}_k/\|\mathbf{a}_k\|_2$) and form the 
hub cost $C^m_{ik}=\|\tilde{\mathbf{z}}^m_i-\tilde{\mathbf{a}}_k\|_2$, 
$\mathbf{C}^m\in\mathbb{R}^{n_m\times K}$.

\paragraph{Target-induced prototype marginal.}
The shared prototype marginal $\mathbf{b}\in\Delta^{K-1}$ is constructed 
from target--prototype affinities:
\begin{equation}
\label{eq:bt_from_target}
\bar{s}_k=\tfrac{1}{n_t}\sum_{j=1}^{n_t}\tilde{\mathbf{z}}^{t\top}_j\tilde{\mathbf{a}}_k,
\qquad
b_k \propto \max\!\big\{\exp(\bar{s}_k/\tau_b),\,\epsilon_b\big\},
\end{equation}
normalized to $\sum_k b_k=1$, where $\tau_b>0$ controls sharpness and 
$\epsilon_b>0$ prevents dead prototypes 
(Appendix~\ref{subsec:ablationA_bt}). City-side marginals are uniform: 
$\mathbf{a}^m=\tfrac{1}{n_m}\mathbf{1}$.

\paragraph{Balanced entropic coupling.}
For each $m\in\mathcal{S}\cup\{t\}$, we solve
\begin{equation}
\label{eq:hub_balanced_ot}
\mathbf{\Pi}^m
\in \arg\min_{\mathbf{P}\ge 0}\;
\langle \mathbf{P},\mathbf{C}^m\rangle
+\varepsilon\textstyle\sum_{i,k} P_{ik}(\log P_{ik}-1)
\quad\text{s.t.}\quad
\mathbf{P}\mathbf{1}=\mathbf{a}^m,\;
\mathbf{P}^\top\mathbf{1}=\mathbf{b},
\end{equation}
via $T$ Sinkhorn iterations, and obtain row-normalized assignments 
$Q^m_{ik}=\Pi^m_{ik}/\sum_{k'}\Pi^m_{ik'}$.

\paragraph{OT-guided contrastive alignment to the hub.}
For each $m\in\mathcal S\cup\{t\}$, we compute region--prototype 
similarities $S^m_{ik}=\tilde{\mathbf{z}}^{m\top}_{i}\tilde{\mathbf{a}}_{k}/\tau$ 
and use the OT-induced assignments $Q^m_{ik}$ as soft positive weights:
\begin{equation}
\label{eq:hub_contrastive}
\mathcal{L}_{\mathrm{Con}}^m
= -\frac{1}{n_m}\sum_{i=1}^{n_m}
\log
\frac{\sum_{k} Q^m_{ik}\exp(S^m_{ik})}
{\sum_{k}\exp(S^m_{ik})}.
\end{equation}
With the transport cost 
$\mathcal{L}_{\mathrm{OT}}^m=\langle \mathbf{\Pi}^m,\mathbf{C}^m\rangle$, 
the per-city alignment loss is 
$\mathcal{L}_{\mathrm{Align}}^m = \mathcal{L}_{\mathrm{OT}}^m + \lambda_c\mathcal{L}_{\mathrm{Con}}^m$.

\paragraph{Hub-cycle stabilization.}
We extend the one-sided cycle of Section~\ref{sec:alignment} to the 
city-to-hub setting. With shared $\mathbf{W}_q,\mathbf{W}_k\in\mathbb R^{d\times d}$, 
we form cross-attention maps between region embeddings 
$\mathbf{Z}^m\in\mathbb{R}^{n_m\times d}$ and prototypes 
$\mathbf{A}\in\mathbb{R}^{K\times d}$:
\begin{equation}
\mathbf{A}_{m\to h}=\mathrm{softmax}\!\left(\tfrac{(\mathbf{Z}^m\mathbf{W}_q)(\mathbf{A}\mathbf{W}_k)^\top}{\sqrt d}\right),
\quad
\mathbf{A}_{h\to m}=\mathrm{softmax}\!\left(\tfrac{(\mathbf{A}\mathbf{W}_q)(\mathbf{Z}^m\mathbf{W}_k)^\top}{\sqrt d}\right).
\end{equation}
The hub-cycle loss and entropy penalty (with floor $\delta=10^{-8}$) are
\begin{equation}
\mathcal{L}_{\mathrm{cyc}}^m
= \tfrac{1}{n_m^2}\big\|\mathbf{A}_{m\to h}\mathbf{A}_{h\to m}-\mathbf{I}_{n_m}\big\|_F^2,
\quad
\mathcal{R}_{\mathrm{ent}}^m
= -\tfrac{1}{n_m}\sum_{i,k}\mathbf{A}_{m\to h}(i,k)\log\!\big(\mathbf{A}_{m\to h}(i,k)+\delta\big),
\end{equation}
giving $\mathcal{L}_{\mathrm{Rec}}^m = \mathcal{L}_{\mathrm{cyc}}^m + \beta\,\mathcal{R}_{\mathrm{ent}}^m$.

\paragraph{Objective.}
Multi-source \scot is trained end-to-end by minimizing
\begin{equation}
\label{eq:hub-obj}
\mathcal{L}
= \sum_{m\in\mathcal{S}\cup\{t\}}\mathcal{L}^{m}_{\mathrm{intra}}
+ \frac{1}{|\mathcal{S}|+1}\sum_{m\in\mathcal{S}\cup\{t\}}
\big(\lambda_{\mathrm{align}}\mathcal{L}_{\mathrm{Align}}^m
+ \lambda_{\mathrm{rec}}\mathcal{L}_{\mathrm{Rec}}^m\big).
\end{equation}

\section{Experiments}
\label{sec:experiments}

\paragraph{Setup, tasks, and metrics.}
We evaluate \scot on mobility data from three Chinese cities 
(Beijing, Xi'an, Chengdu), aggregating anonymized OD trips into 
region-level mobility graphs and evaluating transfer on all ordered 
city pairs. Cross-country generalization with New York City is 
reported in Appendix~\ref{app:nyc_transfer}. For each pair $X\!\to\!Y$, a downstream regressor is fit 
on $(\mathbf{Z}^X, \mathbf{y}^X)$ and applied directly to 
$\mathbf{Z}^Y$ to predict $\mathbf{y}^Y$; we use ridge by default 
and verify in Section~\ref{sec:robustness} that Lasso, SVR, and 
Elastic Net yield comparable results. For two-source experiments, 
each baseline adaptively weights the two transfer directions via 
softmax (with MMD and Adv additionally using a joint-mixture 
variant), and a single predictor is trained on the union of 
labeled source regions. We report MAE and MAPE on GDP, population, 
and CO$_2$ (lower is better).

\paragraph{Implementation.}
A two-layer GAT encoder ($d=128$, $H=8$) with PReLU and a linear 
output layer, trained end-to-end with Adam (lr $=10^{-3}$). 
Hyperparameters are fixed to a single configuration across all 
city pairs and tasks: $\lambda_{\mathrm{align}}=1.0$, 
$\lambda_{\mathrm{rec}}=0.5$, $\eta=0.5$, $\beta=0.05$, $\tau=0.1$, 
$\varepsilon=0.15$. Multi-source additionally uses $K=32$ 
prototypes, $\tau_b=0.5$, and probability floor $10^{-3}$.

\paragraph{Baselines.}
We compare \scot against three paradigms: a \emph{non-alignment} 
baseline (intra-city objectives only); \emph{correspondence-based} 
alignment via surrogate matches (RP, HBP, HSA); and 
\emph{correspondence-free} transfer via distributional or 
relational alignment (MMD~\citep{gretton2012kernel}, 
Adv~\citep{ganin2016domain}, CrossTReS~\citep{jin2022selective}, 
CoRE~\citep{chencross}). Details in Appendix~\ref{app:data}.

\begin{table*}[t]
\centering
\caption{Single-source transfer results on XA$\leftrightarrow$BJ. 
Lower is better. \best{Red}: best; \second{Blue}: runner-up. 
Gain row reports \scot's relative improvement over the strongest baseline.}
\label{tab:xa_bj}
\setlength{\tabcolsep}{5pt}
\renewcommand{\arraystretch}{1.2}
\resizebox{\textwidth}{!}{%
\begin{tabular}{@{}l|cccccc|cccccc@{}}
\toprule
\multirow{3}{*}{\textbf{Method}}
& \multicolumn{6}{c|}{\textbf{XA $\to$ BJ}}
& \multicolumn{6}{c}{\textbf{BJ $\to$ XA}} \\
\cmidrule(lr){2-7} \cmidrule(lr){8-13}
& \multicolumn{2}{c}{GDP} & \multicolumn{2}{c}{Population} & \multicolumn{2}{c|}{CO$_2$}
& \multicolumn{2}{c}{GDP} & \multicolumn{2}{c}{Population} & \multicolumn{2}{c}{CO$_2$} \\
\cmidrule(lr){2-3} \cmidrule(lr){4-5} \cmidrule(lr){6-7}
\cmidrule(lr){8-9} \cmidrule(lr){10-11} \cmidrule(lr){12-13}
& MAE & MAPE & MAE & MAPE & MAE & MAPE
& MAE & MAPE & MAE & MAPE & MAE & MAPE \\
\midrule
Non-Alignment
& 264.30 & 12.08 & 981.07 & 8.56 & 288.41 & 6.40
& 252.91 & 5.84 & 946.93 & 8.53 & 270.42 & 8.66 \\
\rowcolor{gray!06}
RP
& 189.84 & 9.07 & 684.50 & 6.55 & 196.05 & 4.70
& 181.08 & 4.58 & 670.44 & 6.46 & 191.85 & 6.42 \\
HBP
& 177.69 & 8.40 & 665.46 & 5.95 & 188.72 & 4.18
& 196.20 & 3.10 & 627.46 & 4.37 & 176.35 & 4.25 \\
\rowcolor{gray!06}
HSA
& 201.27 & 6.83 & 619.33 & 4.00 & 176.40 & 3.20
& 188.40 & 2.11 & 636.33 & 5.03 & 182.91 & 5.02 \\
MMD
& 183.32 & 5.71 & \second{588.34} & \second{3.17} & \second{165.70} & \second{2.54}
& 180.73 & 1.99 & \second{499.94} & \second{1.85} & \second{141.57} & \second{1.83} \\
\rowcolor{gray!06}
Adv
& 192.59 & 8.72 & 702.19 & 6.78 & 199.23 & 4.83
& 199.21 & 6.32 & 805.01 & 9.16 & 203.27 & 7.21 \\
CrossTReS
& 207.43 & 7.39 & 633.25 & 4.42 & 179.75 & 3.50
& 170.27 & 4.23 & 639.44 & 5.62 & 182.72 & 5.55 \\
\rowcolor{gray!06}
CoRE
& \second{157.83} & \second{5.46} & 611.18 & 4.05 & 166.28 & 2.95
& \second{162.19} & \second{1.91} & 547.74 & 2.17 & 153.63 & 2.09 \\
\midrule
\rowcolor{red!8}
\textbf{\scot (Ours)}
& \best{115.33} & \best{3.17} & \best{528.50} & \best{2.13} & \best{149.42} & \best{1.79}
& \best{154.92} & \best{1.60} & \best{452.67} & \best{1.58} & \best{128.74} & \best{1.63} \\
\rowcolor{blue!6}
\textit{$\Delta$ vs.\ best baseline}
& \cellcolor{blue!18}\textbf{+26.9\%} & \cellcolor{blue!22}\textbf{+41.9\%}
& \cellcolor{blue!10}+10.2\% & \cellcolor{blue!18}\textbf{+32.8\%}
& \cellcolor{blue!10}+9.8\% & \cellcolor{blue!18}\textbf{+29.5\%}
& \cellcolor{blue!6}+4.5\% & \cellcolor{blue!12}+15.7\%
& \cellcolor{blue!10}+9.5\% & \cellcolor{blue!12}+14.6\%
& \cellcolor{blue!10}+9.1\% & \cellcolor{blue!10}+10.9\% \\
\bottomrule
\vspace{-2em}
\end{tabular}}
\end{table*}


\subsection{Experimental Results}

\cI~\textbf{Single-source transfer.} \scot achieves the best results 
across all target cities and tasks under both MAE and MAPE 
(Table~\ref{tab:xa_bj}, 
Appendix~\ref{app:additional_single_source}). \cII~\textbf{Multi-source transfer.} In the two-source setting 
(Table~\ref{tab:multi_source_main}), \scot again attains the best 
performance, with gains from the shared hub stabilizing multi-source 
aggregation. \cIII~\textbf{Multi-source vs.\ single-source.} Multi-source \scot
consistently outperforms its best single-source counterpart on most 
target--task pairs (Fig.~\ref{fig:scot_single_vs_multi_mae}), reflecting complementary signals aggregated through the hub rather 
than reliance on a single closest source. The marginal Xi'an GDP 
degradation is mechanistically diagnosed via hub statistics in 
Appendix~\ref{app:source_quality}.




\begin{table*}[!t]
\centering
\caption{Multi-source cross-city transfer results (two targets shown; 
the remaining target is reported in Appendix). Lower is better. 
\best{Red}: best; \second{Blue}: runner-up. Gain row reports SCOT's 
relative improvement over the strongest baseline.}
\label{tab:multi_source_main}
\footnotesize
\setlength{\tabcolsep}{3.5pt}
\renewcommand{\arraystretch}{1.15}

\begin{minipage}[t]{0.49\textwidth}
\centering
\textbf{Target: Beijing (BJ)}\\[2pt]
\resizebox{\linewidth}{!}{%
\begin{tabular}{@{}l|cc|cc|cc@{}}
\toprule
\multirow{2}{*}{\textbf{Method}}
& \multicolumn{2}{c|}{\textbf{GDP}}
& \multicolumn{2}{c|}{\textbf{Population}}
& \multicolumn{2}{c}{\textbf{CO$_2$}}\\
\cmidrule(lr){2-3}\cmidrule(lr){4-5}\cmidrule(lr){6-7}
& MAE & MAPE & MAE & MAPE & MAE & MAPE\\
\midrule
RP         & 172.76 & 7.64 & 679.86 & 6.98 & 166.80 & 2.69\\
\rowcolor{gray!06}
HBP        & 164.25 & 7.13 & 662.60 & 6.53 & 165.53 & 2.57\\
HSA        & 156.40 & 6.62 & 644.89 & 5.53 & 160.05 & 2.59\\
\rowcolor{gray!06}
MMD        & \second{127.45} & \second{4.93} & \second{605.81} & \second{4.11} & 160.52 & \second{1.29}\\
Adv        & 196.76 & 9.90 & 717.96 & 6.34 & 189.41 & 3.19\\
\rowcolor{gray!06}
CrossTReS  & 151.17 & 6.41 & 666.74 & 5.03 & 187.59 & 2.32\\
CoRE       & 152.88 & 5.86 & 620.34 & 4.30 & \second{152.24} & 1.99\\
\midrule
\rowcolor{red!8}
\textbf{\scot (Ours)} 
& \best{104.16} & \best{2.57} & \best{525.10} & \best{1.87} & \best{143.53} & \best{1.16}\\
\rowcolor{blue!6}
\textit{$\Delta$ vs.\ best}
& \cellcolor{blue!16}\textbf{+18.3\%} & \cellcolor{blue!24}\textbf{+47.9\%}
& \cellcolor{blue!12}\textbf{+13.3\%} & \cellcolor{blue!24}\textbf{+54.5\%}
& \cellcolor{blue!8}+5.7\% & \cellcolor{blue!10}+10.1\%\\
\bottomrule
\end{tabular}}
\end{minipage}
\hfill
\begin{minipage}[t]{0.49\textwidth}
\centering
\textbf{Target: Xi'an (XA)}\\[2pt]
\resizebox{\linewidth}{!}{%
\begin{tabular}{@{}l|cc|cc|cc@{}}
\toprule
\multirow{2}{*}{\textbf{Method}}
& \multicolumn{2}{c|}{\textbf{GDP}}
& \multicolumn{2}{c|}{\textbf{Population}}
& \multicolumn{2}{c}{\textbf{CO$_2$}}\\
\cmidrule(lr){2-3}\cmidrule(lr){4-5}\cmidrule(lr){6-7}
& MAE & MAPE & MAE & MAPE & MAE & MAPE\\
\midrule
RP         & 195.91 & 2.89 & 642.01 & 3.67 & 181.16 & 2.88\\
\rowcolor{gray!06}
HBP        & 200.04 & 4.24 & 670.41 & 3.80 & 150.16 & 2.09\\
HSA        & 183.55 & 2.42 & 648.67 & 3.79 & 155.20 & 2.45\\
\rowcolor{gray!06}
MMD        & \second{163.78} & \second{2.18} & \second{506.22} & \second{3.05} & 144.61 & 3.18\\
Adv        & 221.31 & 4.84 & 731.92 & 5.43 & 184.73 & 3.46\\
\rowcolor{gray!06}
CrossTReS  & 179.01 & 4.94 & 625.63 & 5.52 & 151.22 & 3.48\\
CoRE       & 173.72 & 5.48 & 549.37 & 3.89 & \second{134.19} & \second{1.97}\\
\midrule
\rowcolor{red!8}
\textbf{\scot (Ours)} 
& \best{156.94} & \best{1.71} & \best{446.13} & \best{1.86} & \best{127.66} & \best{1.26}\\
\rowcolor{blue!6}
\textit{$\Delta$ vs.\ best}
& \cellcolor{blue!8}+4.2\% & \cellcolor{blue!18}\textbf{+21.6\%}
& \cellcolor{blue!12}\textbf{+11.9\%} & \cellcolor{blue!22}\textbf{+39.0\%}
& \cellcolor{blue!8}+4.9\% & \cellcolor{blue!22}\textbf{+36.0\%}\\
\bottomrule
\end{tabular}}
\end{minipage}

\vspace{-0.8em}
\end{table*}


\subsection{Robustness to Encoder and Readout Choice}
\label{sec:robustness}

A natural concern is whether SCOT's gains stem from its alignment 
design or from incidental choices in the encoder and downstream 
regressor. To isolate this, we hold the alignment module fixed and 
vary the GNN encoder (GAT, GATv2, SuperGAT) and the regressor 
(Ridge, Lasso, SVR, Elastic Net) on BJ$\to$XA 
(Table~\ref{tab:robustness_bj2xa}). The asymmetry is striking: 
\colorbox{yellow!10}{\textbf{encoder/readout substitution $<$3\%}} 
vs.\ \colorbox{blue!8}{\textbf{alignment-design gap 25--40\%}}---roughly an order of magnitude larger. This provides strong evidence that SCOT's gains stem primarily from its \emph{alignment design} rather than from encoder capacity or readout flexibility, and justifies treating both as \emph{interchangeable 
components}.

\begin{table}[H]
\centering
\vspace{-2.2em}
\footnotesize
\caption{Robustness to encoder and readout choice on BJ$\to$XA 
(mean $\pm$ std). 
\best{Red}: best per column.}
\label{tab:robustness_bj2xa}
\setlength{\tabcolsep}{3.5pt}
\renewcommand{\arraystretch}{1.15}
\vspace{0.2em}
\textbf{(a) Encoder ablation} \\[2pt]
\resizebox{\linewidth}{!}{%
\begin{tabular}{@{}l|cc|cc|cc@{}}
\toprule
\multirow{2}{*}{\textbf{Encoder}}
& \multicolumn{2}{c|}{\textbf{GDP}}
& \multicolumn{2}{c|}{\textbf{Population}}
& \multicolumn{2}{c}{\textbf{CO$_2$}}\\
\cmidrule(lr){2-3}\cmidrule(lr){4-5}\cmidrule(lr){6-7}
& MAE$\downarrow$ & MAPE$\downarrow$
& MAE$\downarrow$ & MAPE$\downarrow$
& MAE$\downarrow$ & MAPE$\downarrow$ \\
\midrule
\rowcolor{blue!5}
GAT~\citep{velivckovic2017graph} \textit{(default)}
& \best{160.21 $\pm$ 3.53} & 1.87 $\pm$ 0.18
& \best{450.14 $\pm$ 2.81} & \best{1.73 $\pm$ 0.12}
& \best{127.79 $\pm$ 1.08} & \best{1.78 $\pm$ 0.10} \\
GATv2~\citep{brody2021attentive}
& 162.60 $\pm$ 2.27 & 1.74 $\pm$ 0.22
& 455.36 $\pm$ 4.56 & 1.74 $\pm$ 0.13
& 128.83 $\pm$ 1.29 & 1.82 $\pm$ 0.11 \\
\rowcolor{gray!06}
SuperGAT~\citep{kim2022find}
& 164.42 $\pm$ 5.29 & \best{1.49 $\pm$ 0.13}
& 461.45 $\pm$ 8.37 & 1.95 $\pm$ 0.21
& 132.31 $\pm$ 3.05 & 1.86 $\pm$ 0.21 \\
\bottomrule
\end{tabular}}

\vspace{8pt}
\textbf{(b) Readout ablation} \\[2pt]
\resizebox{\linewidth}{!}{%
\begin{tabular}{@{}l|cc|cc|cc@{}}
\toprule
\multirow{2}{*}{\textbf{Readout}}
& \multicolumn{2}{c|}{\textbf{GDP}}
& \multicolumn{2}{c|}{\textbf{Population}}
& \multicolumn{2}{c}{\textbf{CO$_2$}}\\
\cmidrule(lr){2-3}\cmidrule(lr){4-5}\cmidrule(lr){6-7}
& MAE$\downarrow$ & MAPE$\downarrow$
& MAE$\downarrow$ & MAPE$\downarrow$
& MAE$\downarrow$ & MAPE$\downarrow$ \\
\midrule
\rowcolor{blue!5}
Ridge~\citep{hoerl1970ridge} \textit{(default)}
& 160.21 $\pm$ 3.53 & 1.87 $\pm$ 0.18
& \best{450.14 $\pm$ 2.81} & \best{1.73 $\pm$ 0.12}
& \best{127.79 $\pm$ 1.08} & \best{1.78 $\pm$ 0.10} \\
Lasso~\citep{tibshirani1996regression}
& \best{158.66 $\pm$ 4.07} & 2.03 $\pm$ 0.12
& 455.68 $\pm$ 5.66 & 1.96 $\pm$ 0.02
& 131.40 $\pm$ 3.26 & 1.99 $\pm$ 0.05 \\
\rowcolor{gray!06}
Linear SVR~\citep{drucker1996support}
& 162.23 $\pm$ 2.00 & 1.61 $\pm$ 0.09
& 456.35 $\pm$ 3.62 & 1.94 $\pm$ 0.19
& 128.62 $\pm$ 0.78 & 1.86 $\pm$ 0.18 \\
Elastic Net~\citep{zou2005regularization}
& 164.60 $\pm$ 2.80 & \best{1.57 $\pm$ 0.19}
& 459.46 $\pm$ 3.93 & 1.89 $\pm$ 0.34
& 129.90 $\pm$ 1.83 & 1.82 $\pm$ 0.31 \\
\bottomrule
\end{tabular}}
\vspace{-6pt}
\end{table}

\subsection{Ablation Study}
\label{sec:ablation}

\begin{figure}[t]
\centering
\begin{minipage}[t]{0.50\columnwidth}
    \centering
    \includegraphics[width=\linewidth]{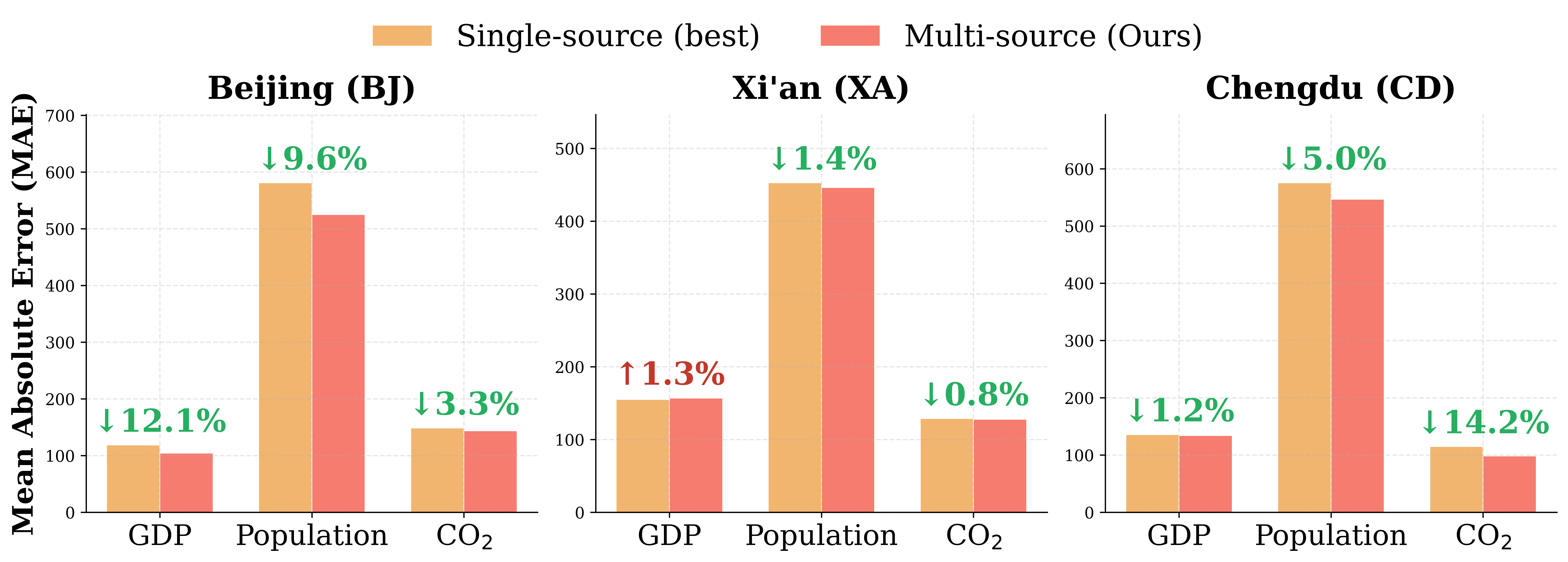}
    \captionsetup{font=small}
    \captionof{figure}{Best single-source \scot (orange) vs.\ 
    multi-source \scot (red). Labels show relative MAE change $\Delta$ 
    (blue: improvement; red: degradation).}
    \label{fig:scot_single_vs_multi_mae}
    
    \vspace{-0.2em}
    
    \includegraphics[width=\linewidth]{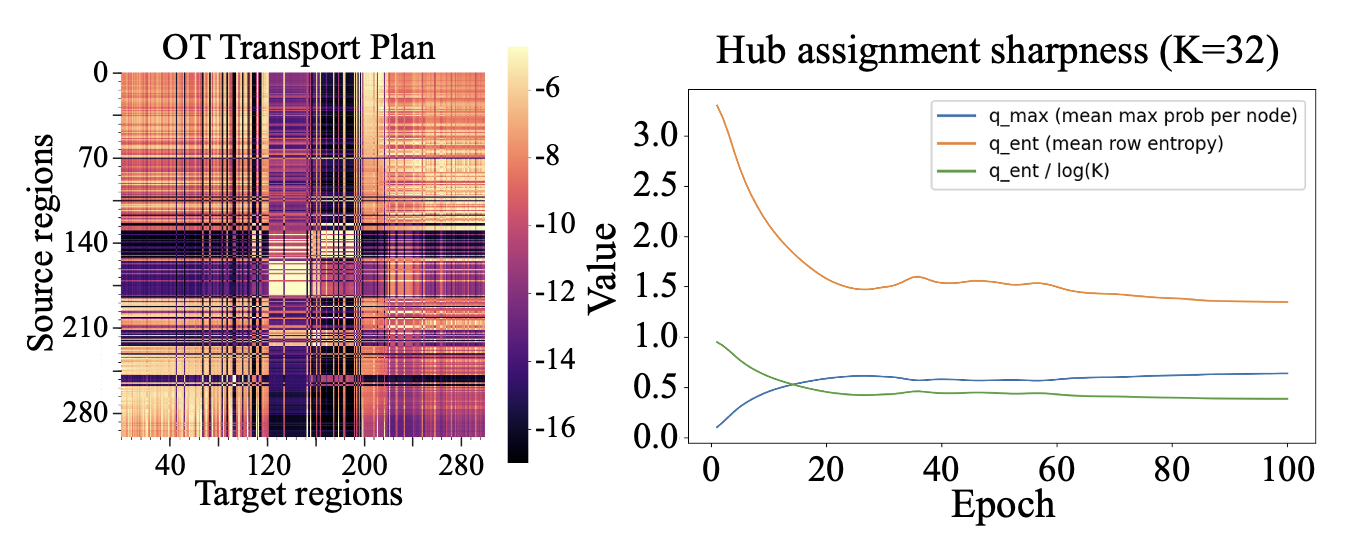}
    \captionsetup{font=small}
    \captionof{figure}{Diagnostics: (left) entropic OT coupling 
    (XA$\to$BJ, epoch 100); (right) hub assignment sharpness for 
    $K=32$.}
    \label{fig:alignment_diagnostics}
\end{minipage}%
\hfill
\begin{minipage}[t]{0.46\columnwidth}
    \centering
     \vspace{-7em}
    \includegraphics[width=\linewidth]{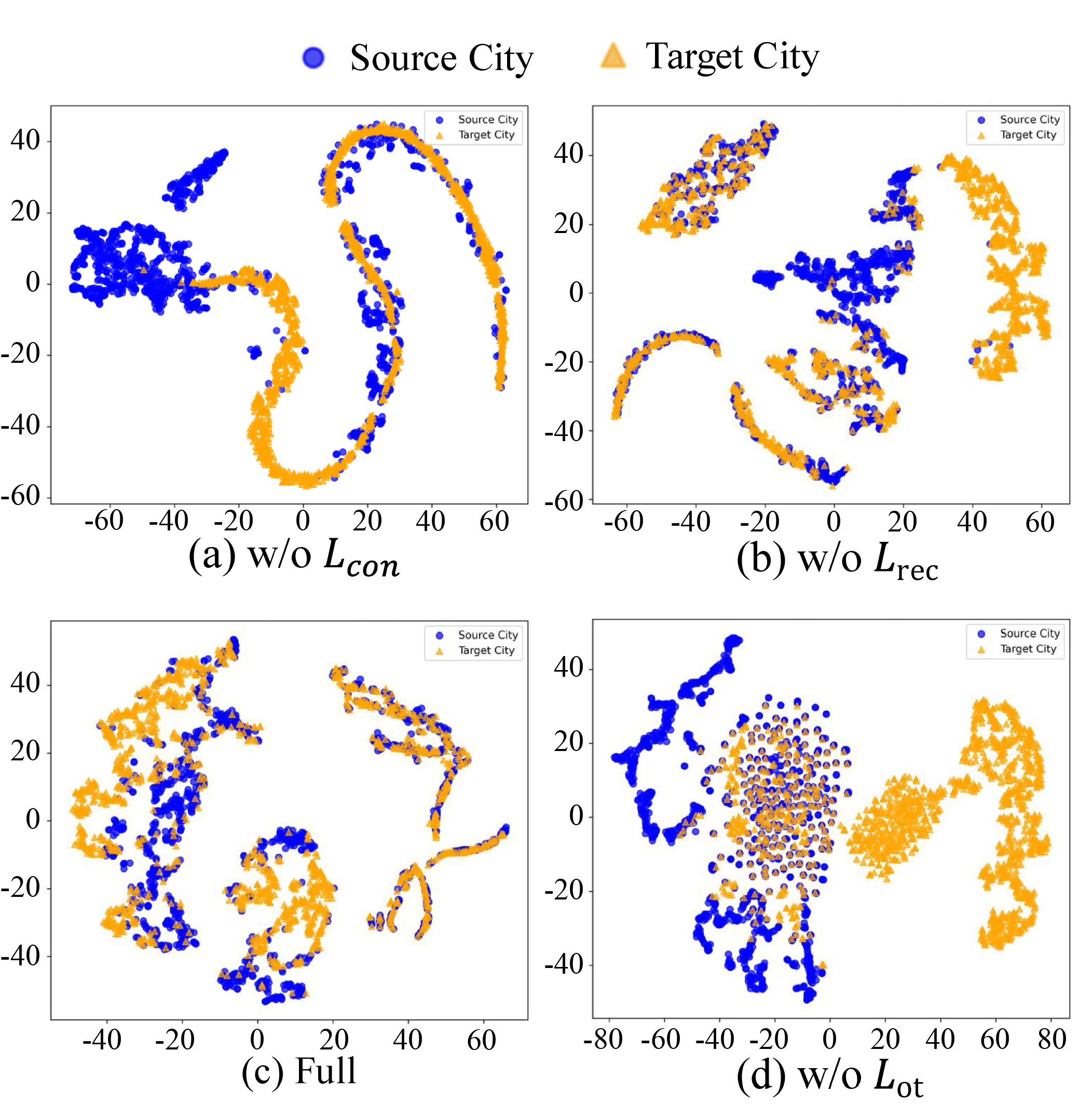}
    \captionsetup{font=small}
    \captionof{figure}{\textbf{Ablation t-SNE of \scot (XA$\to$BJ).} 
    (a) w/o contrastive, (b) w/o reconstruction, (c) full, (d) w/o 
    OT. The full model is the best.}
    \label{fig:tsne_scot_ablation}
\end{minipage}
\end{figure}

\paragraph{Single-source components.}
Removing $\mathcal{L}_{\mathrm{Con}}$, $\mathcal{L}_{\mathrm{OT}}$, or
$\mathcal{L}_{\mathrm{Rec}}$ consistently degrades transfer performance
(Figs.~\ref{fig:tsne_scot_ablation} and~\ref{fig:ablation_transfer}).
We retain the t-SNE plots only as qualitative views of the embedding geometry;
direct diagnostics are reported in
Appendix~\ref{app:single_source_ablation}. The three terms exhibit distinct failure modes. Without
$\mathcal{L}_{\mathrm{Con}}$, the row-normalized transport becomes nearly
uniform, showing insufficient local discrimination. Without
$\mathcal{L}_{\mathrm{OT}}$, the coupling develops a dominant entry together
with a broad diffuse tail, indicating that contrastive learning alone does not
recover structured global mass allocation. Removing
$\mathcal{L}_{\mathrm{Rec}}$ increases cycle error and cross-seed variation.
Thus, $\mathcal{L}_{\mathrm{OT}}$, $\mathcal{L}_{\mathrm{Con}}$, and
$\mathcal{L}_{\mathrm{Rec}}$ respectively contribute global allocation, local
relative discrimination, and correspondence stability; the full model achieves
the lowest MAE/MAPE.


\paragraph{Multi-source and additional ablations.}
\scot's multi-source design rests on three deliberate choices 
(Appendix~\ref{app:ablation_study}): \emph{hub mediation} 
(vs.\ pairwise OT), \emph{target-induced prior} (vs.\ uniform), and 
\emph{balanced OT} (vs.\ unbalanced). 
\colorbox{red!8}{\textbf{Each contributes meaningfully}}; one-sided 
cycle wins under $n_s\!\neq\!n_t$, and alignment preserves 
within-city quality.

\begin{figure}[t]
\vspace{-0.4em}
    \centering
    \begin{minipage}[t]{0.48\columnwidth}
        \centering
        \includegraphics[width=\linewidth]{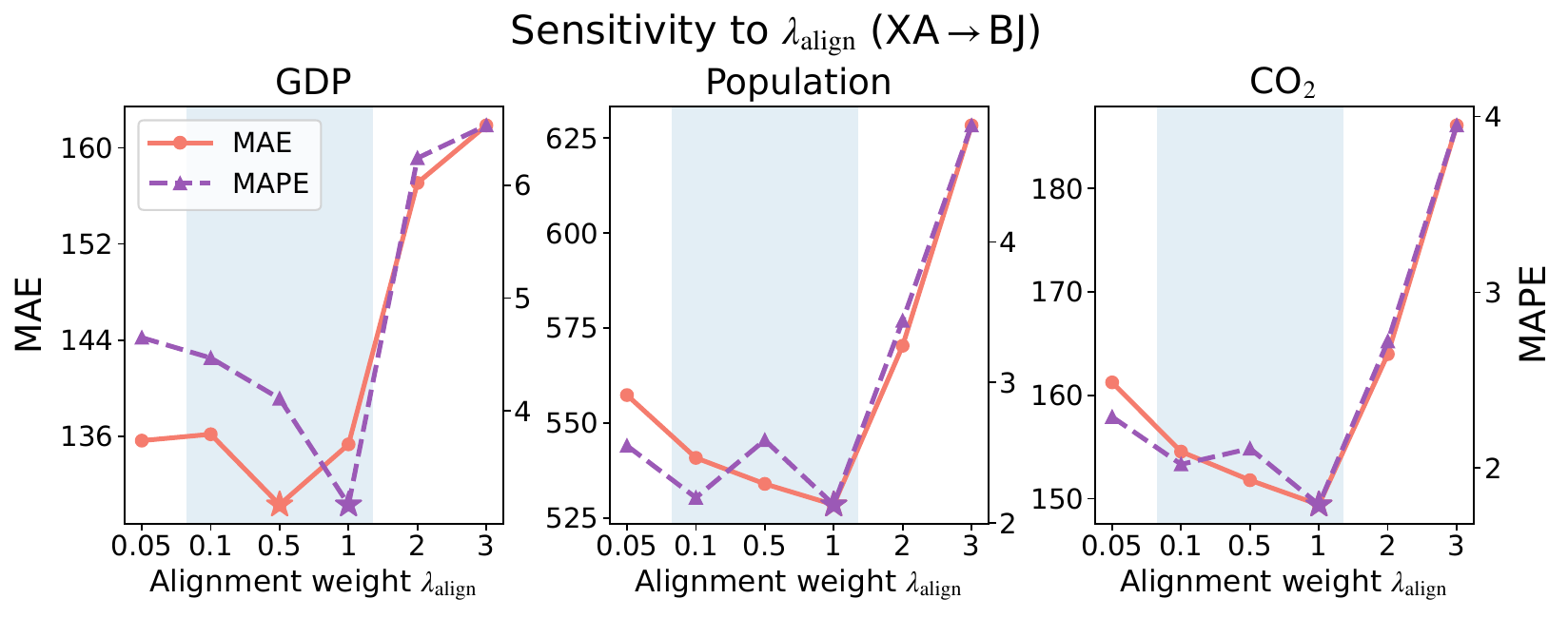}
        \caption{Sensitivity to $\lambda_{\mathrm{align}}$.}
        \label{fig:lambda_align_sens_xa2bj}
    \end{minipage}
    \hfill
    \begin{minipage}[t]{0.48\columnwidth}
        \centering
        \includegraphics[width=\linewidth]{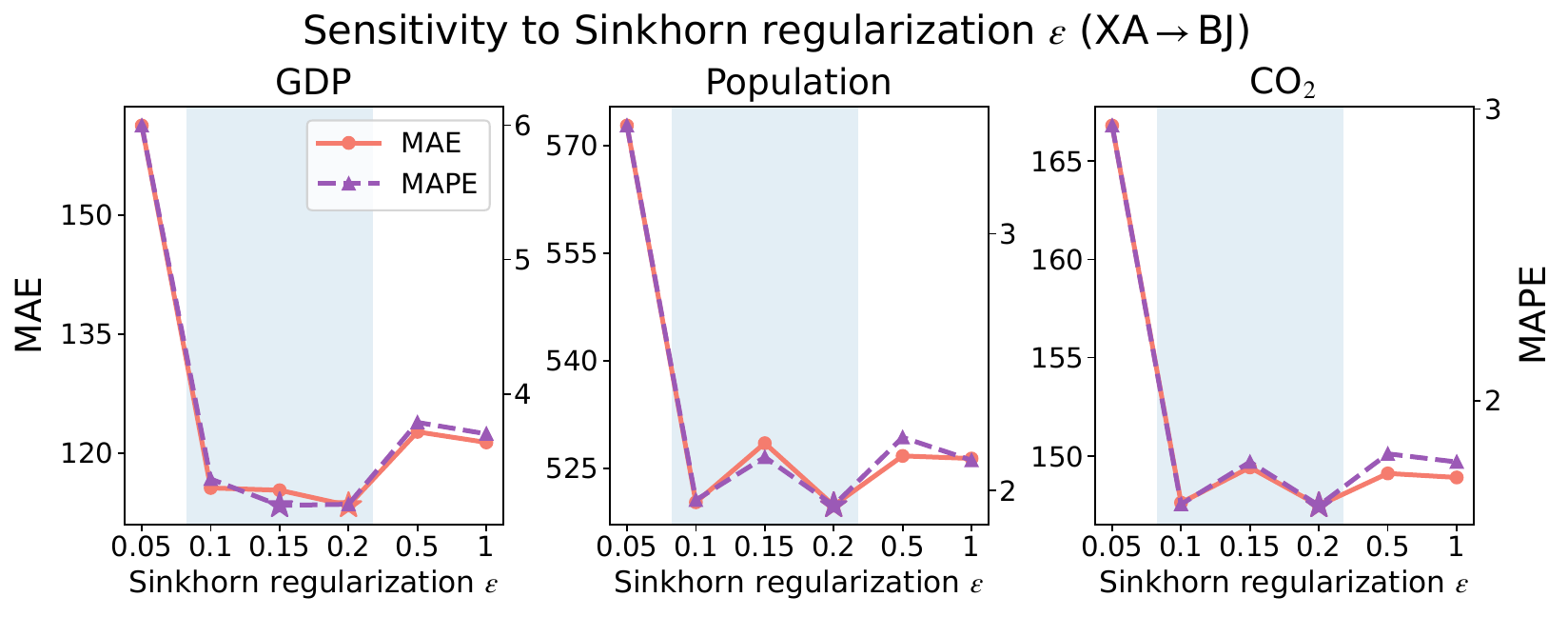}
        \caption{Sensitivity to OT regularization $\varepsilon$.}
        \label{fig:eps_sens_xa2bj}
    \end{minipage}
    \vspace{-1.5em}
\end{figure}

\subsection{Diagnostics}
\label{sec:diagnostics}

\cI~\textbf{OT coupling structure.} The learned coupling 
$\mathbf{P}$ for XA$\to$BJ shows clear block structure after 
barycentric reordering (Fig.~\ref{fig:alignment_diagnostics}, 
left), confirming \emph{selective many-to-many correspondences} 
without hubness or indiscriminate spreading.

\cII~\textbf{Hub assignment sharpness.} 
The normalized entropy $q_{\mathrm{ent}}/\log K\!\approx\!0.4$ 
($K\!=\!32$) corresponds to roughly four active prototypes per 
region---\emph{stable specialization rather than uniform pooling}, 
indicating the target-induced prior steers prototypes toward 
task-relevant semantics (App.~\ref{app:ot_diagnostics}).

\cIII~\textbf{Semantic coherence of transport blocks.}
Using two region attributes excluded from training, \textsc{Function} and
\textsc{Form\_Type}, the top-$k$ targets of the source region in
Fig.~\ref{fig:alignment_diagnostics} are more semantically coherent than size-matched random
subsets, with purity gains of $10.2$--$26.9\%$ and $10.4$--$22.1\%$,
respectively. A row-shuffled control preserves each transport profile but breaks
its source-region association, reducing same-\textsc{Function} mass and
worsening \textsc{Form\_Type} similarity. This indicates source-dependent semantic concentration rather than global block reuse
(App.~\ref{app:transport_semantics}).



\subsection{Hyperparameter Sensitivity}
\label{sec:hp_sensitivity}




\colorbox{yellow!20}{\textbf{\scot remains stable across broad 
parameter ranges}}, staying competitive with baselines even at 
off-recommended settings---valuable for label-scarce targets where 
per-target tuning is infeasible. We show three cases; remaining 
sweeps ($\eta$, $\tau$, $\tau_b$) are in 
Appendix~\ref{app:hyperpameter_sensitivity}.

\paragraph{Sensitivity to $\lambda_{\mathrm{align}}$.}
Sweeping $\lambda_{\mathrm{align}}\!\in\!\{0.05, 0.1, 0.5, 1, 2, 3\}$ 
(Fig.~\ref{fig:lambda_align_sens_xa2bj}), performance is best for 
$\lambda_{\mathrm{align}}\!\in\![0.1,1]$ and degrades only at 
$\lambda_{\mathrm{align}}\!\ge\!2$ (over-alignment). Default: $1$.

\paragraph{Sensitivity to $\varepsilon$.}
Sweeping $\varepsilon\!\in\!\{0.05, 0.1, 0.15, 0.2, 0.5, 1\}$ 
(Fig.~\ref{fig:eps_sens_xa2bj}), performance peaks at 
$\varepsilon\!\in\![0.1, 0.2]$; smaller values give noisy couplings, 
larger values diffuse matching. Default: $0.15$.

\paragraph{Hub size $K$.}
On CD,BJ$\to$XA, performance is stable for $K\!\in\![4,32]$ 
(smaller $K$ underfits, $K\!\ge\!64$ over-resolves; 
App.~\ref{app:hubK_sens}). Default: $32$. For $M\!>\!2$ sources, 
$K\!\approx\!\min(32\log_2(M{+}1),\, n_t/2)$ serves as a starting 
heuristic; hub-conflict diagnostics 
(App.~\ref{app:source_quality}) provide an unsupervised adjustment 
signal.

\subsection{Generality beyond mobility graphs.}
SCOT's cross-city alignment operates on normalized region embeddings and does
not require shared edges, node identities, equal region counts, or comparable
partitions. It can therefore be paired with different graph encoders and input
signals. This is supported by our encoder and readout ablations
(Section~\ref{sec:ablation}) and by the cross-country experiments with NYC
(Appendix~\ref{app:nyc_transfer}). We further verify the alignment module in a controlled
study with mobility, functional-similarity, and spatial-functional graphs, and
under increasing source--target heterogeneity
(Appendix~\ref{app:controlled_generality}).


\section{Conclusion}
\label{sec:conclusion}

We proposed \scot, a framework for cross-city region 
transfer that addresses two complementary failure modes of existing 
alignment paradigms---hubness under unequal partitions and 
over-mixing under heterogeneity---through Sinkhorn-based entropic 
OT paired with an OT-guided contrastive objective. A shared 
prototype hub with target-induced prior extends this to multi-source 
transfer, coordinating heterogeneous sources without gradient 
conflict or explicit source selection. Experiments on GDP, 
population, and CO$_2$ across multiple cities and directions show 
consistent gains over strong baselines. Future work includes 
uncertainty-aware mechanisms for selectively integrating sources 
under severe heterogeneity.

\bibliographystyle{plainnat}
\bibliography{example_paper}

\newpage

\appendix

\section{Experimental Details}
\label{app:details}

\subsection{Data}
\label{app:data}
We use datasets from Xi'an (XA), Chengdu (CD), and Beijing (BJ). Each city is partitioned into irregular road-network-based regions, with one month of anonymized taxi OD trips mapped to regions to form a directed mobility graph. We evaluate three region-level targets (GDP, population, and carbon emissions) aggregated from public gridded/raster products by assigning grid cells to polygons and summing within each region.

\begin{table}[H]
\centering
\vspace{-0.4em}
\caption{Dataset summary.}
\footnotesize
\label{tab:dataset_summary}
\begin{tabular}{lccc}
\toprule
\textbf{City} & \textbf{\# Regions} & \textbf{\# Trips} & \textbf{Targets} \\
\midrule
XA & 1306 & 559{,}729 & GDP / Pop / CO$_2$ \\
CD & 1056 & 384{,}618 & GDP / Pop / CO$_2$ \\
BJ & 1311 & 78{,}945  & GDP / Pop / CO$_2$ \\
\bottomrule
\end{tabular}
\end{table}

\subsection{Baselines.}

\label{app:baselines}


\paragraph{Baselines.}
We compare \scot with the following methods. \textbf{Non-Alignment} 
trains on the source and directly applies to the target without 
adaptation. \textbf{RP}~\citep{yabe2019city2city} forms one-to-one 
anchors by rank-matching regions and aligns them via 
Procrustes~\citep{schonemann1966generalized}. 
\textbf{HBP}~\citep{yabe2020unsupervised} uses level-wise prototype 
(mean) vectors as anchors under a hierarchical partition, aligned by 
Procrustes. \textbf{HSA}~\citep{yabe2020unsupervised} samples anchors 
within each hierarchical level and fits an unconstrained affine map 
for more flexible alignment. \textbf{MMD}~\citep{gretton2012kernel} 
is an RKHS-based correspondence-free loss matching source and target 
embedding distributions. \textbf{Adv} 
(DANN)~\citep{ganin2016domain} learns domain-invariant embeddings 
via gradient reversal against a domain discriminator. 
\textbf{CrossTReS}~\citep{jin2022selective} is a selective 
fine-tuning framework that meta-learns region weights to prioritize 
source regions most helpful for the target. 
\textbf{CoRE}~\citep{chencross} jointly learns region embeddings and 
aligns the two latent spaces both globally and at the region level.

\section{Proof of Proposition~\ref{prop:contrastive_mae}}
\label{app:contrastive_mae_proof}

\contrastivemaeprop*

\begin{proof}
Throughout, let $Z_i:=\sum_k e^{\langle u_i,v_k\rangle/\tau}$,
$q_i(j):=e^{\langle u_i,v_j\rangle/\tau}/Z_i$, and
\[
\mathcal L_{\mathrm{Con}}(P):=-\sum_i a_i\log\sum_j P_{ij}\,q_i(j),
\]
which coincides with Eq.~\eqref{eq:lcon} for uniform $a$. Rows with
$a_i=0$ contribute nothing and are dropped.

Let $\phi:=|h-g|$. For any $x,x'\in\mathbb S^{d-1}$,
$|\phi(x)-\phi(x')|\le|(h-g)(x)-(h-g)(x')|\le(L_h+L_g)\|x-x'\|_2$.
Let $(I,J)\sim P$; since $P\mathbf 1=a$ and $P^\top\mathbf 1=b$, we have
$I\sim a$ and $J\sim b$. Hence
\[
\mathcal R_t^b(h)-\mathcal R_s^a(h)
=\mathbb E\big[\phi(v_J)-\phi(u_I)\big]
\le (L_h+L_g)\,\mathbb E\|u_I-v_J\|_2 .
\]

Since $\|u_i\|_2=\|v_j\|_2=1$, we have $\|u_i-v_j\|_2^2=2-2\langle u_i,v_j\rangle$,
and Jensen's inequality gives
\[
\mathbb E\|u_I-v_J\|_2\le\sqrt{\mathbb E\|u_I-v_J\|_2^2}
=\sqrt{2-2\,\mathbb E\langle u_I,v_J\rangle}.
\]

Fix $i$ and let $p_i(j):=P_{ij}/a_i$ and
$\ell_i:=-\log\sum_j p_i(j)q_i(j)$. Using $\langle u_i,v_k\rangle\ge-1$,
\[
\mathbb E_{J\sim p_i}e^{\langle u_i,v_J\rangle/\tau}
=Z_i e^{-\ell_i}\ge n_t\,e^{-1/\tau-\ell_i}.
\]
For $X:=\langle u_i,v_J\rangle\in[-1,1]$, Hoeffding's lemma gives
$\log\mathbb E e^{\lambda X}\le\lambda\,\mathbb E X+\lambda^2/2$. Taking
$\lambda=1/\tau$,
\[
\mathbb E_{J\sim p_i}\langle u_i,v_J\rangle
\ge\tau\log\mathbb E_{J\sim p_i}e^{X/\tau}-\frac{1}{2\tau}
\ge\tau\log n_t-1-\tau\ell_i-\frac{1}{2\tau}.
\]
Since $\log\sum_j P_{ij}q_i(j)=\log a_i-\ell_i$, we have
$\sum_i a_i\ell_i=\mathcal L_{\mathrm{Con}}(P)-H(a)$. Averaging over
$I\sim a$ therefore yields
\[
\mathbb E\langle u_I,v_J\rangle
\ge\tau\log n_t+\tau H(a)-\tau\mathcal L_{\mathrm{Con}}(P)-1-\frac{1}{2\tau}.
\]
Together with $\langle u_I,v_J\rangle\ge-1$, this gives
$\mathbb E\langle u_I,v_J\rangle\ge\underline m$. Note also
$\underline m\le\mathbb E\langle u_I,v_J\rangle\le1$, so the square root
below is well defined. Combining the three steps,
\[
\mathcal R_t^b(h)-\mathcal R_s^a(h)\le(L_h+L_g)\sqrt{2-2\,\underline m}.
\qedhere
\]
\end{proof}

\begin{remark}
The constants come from worst-case steps ($Z_i\ge n_t e^{-1/\tau}$ and
Hoeffding's lemma), so the bound is loose and can be vacuous
($\underline m=-1$) at small $\tau$. We use it qualitatively: the transfer
gap is non-increasing in $\mathcal L_{\mathrm{Con}}$, a relation we verify
empirically in Appendix~\ref{app:theory_empirical_check}.
\end{remark}

\section{Training Algorithms}
\label{app:algorithms}

We provide pseudocode for the single-source 
(Algorithm~\ref{alg:scot_single_sinkhorn}) and multi-source 
(Algorithm~\ref{alg:scot_multi_hub_ot_single_style}) variants of \scot.

\subsection{Single-source \scot Training Algorithm}

\begin{algorithm}[H]
\caption{Single-source \scot training}
\label{alg:scot_single_sinkhorn}
\definecolor{stage1}{RGB}{230, 240, 255}
\definecolor{stage2}{RGB}{255, 240, 230}
\definecolor{stage3}{RGB}{235, 250, 235}
\definecolor{stage4}{RGB}{250, 235, 245}
\definecolor{commentgray}{RGB}{120, 120, 120}

\begin{algorithmic}[1]
\Require Graphs $\mathcal{G}_s,\mathcal{G}_t$; mobility $\mathbf{M}_s,\mathbf{M}_t$; 
hyperparameters $\tau,\varepsilon,T,\eta,\beta,\lambda_{\mathrm{align}},\lambda_{\mathrm{rec}}$; 
step size $\alpha$
\Ensure Trained parameters $\Theta$

\Statex \colorbox{stage1}{\parbox{0.96\linewidth}{\small\textsc{\textbf{$\blacktriangleright$ Stage 1.}\ \ Encode regions \& compute intra-city objective}}}
\For{$\text{epoch}=1,2,\dots$}
  \State $\mathbf{z}_c \leftarrow \mathrm{GAT}(\mathcal{G}_c;\Theta)$,
         \quad $\mathcal{L}_{\mathrm{intra}}^c \leftarrow \mathcal{L}_{\mathrm{intra}}(\mathbf{z}_c;\mathbf{M}_c)$
         \quad for $c\in\{s,t\}$

\Statex \colorbox{stage2}{\parbox{0.96\linewidth}{\small\textsc{\textbf{$\blacktriangleright$ Stage 2.}\ \ Sinkhorn entropic OT coupling}}}
  \State $\tilde{\mathbf{z}}_c \leftarrow \mathrm{RowNorm}(\mathbf{z}_c)$,
         \quad $C_{ij} \leftarrow \lVert\tilde{\mathbf{z}}_i^s - \tilde{\mathbf{z}}_j^t\rVert_2$,
         \quad $\mathbf{K} \leftarrow \exp(-\mathbf{C}/\varepsilon)$
  \State Initialize $\mathbf{u},\mathbf{v}\leftarrow\mathbf{1}$
  \For{$k=1,\dots,T$}
       \State $\mathbf{u}\leftarrow\mathbf{1}\oslash(\mathbf{K}\mathbf{v})$,
              \quad $\mathbf{v}\leftarrow\mathbf{1}\oslash(\mathbf{K}^\top\mathbf{u})$
  \EndFor
  \State $\mathbf{P} \leftarrow \mathrm{diag}(\mathbf{u})\,\mathbf{K}\,\mathrm{diag}(\mathbf{v})$
         \hfill {\color{commentgray}\textit{// soft correspondence}}

\Statex \colorbox{stage3}{\parbox{0.96\linewidth}{\small\textsc{\textbf{$\blacktriangleright$ Stage 3.}\ \ Alignment \& reconstruction losses}}}
  \State $\mathcal{L}_{\mathrm{OT}} \leftarrow \langle\mathbf{P},\mathbf{C}\rangle / \min(n_s,n_t)$,
         \quad $\mathcal{L}_{\mathrm{Con}}$ via Eq.~\eqref{eq:lcon}
  \State $\mathcal{L}_{\mathrm{Align}} \leftarrow \mathcal{L}_{\mathrm{OT}} + \eta\,\mathcal{L}_{\mathrm{Con}}$,
         \quad $\mathcal{L}_{\mathrm{Rec}} \leftarrow \mathcal{L}_{\mathrm{cyc}} + \beta\,\mathcal{R}_{\mathrm{ent}}$

\Statex \colorbox{stage4}{\parbox{0.96\linewidth}{\small\textsc{\textbf{$\blacktriangleright$ Stage 4.}\ \ Total objective \& parameter update}}}
  \State $\mathcal{L} \leftarrow 
         \underbrace{\mathcal{L}_{\mathrm{intra}}^s + \mathcal{L}_{\mathrm{intra}}^t}_{\text{\scriptsize intra-city}}
         \,+\, \lambda_{\mathrm{align}}\underbrace{\mathcal{L}_{\mathrm{Align}}}_{\text{\scriptsize cross-city}}
         \,+\, \lambda_{\mathrm{rec}}\underbrace{\mathcal{L}_{\mathrm{Rec}}}_{\text{\scriptsize stabilization}}$
  \State $\Theta \leftarrow \Theta - \alpha\,\nabla_\Theta \mathcal{L}$

\EndFor
\end{algorithmic}
\end{algorithm}


\subsection{Multi-source hub alignment (illustration)}
\label{app:multi_source_hub}

\begin{wrapfigure}{r}{0.5\textwidth}
    \vspace{-1em} 
    \centering
    \includegraphics[width=\linewidth]{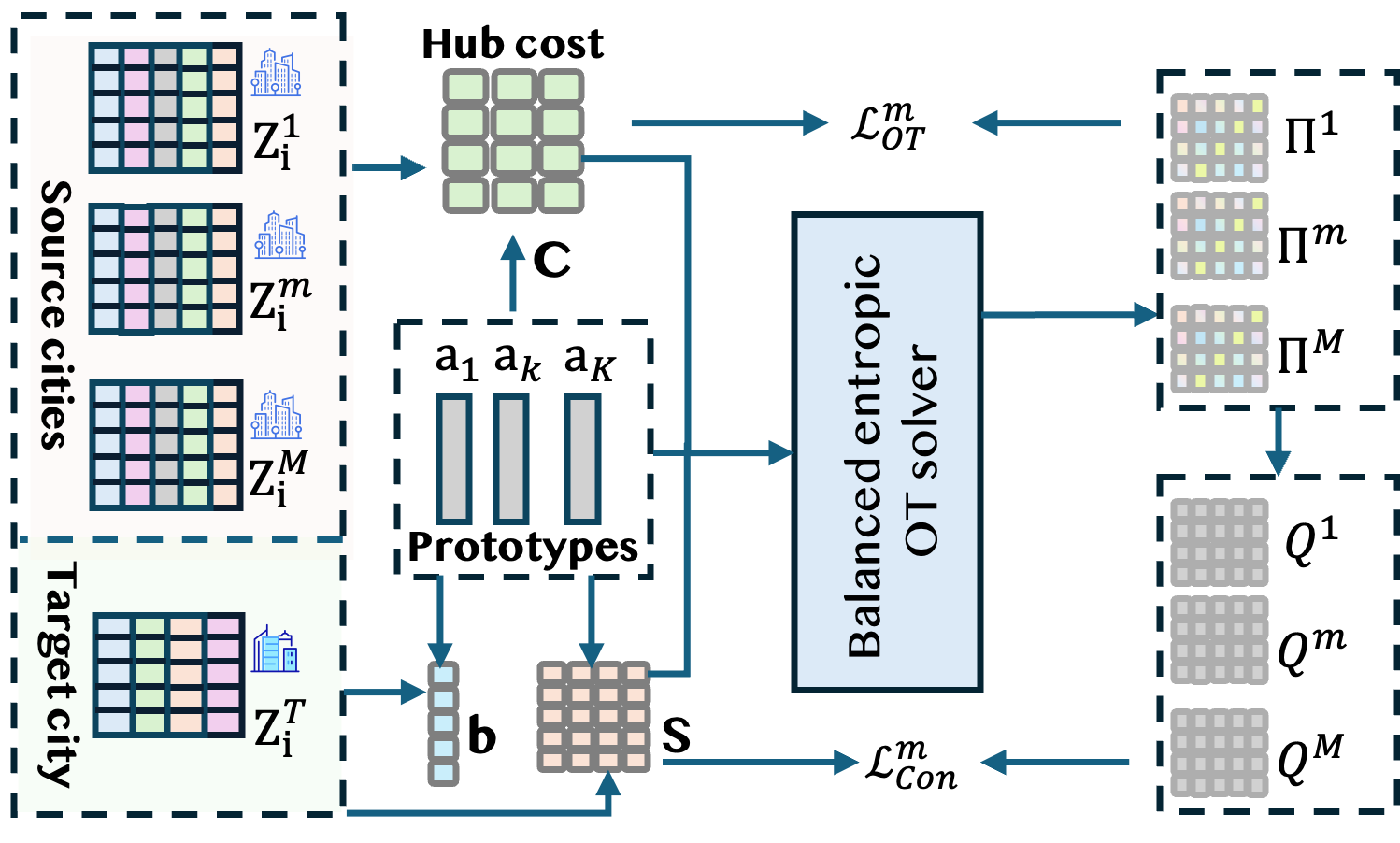}
    \caption{\textbf{Illustration of multi-source \scot.}}
    \vspace{-.5em} 
    \label{fig:multi}
\end{wrapfigure}

Figure~\ref{fig:multi} provides an illustration of the proposed multi-source hub alignment mechanism.
Given region embeddings from multiple source cities $\{Z^{(1)},\dots,Z^{(M)}\}$ and a target city $Z^{(T)}$, we introduce a set of shared prototypes (hubs) that serve as intermediate anchors to mediate cross-city alignment.

Specifically, each city is softly assigned to the prototypes via learned assignment distributions, inducing hub-level representations that summarize transferable structure across sources.
These hub representations are then aligned with the target city through a balanced entropic OT solver, producing transport plans $\{\Pi^{(m)}\}$ that guide both the OT alignment loss $\mathcal{L}^{m}_{\mathrm{OT}}$ and the OT-weighted contrastive objective $\mathcal{L}^{m}_{\mathrm{Con}}$. This design enables structured many-to-many alignment across cities while avoiding brittle pairwise correspondences and mitigating hub collapse, yielding a scalable and robust extension of \scot to the multi-source setting.

\subsection{Algorithm: Multi-source \scot training}

\begin{algorithm}[H]
\caption{Multi-source \scot training (hub entropic OT)}
\label{alg:scot_multi_hub_ot_single_style}
\definecolor{stage1}{RGB}{230, 240, 255}
\definecolor{stage2}{RGB}{255, 240, 230}
\definecolor{stage3}{RGB}{235, 250, 235}
\definecolor{stage4}{RGB}{250, 235, 245}
\definecolor{commentgray}{RGB}{120, 120, 120}

\begin{algorithmic}[1]
\Require $\{\mathcal{G}_m,\mathbf{M}_m\}_{m\in\mathcal{S}}$, $\mathcal{G}_t,\mathbf{M}_t$; 
hub size $K$; hyperparameters 
$\tau,\varepsilon,T,\lambda_{\mathrm{align}},\lambda_{\mathrm{rec}},\lambda_c,\lambda_{\mathrm{hub}},\beta$; 
step size $\alpha$
\Ensure Trained parameters $\Theta$ (including learnable prototypes $\mathbf{a}$)

\Statex \colorbox{stage1}{\parbox{0.96\linewidth}{\small\textsc{\textbf{$\blacktriangleright$ Stage 1.}\ \ Encode all cities \& compute intra-city losses}}}
\For{$\text{epoch}=1,2,\dots$}
  \State $\mathbf{z}_m \leftarrow \mathrm{GAT}_m(\mathcal{G}_m;\Theta)$,
         \quad $\mathcal{L}_{\mathrm{intra}}^{m} \leftarrow \mathcal{L}_{\mathrm{intra}}(\mathbf{z}_m;\mathbf{M}_m)$
         \quad for $m \in \mathcal{S}\cup\{t\}$

\Statex \colorbox{stage2}{\parbox{0.96\linewidth}{\small\textsc{\textbf{$\blacktriangleright$ Stage 2.}\ \ City-to-hub coupling via balanced entropic OT}}}
  \State $\tilde{\mathbf{z}}_m \leftarrow \mathrm{RowNorm}(\mathbf{z}_m)$ for $m\in\mathcal{S}\cup\{t\}$,
         \quad $\tilde{\mathbf{a}} \leftarrow \mathrm{RowNorm}(\mathbf{a})$
  \ForAll{$m \in \mathcal{S}\cup\{t\}$}
       \State $\mathbf{C}^{m}_{ik} \leftarrow \lVert\tilde{\mathbf{z}}^{m}_{i} - \tilde{\mathbf{a}}_{k}\rVert_{2}$
              \hfill {\color{commentgray}\textit{// city-to-hub cost}}
       \State $\boldsymbol{\Pi}^{m} \leftarrow \arg\min_{\mathbf{P}\in\Pi(\mathbf{a}^m,\mathbf{b})} 
              \langle\mathbf{P},\mathbf{C}^{m}\rangle - \varepsilon\,H(\mathbf{P})$
              \hfill {\color{commentgray}\textit{// Sinkhorn, $T$ iters}}
       \State $\mathbf{Q}^{m} \leftarrow \mathrm{RowNorm}(\boldsymbol{\Pi}^{m})$
  \EndFor

\Statex \colorbox{stage3}{\parbox{0.96\linewidth}{\small\textsc{\textbf{$\blacktriangleright$ Stage 3.}\ \ Per-city alignment \& hub-balance losses}}}
  \ForAll{$m \in \mathcal{S}\cup\{t\}$}
       \State $\mathcal{L}_{\mathrm{OT}}^{m} \leftarrow \langle\boldsymbol{\Pi}^{m},\mathbf{C}^{m}\rangle/\min(n_m,K)$
       \State $\mathcal{L}_{\mathrm{Con}}^{m} \leftarrow 
              -\tfrac{1}{n_m}\sum_i \log\!\Big(
              \tfrac{\sum_k Q^m_{ik}\exp(\tilde{\mathbf{z}}_i^{m\top}\tilde{\mathbf{a}}_k/\tau)}
              {\sum_k \exp(\tilde{\mathbf{z}}_i^{m\top}\tilde{\mathbf{a}}_k/\tau)}\Big)$
       \State $\mathcal{L}_{\mathrm{align}}^{m} \leftarrow \mathcal{L}_{\mathrm{OT}}^{m} + \lambda_c\,\mathcal{L}_{\mathrm{Con}}^{m}$
       \State $\mathbf{p}^{m} \leftarrow (\boldsymbol{\Pi}^{m})^\top\mathbf{1}$,
              \quad $\mathcal{L}_{\mathrm{hub}}^{m} \leftarrow \mathrm{KL}\!\big(\mathbf{p}^{m}\,\|\,\tfrac{1}{K}\mathbf{1}\big)$
  \EndFor

\Statex \colorbox{stage4}{\parbox{0.96\linewidth}{\small\textsc{\textbf{$\blacktriangleright$ Stage 4.}\ \ Aggregate \& update}}}
  \State $\mathcal{L}_{\mathrm{align}} \leftarrow 
         \tfrac{1}{|\mathcal{S}|+1}\sum_{m}\mathcal{L}_{\mathrm{align}}^{m} 
         + \lambda_{\mathrm{hub}}\cdot\tfrac{1}{|\mathcal{S}|+1}\sum_{m}\mathcal{L}_{\mathrm{hub}}^{m}$
  \State $\mathcal{L}_{\mathrm{rec}} \leftarrow 
         \mathcal{L}_{\mathrm{cycle}}(\{\mathbf{z}_m\}_{m\in\mathcal{S}},\mathbf{z}_t) 
         + \beta\,\mathrm{Ent}(\mathbf{A}_{\cdot\to t})$
  \State $\mathcal{L} \leftarrow 
         \underbrace{\textstyle\sum_{m\in\mathcal{S}}\mathcal{L}_{\mathrm{intra}}^m + \mathcal{L}_{\mathrm{intra}}^t}_{\text{\scriptsize intra-city}}
         + \lambda_{\mathrm{align}}\underbrace{\mathcal{L}_{\mathrm{align}}}_{\text{\scriptsize hub alignment}}
         + \lambda_{\mathrm{rec}}\underbrace{\mathcal{L}_{\mathrm{rec}}}_{\text{\scriptsize stabilization}}$
  \State $\Theta \leftarrow \Theta - \alpha\,\nabla_\Theta\mathcal{L}$

\EndFor
\end{algorithmic}
\end{algorithm}

\section{Intra-city Prediction with and without Alignment}
\label{app:intra_city_alignment}

Cross-city gains should not come at the cost of within-city accuracy. We
therefore compare full \scot with an alignment-free variant on intra-city
prediction (XA$\to$XA, BJ$\to$BJ). The differences are small: MAE changes
by less than $1\%$ in every case, and MAPE by at most $0.2$ points
(Table~\ref{tab:intra_city_alignment}). Alignment thus leaves within-city
quality largely intact.

\begin{table}[H]
\centering
\footnotesize
\caption{Intra-city prediction with and without alignment 
(4 seeds, mean $\pm$ std). MAE differs by less than $1\%$ in all cases.
Lower is better.}
\label{tab:intra_city_alignment}
\setlength{\tabcolsep}{4pt}
\renewcommand{\arraystretch}{1.2}
\begin{tabular}{@{}l|cc|cc|cc@{}}
\toprule
\multirow{2}{*}{\textbf{Direction / Variant}}
& \multicolumn{2}{c|}{\textbf{GDP}}
& \multicolumn{2}{c|}{\textbf{Population}}
& \multicolumn{2}{c}{\textbf{CO$_2$}} \\
\cmidrule(lr){2-3}\cmidrule(lr){4-5}\cmidrule(lr){6-7}
& MAE$\downarrow$ & MAPE$\downarrow$
& MAE$\downarrow$ & MAPE$\downarrow$
& MAE$\downarrow$ & MAPE$\downarrow$ \\
\midrule

\rowcolor{cyan!18}
\multicolumn{7}{@{}l}{\textcolor{cyan!50!black}{$\blacktriangleright$}\ \textbf{\textsc{Target: Xi'an (XA)}}} \\
\midrule
\rowcolor{blue!5}
\hspace{0.5em}XA$\to$XA \textit{(w/o Alignment)}
& 155.31 $\pm$ 1.92 & 3.63 $\pm$ 0.06
& 467.03 $\pm$ 4.45 & 2.50 $\pm$ 0.04
& 130.35 $\pm$ 0.97 & 2.60 $\pm$ 0.05 \\
\hspace{0.5em}XA$\to$XA \textit{(Full \scot)}
& 156.44 $\pm$ 2.29 & 3.67 $\pm$ 0.11
& 467.96 $\pm$ 8.19 & 2.51 $\pm$ 0.05
& 130.55 $\pm$ 1.93 & 2.59 $\pm$ 0.06 \\

\midrule
\rowcolor{purple!18}
\multicolumn{7}{@{}l}{\textcolor{purple!50!black}{$\blacktriangleright$}\ \textbf{\textsc{Target: Beijing (BJ)}}} \\
\midrule
\rowcolor{blue!5}
\hspace{0.5em}BJ$\to$BJ \textit{(w/o Alignment)}
& 95.41 $\pm$ 1.19 & 2.48 $\pm$ 0.06
& 531.27 $\pm$ 10.10 & 2.85 $\pm$ 0.06
& 154.57 $\pm$ 2.24 & 2.35 $\pm$ 0.28 \\
\hspace{0.5em}BJ$\to$BJ \textit{(Full \scot)}
& 96.05 $\pm$ 1.65 & 2.45 $\pm$ 0.07
& 534.79 $\pm$ 13.17 & 2.93 $\pm$ 0.07
& 155.58 $\pm$ 1.71 & 2.55 $\pm$ 0.06 \\

\bottomrule
\end{tabular}
\end{table}

\section{Additional Experiment Details}

\subsection{Additional Single-Source Results}
\label{app:additional_single_source}

Tables~\ref{tab:xa_cd} and~\ref{tab:bj_cd} report the remaining
single-source directions (XA$\leftrightarrow$CD and BJ$\leftrightarrow$CD).
\scot is best on every task, metric, and direction.


\begin{table*}[!t]
\centering
\caption{Single-source transfer results on XA$\leftrightarrow$CD. 
Lower is better. \best{Red}: best; \second{Blue}: runner-up. 
Gain row reports \scot's relative improvement over the strongest baseline.}
\label{tab:xa_cd}
\setlength{\tabcolsep}{5pt}
\renewcommand{\arraystretch}{1.2}
\resizebox{\textwidth}{!}{%
\begin{tabular}{@{}l|cccccc|cccccc@{}}
\toprule
\multirow{3}{*}{\textbf{Method}}
& \multicolumn{6}{c|}{\textbf{XA $\to$ CD}}
& \multicolumn{6}{c}{\textbf{CD $\to$ XA}} \\
\cmidrule(lr){2-7} \cmidrule(lr){8-13}
& \multicolumn{2}{c}{GDP} & \multicolumn{2}{c}{Population} & \multicolumn{2}{c|}{CO$_2$}
& \multicolumn{2}{c}{GDP} & \multicolumn{2}{c}{Population} & \multicolumn{2}{c}{CO$_2$} \\
\cmidrule(lr){2-3} \cmidrule(lr){4-5} \cmidrule(lr){6-7}
\cmidrule(lr){8-9} \cmidrule(lr){10-11} \cmidrule(lr){12-13}
& MAE & MAPE & MAE & MAPE & MAE & MAPE
& MAE & MAPE & MAE & MAPE & MAE & MAPE \\
\midrule
Non-Alignment
& 323.90 & 11.49 & 908.18 & 4.73 & 237.73 & 12.91
& 232.86 & \second{3.74} & 748.16 & 5.22 & 179.25 & \second{1.67} \\
\rowcolor{gray!06}
RP
& 191.11 & 11.64 & 644.22 & 3.32 & 144.87 & 9.18
& 194.67 & 6.57 & 671.30 & 6.41 & 145.51 & 2.30 \\
HBP
& 188.61 & 13.42 & 660.01 & 4.26 & 135.14 & 11.59
& 199.79 & 7.06 & 697.44 & 6.93 & 146.76 & 2.47 \\
\rowcolor{gray!06}
HSA
& 184.81 & 11.95 & 619.72 & 3.57 & 137.29 & 9.92
& 202.01 & 7.27 & 708.80 & 7.14 & 147.27 & 2.54 \\
MMD
& 201.60 & \second{10.49} & 645.42 & 3.15 & 162.21 & 8.83
& 176.00 & 5.26 & 608.58 & 5.15 & 139.15 & 1.83 \\
\rowcolor{gray!06}
Adv
& 194.71 & 13.68 & 645.59 & 3.98 & \second{124.98} & 10.01
& 215.47 & 8.54 & 785.78 & 8.50 & 150.10 & 2.99 \\
CrossTReS
& 183.22 & 12.82 & 712.74 & 4.55 & 159.86 & 9.26
& 180.31 & 5.09 & 630.66 & 5.37 & 143.58 & 1.92 \\
\rowcolor{gray!06}
CoRE
& \second{174.28} & 10.55 & \second{615.97} & \second{2.96} & 128.77 & \second{8.82}
& \second{174.52} & 5.22 & \second{600.45} & \second{4.98} & \second{138.12} & 1.78 \\
\midrule
\rowcolor{red!8}
\textbf{\scot (Ours)}
& \best{165.88} & \best{7.67} & \best{575.43} & \best{2.35} & \best{114.68} & \best{7.83}
& \best{158.95} & \best{3.12} & \best{538.23} & \best{3.37} & \best{130.04} & \best{1.24} \\
\rowcolor{blue!6}
\textit{$\Delta$ vs.\ best baseline}
& \cellcolor{blue!6}+4.8\% & \cellcolor{blue!18}\textbf{+26.9\%}
& \cellcolor{blue!8}+6.6\% & \cellcolor{blue!16}\textbf{+20.6\%}
& \cellcolor{blue!10}\textbf{+8.2\%} & \cellcolor{blue!12}\textbf{+11.2\%}
& \cellcolor{blue!10}\textbf{+8.9\%} & \cellcolor{blue!12}\textbf{+16.6\%}
& \cellcolor{blue!10}\textbf{+10.4\%} & \cellcolor{blue!22}\textbf{+32.3\%}
& \cellcolor{blue!8}+5.8\% & \cellcolor{blue!18}\textbf{+25.7\%} \\
\bottomrule
\end{tabular}}
\end{table*}

\begin{table*}[!t]
\centering
\caption{Single-source transfer results on BJ$\leftrightarrow$CD. 
Lower is better. \best{Red}: best; \second{Blue}: runner-up. 
Gain row reports \scot's relative improvement over the strongest baseline.}
\label{tab:bj_cd}
\setlength{\tabcolsep}{5pt}
\renewcommand{\arraystretch}{1.2}
\resizebox{\textwidth}{!}{%
\begin{tabular}{@{}l|cccccc|cccccc@{}}
\toprule
\multirow{3}{*}{\textbf{Method}}
& \multicolumn{6}{c|}{\textbf{BJ $\to$ CD}}
& \multicolumn{6}{c}{\textbf{CD $\to$ BJ}} \\
\cmidrule(lr){2-7} \cmidrule(lr){8-13}
& \multicolumn{2}{c}{GDP} & \multicolumn{2}{c}{Population} & \multicolumn{2}{c|}{CO$_2$}
& \multicolumn{2}{c}{GDP} & \multicolumn{2}{c}{Population} & \multicolumn{2}{c}{CO$_2$} \\
\cmidrule(lr){2-3} \cmidrule(lr){4-5} \cmidrule(lr){6-7}
\cmidrule(lr){8-9} \cmidrule(lr){10-11} \cmidrule(lr){12-13}
& MAE & MAPE & MAE & MAPE & MAE & MAPE
& MAE & MAPE & MAE & MAPE & MAE & MAPE \\
\midrule
Non-Alignment
& 191.36 & \second{5.31} & 889.37 & 3.97 & 192.64 & 15.77
& 163.93 & 7.04 & 805.78 & \second{5.07} & 192.96 & 1.98 \\
\rowcolor{gray!06}
RP
& 158.57 & 8.71 & 726.42 & 5.39 & 175.21 & 14.49
& 150.56 & 6.63 & 697.53 & 6.64 & 156.97 & 2.03 \\
HBP
& 155.05 & 8.30 & 698.75 & 4.80 & 172.65 & 13.22
& 166.30 & 7.29 & 715.27 & 6.84 & 154.77 & 1.79 \\
\rowcolor{gray!06}
HSA
& \second{147.74} & 6.87 & 667.95 & 4.20 & 160.10 & 11.55
& 148.87 & 6.59 & 648.28 & 5.75 & 166.55 & 2.64 \\
MMD
& 154.96 & 8.29 & 706.31 & 4.98 & 170.85 & 12.15
& \second{130.41} & \second{4.94} & \second{632.71} & 5.08 & \second{151.57} & \second{1.75} \\
\rowcolor{gray!06}
Adv
& 178.84 & 11.91 & 861.43 & 7.44 & 226.49 & 19.88
& 162.20 & 7.42 & 701.22 & 6.18 & 184.74 & 3.20 \\
CrossTReS
& 159.73 & 6.86 & \second{654.14} & \second{3.45} & 137.63 & 9.83
& 140.37 & 5.58 & 686.19 & 5.66 & 198.34 & 2.39 \\
\rowcolor{gray!06}
CoRE
& 150.51 & 7.57 & 680.81 & 4.34 & \second{126.66} & \second{9.66}
& 159.00 & 6.81 & 673.17 & 5.96 & 163.39 & 2.51 \\
\midrule
\rowcolor{red!8}
\textbf{\scot (Ours)}
& \best{135.63} & \best{3.55} & \best{597.80} & \best{2.38} & \best{121.21} & \best{8.94}
& \best{118.48} & \best{3.41} & \best{580.95} & \best{2.74} & \best{148.50} & \best{1.54} \\
\rowcolor{blue!6}
\textit{$\Delta$ vs.\ best baseline}
& \cellcolor{blue!10}\textbf{+8.2\%} & \cellcolor{blue!18}\textbf{+33.1\%}
& \cellcolor{blue!10}\textbf{+8.6\%} & \cellcolor{blue!18}\textbf{+31.0\%}
& \cellcolor{blue!6}+4.3\% & \cellcolor{blue!8}+7.5\%
& \cellcolor{blue!10}\textbf{+9.1\%} & \cellcolor{blue!18}\textbf{+31.0\%}
& \cellcolor{blue!10}\textbf{+8.2\%} & \cellcolor{blue!22}\textbf{+46.0\%}
& \cellcolor{blue!6}+2.0\% & \cellcolor{blue!12}\textbf{+12.0\%} \\
\bottomrule
\end{tabular}}
\end{table*}

\subsection{Additional Results on XA$\leftrightarrow$BJ (4 Random Seeds)}
\label{app:seed4}

To verify that \scot's gains are not artifacts of seed selection, we 
report single-source transfer in both directions (XA$\to$BJ and 
BJ$\to$XA) under \emph{four random seeds}, with mean $\pm$ standard 
deviation of MAE and MAPE for GDP, population, and CO$_2$ 
(Tables~\ref{tab:xa2bj_seed4},~\ref{tab:bj2xa_seed4}). \scot achieves 
\textbf{the lowest mean error on every metric in both directions}, with 
\emph{the smallest standard deviation} on most entries, indicating that 
the improvements are both consistent and statistically robust rather 
than driven by favorable seeds.

\begin{table*}[!t]
\centering
\caption{XA$\to$BJ results averaged over 4 random seeds (mean $\pm$ std). 
Lower is better. \best{Red}: best; \second{Blue}: runner-up. 
Gain row reports \scot's relative improvement over the strongest baseline.}
\label{tab:xa2bj_seed4}
\setlength{\tabcolsep}{5pt}
\renewcommand{\arraystretch}{1.2}
\resizebox{\textwidth}{!}{%
\begin{tabular}{@{}l|cc|cc|cc@{}}
\toprule
\multirow{2}{*}{\textbf{Method}}
& \multicolumn{2}{c|}{\textbf{GDP}}
& \multicolumn{2}{c|}{\textbf{Population}}
& \multicolumn{2}{c}{\textbf{CO$_2$}}\\
\cmidrule(lr){2-3}\cmidrule(lr){4-5}\cmidrule(lr){6-7}
& MAE$\downarrow$ & MAPE$\downarrow$
& MAE$\downarrow$ & MAPE$\downarrow$
& MAE$\downarrow$ & MAPE$\downarrow$ \\
\midrule
Non-Alignment
& 276.33 $\pm$ 10.63 & 9.46 $\pm$ 1.98
& 958.52 $\pm$ 31.66 & 5.99 $\pm$ 1.82
& 278.25 $\pm$ 9.95 & 5.39 $\pm$ 0.75 \\
\rowcolor{gray!06}
RP
& 192.05 $\pm$ 26.62 & 6.65 $\pm$ 1.48
& 676.21 $\pm$ 28.24 & 4.00 $\pm$ 0.66
& 194.70 $\pm$ 11.88 & 3.71 $\pm$ 0.22 \\
HBP
& 186.21 $\pm$ 10.14 & 8.17 $\pm$ 0.88
& 664.15 $\pm$ 24.83 & 4.68 $\pm$ 0.89
& 188.14 $\pm$ 5.39 & 3.99 $\pm$ 0.20 \\
\rowcolor{gray!06}
HSA
& 179.96 $\pm$ 17.62 & 7.30 $\pm$ 1.00
& 631.69 $\pm$ 27.93 & 4.64 $\pm$ 1.19
& 180.21 $\pm$ 8.29 & 3.57 $\pm$ 0.64 \\
MMD
& 162.63 $\pm$ 16.27 & \second{5.93 $\pm$ 0.90}
& \second{596.60 $\pm$ 21.41} & \second{3.63 $\pm$ 0.66}
& \second{169.99 $\pm$ 7.14} & \second{2.91 $\pm$ 0.46} \\
\rowcolor{gray!06}
Adv
& 200.33 $\pm$ 13.90 & 8.98 $\pm$ 0.60
& 694.64 $\pm$ 9.39 & 6.15 $\pm$ 1.04
& 199.99 $\pm$ 2.61 & 4.63 $\pm$ 0.37 \\
CrossTReS
& 194.87 $\pm$ 28.96 & 7.28 $\pm$ 0.90
& 629.37 $\pm$ 22.46 & 4.29 $\pm$ 0.18
& 182.88 $\pm$ 9.74 & 3.59 $\pm$ 0.28 \\
\rowcolor{gray!06}
CoRE
& \second{159.53 $\pm$ 14.64} & 6.19 $\pm$ 1.65
& 607.79 $\pm$ 39.24 & 4.19 $\pm$ 1.13
& 170.55 $\pm$ 11.99 & 3.12 $\pm$ 0.67 \\
\midrule
\rowcolor{red!8}
\textbf{\scot (Ours)}
& \best{120.25 $\pm$ 7.30} & \best{3.59 $\pm$ 0.48}
& \best{527.04 $\pm$ 6.38} & \best{2.17 $\pm$ 0.23}
& \best{149.20 $\pm$ 1.58} & \best{1.80 $\pm$ 0.17} \\
\rowcolor{blue!6}
\textit{$\Delta$ vs.\ best baseline}
& \cellcolor{blue!18}\textbf{+24.6\%} & \cellcolor{blue!22}\textbf{+39.5\%}
& \cellcolor{blue!12}\textbf{+11.7\%} & \cellcolor{blue!22}\textbf{+40.2\%}
& \cellcolor{blue!12}\textbf{+12.2\%} & \cellcolor{blue!22}\textbf{+38.1\%} \\
\bottomrule
\end{tabular}}
\end{table*}

\begin{table*}[!t]
\centering
\caption{BJ$\to$XA results averaged over 4 random seeds (mean $\pm$ std). 
Lower is better. \best{Red}: best; \second{Blue}: runner-up. 
Gain row reports \scot's relative improvement over the strongest baseline.}
\label{tab:bj2xa_seed4}
\setlength{\tabcolsep}{5pt}
\renewcommand{\arraystretch}{1.2}
\resizebox{\textwidth}{!}{%
\begin{tabular}{@{}l|cc|cc|cc@{}}
\toprule
\multirow{2}{*}{\textbf{Method}}
& \multicolumn{2}{c|}{\textbf{GDP}}
& \multicolumn{2}{c|}{\textbf{Population}}
& \multicolumn{2}{c}{\textbf{CO$_2$}}\\
\cmidrule(lr){2-3}\cmidrule(lr){4-5}\cmidrule(lr){6-7}
& MAE$\downarrow$ & MAPE$\downarrow$
& MAE$\downarrow$ & MAPE$\downarrow$
& MAE$\downarrow$ & MAPE$\downarrow$ \\
\midrule
Non-Alignment
& 220.31 $\pm$ 22.10 & 7.91 $\pm$ 1.43
& 975.83 $\pm$ 65.55 & 9.63 $\pm$ 1.27
& 276.84 $\pm$ 18.56 & 9.70 $\pm$ 1.14 \\
\rowcolor{gray!06}
RP
& 175.21 $\pm$ 4.15 & 4.60 $\pm$ 0.53
& 671.08 $\pm$ 11.95 & 6.37 $\pm$ 0.22
& 191.56 $\pm$ 3.33 & 6.30 $\pm$ 0.24 \\
HBP
& 180.60 $\pm$ 13.16 & 2.97 $\pm$ 0.82
& 629.09 $\pm$ 5.07 & 4.95 $\pm$ 0.39
& 179.17 $\pm$ 2.76 & 4.91 $\pm$ 0.45 \\
\rowcolor{gray!06}
HSA
& 176.42 $\pm$ 10.52 & 3.66 $\pm$ 1.35
& 649.08 $\pm$ 16.82 & 5.66 $\pm$ 0.62
& 185.72 $\pm$ 4.29 & 5.62 $\pm$ 0.57 \\
MMD
& 180.71 $\pm$ 5.39 & \second{2.25 $\pm$ 0.22}
& \second{500.12 $\pm$ 27.35} & \second{1.93 $\pm$ 0.09}
& \second{141.24 $\pm$ 4.96} & \second{1.91 $\pm$ 0.08} \\
\rowcolor{gray!06}
Adv
& 192.06 $\pm$ 6.28 & 6.15 $\pm$ 0.34
& 778.54 $\pm$ 30.34 & 8.53 $\pm$ 0.62
& 212.72 $\pm$ 9.74 & 7.84 $\pm$ 0.62 \\
CrossTReS
& 165.18 $\pm$ 3.45 & 3.98 $\pm$ 0.22
& 627.96 $\pm$ 8.88 & 5.18 $\pm$ 0.32
& 179.42 $\pm$ 2.30 & 5.16 $\pm$ 0.29 \\
\rowcolor{gray!06}
CoRE
& \second{162.64 $\pm$ 5.07} & 2.80 $\pm$ 0.80
& 576.95 $\pm$ 25.66 & 3.43 $\pm$ 1.15
& 164.26 $\pm$ 8.68 & 3.44 $\pm$ 1.18 \\
\midrule
\rowcolor{red!8}
\textbf{\scot (Ours)}
& \best{160.21 $\pm$ 3.53} & \best{1.87 $\pm$ 0.18}
& \best{450.14 $\pm$ 2.81} & \best{1.73 $\pm$ 0.12}
& \best{127.79 $\pm$ 1.08} & \best{1.78 $\pm$ 0.10} \\
\rowcolor{blue!6}
\textit{$\Delta$ vs.\ best baseline}
& \cellcolor{blue!6}+1.8\% & \cellcolor{blue!12}\textbf{+16.9\%}
& \cellcolor{blue!10}\textbf{+10.0\%} & \cellcolor{blue!10}\textbf{+10.4\%}
& \cellcolor{blue!10}\textbf{+9.5\%} & \cellcolor{blue!8}+6.8\% \\
\bottomrule
\end{tabular}}
\end{table*}

\subsection{Additional Multi-Source Results (Target: CD)}
\label{app:multisource_cd}

Table~\ref{tab:multi_source_cd_app} reports multi-source transfer 
with Chengdu (CD) as target. \scot achieves the \textbf{best performance 
on all three tasks} under both MAE and MAPE, with the largest gains on 
\emph{population and CO}$_2$. When existing single-source alignment and 
distribution-matching strategies are naively extended to multiple sources, 
they yield only limited improvements. This confirms that effective 
multi-source transfer is not a matter of summing per-source losses, but 
requires the \emph{coordinated aggregation} provided by \scot's shared hub.

\begin{table}[H]
\centering
\footnotesize
\caption{Multi-source transfer results (Target: CD). Lower is better. 
\best{Red}: best; \second{Blue}: runner-up. Gain row reports \scot's 
relative improvement over the strongest baseline.}
\label{tab:multi_source_cd_app}
\vspace{1em}
\setlength{\tabcolsep}{6pt}
\renewcommand{\arraystretch}{1.15}
\begin{tabular}{@{}l|cc|cc|cc@{}}
\toprule
\multirow{2}{*}{\textbf{Method}}
& \multicolumn{2}{c|}{\textbf{GDP}}
& \multicolumn{2}{c|}{\textbf{Population}}
& \multicolumn{2}{c}{\textbf{CO$_2$}}\\
\cmidrule(lr){2-3}\cmidrule(lr){4-5}\cmidrule(lr){6-7}
& MAE$\downarrow$ & MAPE$\downarrow$
& MAE$\downarrow$ & MAPE$\downarrow$
& MAE$\downarrow$ & MAPE$\downarrow$ \\
\midrule
RP         & 187.13 & 13.31 & 630.44 & 3.62 & 167.76 & 14.05 \\
\rowcolor{gray!06}
HBP        & 176.35 & 11.11 & 631.26 & 3.69 & 125.97 & 9.96 \\
HSA        & 180.71 & 8.18  & 651.06 & \second{3.41} & 159.86 & 11.79 \\
\rowcolor{gray!06}
MMD        & 163.82 & 5.07  & 639.32 & 3.93 & 145.74 & 10.83 \\
Adv        & 183.26 & 12.43 & 687.72 & 4.73 & 156.37 & 12.92 \\
\rowcolor{gray!06}
CrossTReS  & 167.53 & 10.26 & 668.00 & 4.37 & 139.22 & 11.84 \\
CoRE       & \second{156.10} & \second{4.40} & \second{621.11} & 3.44 & \second{121.33} & \second{9.57} \\
\midrule
\rowcolor{red!8}
\textbf{\scot (Ours)}
& \best{133.94} & \best{3.82} & \best{546.82} & \best{2.43} & \best{98.43} & \best{5.10} \\
\rowcolor{blue!6}
\textit{$\Delta$ vs.\ best baseline}
& \cellcolor{blue!12}\textbf{+14.1\%} & \cellcolor{blue!12}\textbf{+13.1\%}
& \cellcolor{blue!12}\textbf{+11.9\%} & \cellcolor{blue!18}\textbf{+28.7\%}
& \cellcolor{blue!16}\textbf{+18.8\%} & \cellcolor{blue!22}\textbf{+46.7\%} \\
\bottomrule
\end{tabular}
\end{table}

\subsection{Empirical Check of the Proposition}
\label{app:theory_empirical_check}

To complement Proposition~\ref{prop:contrastive_mae}, we test its qualitative 
mechanism by relating target error $y$ to both alignment terms in a 
\emph{joint regression}:
\[
y \;=\; \beta_0 + \beta_1\,\mathcal{L}_{\mathrm{Con}} + 
\beta_2\,\mathcal{L}_{\mathrm{OT}} + \varepsilon.
\]
This isolates the effect of $\mathcal{L}_{\mathrm{Con}}$ on target error 
\emph{after controlling for $\mathcal{L}_{\mathrm{OT}}$}, addressing the 
concern that the two alignment terms might be confounded. As shown in 
Table~\ref{tab:theory_empirical_check}, the standardized coefficient on 
$\mathcal{L}_{\mathrm{Con}}$ is \textbf{0.77}, and its partial 
Pearson/Spearman correlations remain at \textbf{0.95/0.93} after 
controlling for $\mathcal{L}_{\mathrm{OT}}$. This empirically supports 
the proposition's qualitative message: \emph{stronger contrastive alignment 
is closely associated with lower target error, independently of OT 
alignment quality}.

\begin{table}[H]
\centering
\footnotesize
\caption{Empirical verification of Proposition~\ref{prop:contrastive_mae}. 
The standardized coefficient on $\mathcal{L}_{\mathrm{Con}}$ remains 
\textbf{strongly positive after controlling for $\mathcal{L}_{\mathrm{OT}}$}, 
and partial Pearson/Spearman correlations both exceed $0.93$ with 
$p<0.001$.}
\vspace{1em}
\label{tab:theory_empirical_check}
\setlength{\tabcolsep}{10pt}
\renewcommand{\arraystretch}{1.3}
\begin{tabular}{@{}cc|ccc@{}}
\toprule
\multicolumn{2}{c|}{\textbf{Joint OLS:} $y \sim \mathcal{L}_{\mathrm{Con}} + \mathcal{L}_{\mathrm{OT}}$}
& \multicolumn{3}{c}{\textbf{Partial correlation with $\mathcal{L}_{\mathrm{Con}}$}} \\
& & \multicolumn{3}{c}{\footnotesize\textit{(controlling for $\mathcal{L}_{\mathrm{OT}}$)}} \\
\cmidrule(lr){1-2}\cmidrule(lr){3-5}
Std.\ coef.\ on $\mathcal{L}_{\mathrm{Con}}$ & Adj.\ $R^2$
& Pearson $r$ & Spearman $\rho$ & $p$-value \\
\midrule
\rowcolor{yellow!6}
\cellcolor{yellow!18}\textbf{0.77}
& \cellcolor{yellow!18}\textbf{0.94}
& \cellcolor{yellow!16}\textbf{0.95}
& \cellcolor{yellow!16}\textbf{0.93}
& \cellcolor{yellow!12}\textbf{$<\!0.001$} \\
\bottomrule
\end{tabular}
\end{table}

\subsection{Additional OT Coupling Diagnostics}
\label{app:ot_diagnostics}

The log-scale plot in Fig.~\ref{fig:alignment_diagnostics} visually
exaggerates tiny tail entries, so we quantify the coupling directly using
row-normalized profiles $q(j\mid i)=P_{ij}/\sum_{j'}P_{ij'}$. We ask two
questions: is each row's mass concentrated, and do different source regions
share the same target pattern?

\begin{table*}[!t]
\centering
\caption{Selectivity and region specificity of the XA$\to$BJ coupling.
Coverage is the number of target regions required to contain the specified
fraction of transported mass.}
\label{tab:ot_selectivity}
\setlength{\tabcolsep}{7pt}
\renewcommand{\arraystretch}{1.22}
\resizebox{\textwidth}{!}{%
\begin{tabular}{@{}lc|lc|lc@{}}
\toprule

\multicolumn{2}{c|}{
\cellcolor{blue!10}
\textbf{Representative source region}
}
&
\multicolumn{2}{c|}{
\cellcolor{red!10}
\textbf{Average over all source regions}
}
&
\multicolumn{2}{c}{
\cellcolor{blue!10}
\textbf{Across-row diversity}
}
\\

\cmidrule(lr){1-2}
\cmidrule(lr){3-4}
\cmidrule(lr){5-6}

\textbf{Diagnostic}
& \textbf{Result}
& \textbf{Diagnostic}
& \textbf{Result}
& \textbf{Diagnostic}
& \textbf{Result}
\\
\midrule

Mass in top-10 targets
& \cellcolor{blue!16}\textbf{32.5\%}
& Targets covering 50\% mass
& \cellcolor{red!14}\textbf{35}
& Mean cosine similarity
& \cellcolor{blue!16}\textbf{0.082}
\\

\rowcolor{gray!06}
Mass in top-20 targets
& \cellcolor{blue!20}\textbf{50.8\%}
& Targets covering 80\% mass
& \cellcolor{red!14}\textbf{76}
& Median cosine similarity
& \cellcolor{blue!20}\textbf{0.002}
\\

Targets covering 90\% mass
& \cellcolor{blue!14}\textbf{108}
& Targets covering 90\% mass
& \cellcolor{red!14}\textbf{112}
& Mean JS divergence
& \cellcolor{blue!18}\textbf{0.587}
\\

\bottomrule
\end{tabular}}
\end{table*}

\textbf{Mass is concentrated.} On average, 35, 76, and 112 of the 1,311
Beijing regions cover $50\%$, $80\%$, and $90\%$ of a row's mass; the
representative region in Fig.~\ref{fig:alignment_diagnostics} behaves
similarly (top-20 targets hold $50.8\%$).

\textbf{Rows are region-specific.} Pairwise cosine similarity between rows is
near zero (mean $0.082$, median $0.002$) and the mean JS divergence is
$0.587$, so source regions do not collapse onto a shared target pattern.

In short, the coupling is dense in support but selective in mass.

\section{Ablation Study}
\label{app:ablation_study}

\subsection{Ablation Study for Single-Source SCOT}
\label{app:single_source_ablation}

We remove each of $\mathcal{L}_{\mathrm{OT}}$, $\mathcal{L}_{\mathrm{Con}}$,
and $\mathcal{L}_{\mathrm{Rec}}$ in turn.

\paragraph{Transfer performance.}
Across six directions (Fig.~\ref{fig:ablation_transfer}), removing
$\mathcal{L}_{\mathrm{OT}}$ causes the largest drop, removing
$\mathcal{L}_{\mathrm{Con}}$ consistently hurts, and removing
$\mathcal{L}_{\mathrm{Rec}}$ has a smaller but noticeable effect.

\begin{figure}[H]
    \centering
    \includegraphics[width=0.8\textwidth]{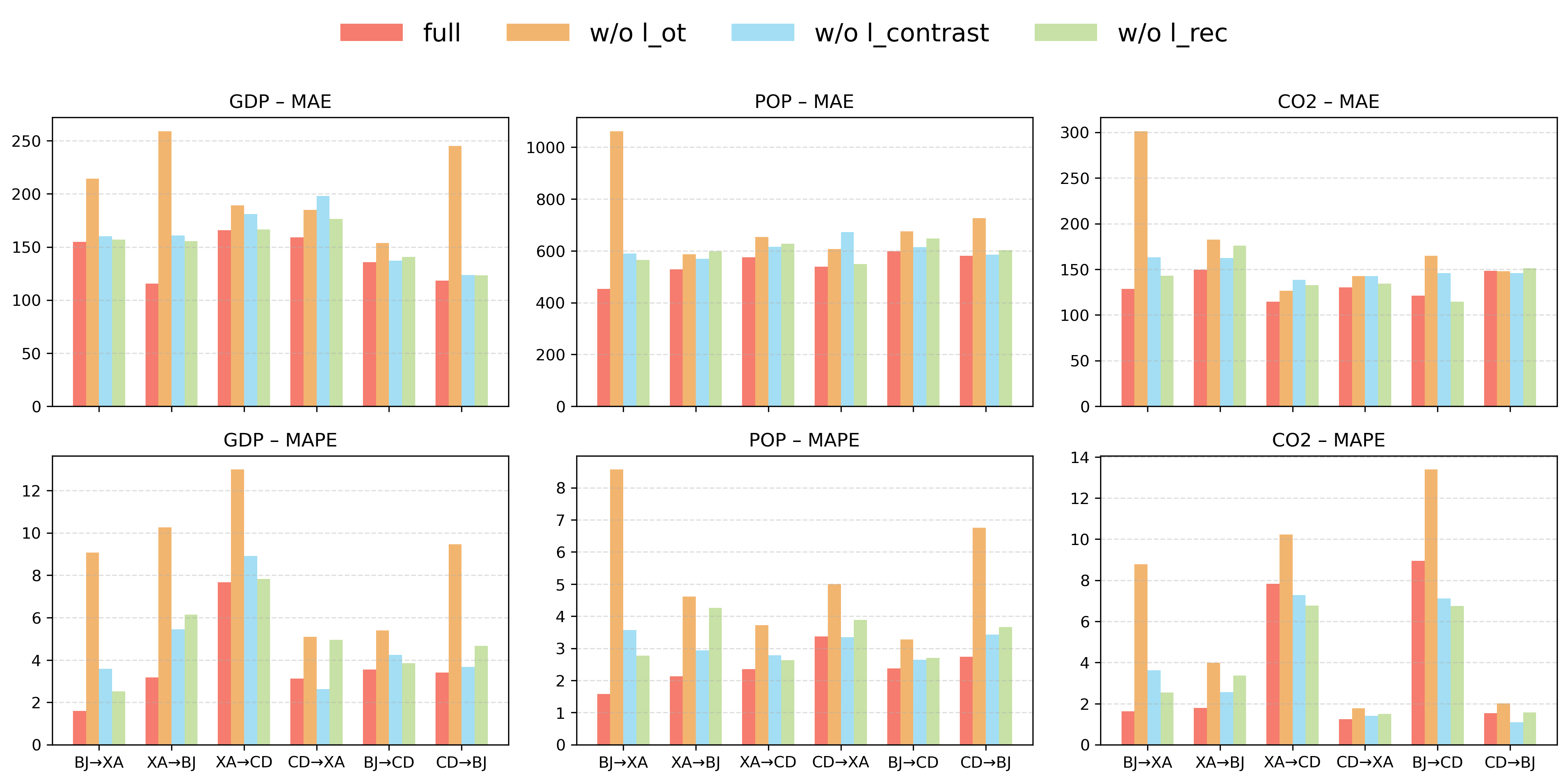}
    \caption{
    \textbf{Single-source ablation results.}
    MAE (top) and MAPE (bottom) are reported for GDP, population, and
    CO$_2$ prediction across six transfer directions:
    BJ$\rightarrow$XA, XA$\rightarrow$BJ, XA$\rightarrow$CD,
    CD$\rightarrow$XA, BJ$\rightarrow$CD, and CD$\rightarrow$BJ.
    }
    \label{fig:ablation_transfer}
\end{figure}

\paragraph{Effect on the coupling.}
To see how $\mathcal{L}_{\mathrm{OT}}$ and $\mathcal{L}_{\mathrm{Con}}$
shape the coupling, we measure each row of
$q(j\mid i)=P_{ij}/\sum_{j'}P_{ij'}$ by its \emph{top-1 mass} (largest
entry) and \emph{normalized entropy} (spread), averaged over source regions
(Table~\ref{tab:loss_role_coupling}).

\begin{table}[H]
\centering
\small
\caption{Coupling diagnostics for XA$\rightarrow$BJ
(mean $\pm$ std over ablation runs). Values in parentheses are ratios
relative to full \scot.}
\vspace{0.7em}
\label{tab:loss_role_coupling}
\setlength{\tabcolsep}{6pt}
\renewcommand{\arraystretch}{1.15}

\begin{tabular}{@{}l|cc@{}}
\toprule
\textbf{Variant}
& \textbf{Top-1 mass}
& \textbf{Normalized entropy} \\
\midrule

\rowcolor{red!8}
\textbf{\scot}
& \cellcolor{red!14}\textbf{$0.015 \pm 0.002$}
& \cellcolor{red!14}\textbf{$0.745 \pm 0.015$} \\

\rowcolor{gray!06}
w/o $\mathcal{L}_{\mathrm{Con}}$
& \cellcolor{blue!12}
$0.0011 \pm 0.0001\;(0.07\times)$
& \cellcolor{blue!16}
$0.9999 \pm 0.0001\;(1.34\times)$ \\

w/o $\mathcal{L}_{\mathrm{OT}}$
& \cellcolor{blue!16}
$0.205 \pm 0.025\;(13.7\times)$
& \cellcolor{blue!12}
$0.815 \pm 0.020\;(1.09\times)$ \\

\bottomrule
\end{tabular}
\end{table}

Without $\mathcal{L}_{\mathrm{Con}}$, rows become nearly uniform
(entropy $\approx 1$): OT alone gives little local discrimination. Without
$\mathcal{L}_{\mathrm{OT}}$, each row puts large mass on one target but
keeps a diffuse tail (entropy $0.815$): contrastive learning alone does not
yield a structured global allocation. The two terms are thus complementary.

\subsection{Ablation: Hub vs.\ Pairwise OT (Multi-source)}
\label{app:ablation_hub_vs_pairwise}

We compare our hub, which aligns all cities to $K$ shared prototypes, with a
pairwise variant that aligns each source to the target separately and
combines the two objectives with a learnable global gate
(Table~\ref{tab:ablation_hub_pairwise_multi}).

\begin{table}[h]
\centering
\footnotesize
\caption{Multi-source ablation: shared-prototype \textbf{Hub} 
vs.\ pairwise OT with global gating (\textbf{No Hub}), across three 
targets (mean $\pm$ std over 4 seeds). \best{Red}: better in each pair. 
Lower is better.}
\label{tab:ablation_hub_pairwise_multi}
\setlength{\tabcolsep}{5pt}
\renewcommand{\arraystretch}{1.2}
\begin{tabular}{@{}l|cc|cc|cc@{}}
\toprule
\multirow{2}{*}{\textbf{Target / Variant}}
& \multicolumn{2}{c|}{\textbf{GDP}}
& \multicolumn{2}{c|}{\textbf{Population}}
& \multicolumn{2}{c}{\textbf{CO$_2$}} \\
\cmidrule(lr){2-3}\cmidrule(lr){4-5}\cmidrule(lr){6-7}
& MAE$\downarrow$ & MAPE$\downarrow$
& MAE$\downarrow$ & MAPE$\downarrow$
& MAE$\downarrow$ & MAPE$\downarrow$ \\
\midrule

\rowcolor{cyan!18}
\multicolumn{7}{@{}l}{\textcolor{cyan!50!black}{$\blacktriangleright$}\ \textbf{\textsc{Target: Xi'an (XA)}}} \\
\midrule
\rowcolor{cyan!8}
\hspace{0.5em}\textbf{Ours (Hub)}
& \best{154.49 $\pm$ 2.10} & \best{2.12 $\pm$ 0.31}
& \best{467.54 $\pm$ 17.42} & \best{2.23 $\pm$ 0.25}
& \best{133.58 $\pm$ 4.64} & \best{1.51 $\pm$ 0.18} \\
\hspace{0.5em}No Hub (Pairwise)
& 157.45 $\pm$ 3.31 & 2.52 $\pm$ 0.48
& 511.37 $\pm$ 23.34 & 2.98 $\pm$ 0.64
& 135.56 $\pm$ 6.90 & 1.71 $\pm$ 0.35 \\

\midrule

\rowcolor{orange!18}
\multicolumn{7}{@{}l}{\textcolor{orange!60!black}{$\blacktriangleright$}\ \textbf{\textsc{Target: Chengdu (CD)}}} \\
\midrule
\rowcolor{orange!8}
\hspace{0.5em}\textbf{Ours (Hub)}
& \best{143.91 $\pm$ 6.84} & \best{4.54 $\pm$ 0.49}
& \best{565.65 $\pm$ 14.55} & 2.38 $\pm$ 0.09
& \best{102.20 $\pm$ 11.20} & 5.92 $\pm$ 0.67 \\
\hspace{0.5em}No Hub (Pairwise)
& 146.72 $\pm$ 7.48 & 6.00 $\pm$ 1.03
& 585.49 $\pm$ 15.77 & \best{2.17 $\pm$ 0.20}
& 105.96 $\pm$ 2.38 & \best{5.87 $\pm$ 0.47} \\

\midrule

\rowcolor{purple!18}
\multicolumn{7}{@{}l}{\textcolor{purple!60!black}{$\blacktriangleright$}\ \textbf{\textsc{Target: Beijing (BJ)}}} \\
\midrule
\rowcolor{purple!8}
\hspace{0.5em}\textbf{Ours (Hub)}
& \best{110.40 $\pm$ 5.34} & \best{3.30 $\pm$ 0.50}
& \best{533.11 $\pm$ 10.52} & \best{2.98 $\pm$ 0.74}
& \best{145.72 $\pm$ 3.58} & \best{1.46 $\pm$ 0.21} \\
\hspace{0.5em}No Hub (Pairwise)
& 140.86 $\pm$ 15.09 & 5.15 $\pm$ 1.18
& 580.37 $\pm$ 20.59 & 4.20 $\pm$ 0.55
& 152.83 $\pm$ 4.38 & 1.92 $\pm$ 0.27 \\

\bottomrule
\end{tabular}
\end{table}

The hub is better on 16 of 18 metrics and wins every MAE, with the largest
gain on Beijing GDP ($110.40$ vs.\ $140.86$). The pairwise variant is
slightly better only on two Chengdu MAPE entries.

\subsection{Ablation: Target-Induced Prototype Prior}
\label{subsec:ablationA_bt}

We compare three choices of the prototype marginal $\mathbf{b}$
(Eq.~\eqref{eq:bt_from_target}): \textbf{Uniform} ($1/K$, no target
signal), \textbf{Frozen} (computed from early target embeddings and then
fixed), and \textbf{Adaptive} (ours, updated during training).

\begin{table}[h]
\centering
\footnotesize
\caption{Ablation on the target-induced prototype marginal $\mathbf{b}$ 
(XA as target, 4 seeds, mean $\pm$ std). The \textbf{adaptive} prior 
wins on every metric. Lower is better.}
\vspace{1em}
\label{tab:ablationA_bt}
\setlength{\tabcolsep}{3pt}
\renewcommand{\arraystretch}{1.2}
\begin{tabular}{@{}l|cc|cc|cc@{}}
\toprule
\multirow{2}{*}{\textbf{Prior Variant}}
& \multicolumn{2}{c|}{\textbf{GDP}}
& \multicolumn{2}{c|}{\textbf{Population}}
& \multicolumn{2}{c}{\textbf{CO$_2$}} \\
\cmidrule(lr){2-3}\cmidrule(lr){4-5}\cmidrule(lr){6-7}
& MAE$\downarrow$ & MAPE$\downarrow$
& MAE$\downarrow$ & MAPE$\downarrow$
& MAE$\downarrow$ & MAPE$\downarrow$ \\
\midrule
Uniform {\scriptsize\textit{(no target)}}
& 186.33 $\pm$ 3.93 & 3.42 $\pm$ 0.14
& 575.81 $\pm$ 17.09 & 3.27 $\pm$ 0.33
& 156.56 $\pm$ 10.13 & 1.85 $\pm$ 0.18 \\
\rowcolor{gray!06}
Frozen {\scriptsize\textit{(fixed)}}
& 181.47 $\pm$ 3.16 & 4.79 $\pm$ 0.19
& 553.05 $\pm$ 8.68 & 3.81 $\pm$ 0.10
& 148.89 $\pm$ 0.67 & 2.83 $\pm$ 0.02 \\
\midrule
\rowcolor{red!8}
\textbf{Adaptive (Ours)}
& \best{154.49 $\pm$ 2.10} & \best{2.12 $\pm$ 0.31}
& \best{467.54 $\pm$ 17.42} & \best{2.23 $\pm$ 0.25}
& \best{133.58 $\pm$ 4.64} & \best{1.51 $\pm$ 0.18} \\
\rowcolor{blue!6}
\textit{$\Delta$ vs.\ Uniform}
& \cellcolor{blue!16}\textbf{+17.1\%} & \cellcolor{blue!22}\textbf{+38.0\%}
& \cellcolor{blue!16}\textbf{+18.8\%} & \cellcolor{blue!18}\textbf{+31.8\%}
& \cellcolor{blue!12}\textbf{+14.7\%} & \cellcolor{blue!12}\textbf{+18.4\%} \\
\bottomrule
\end{tabular}
\end{table}


\begin{figure}[t]
    \centering
    \begin{subfigure}[b]{0.32\linewidth}
        \centering
        \includegraphics[width=\linewidth]{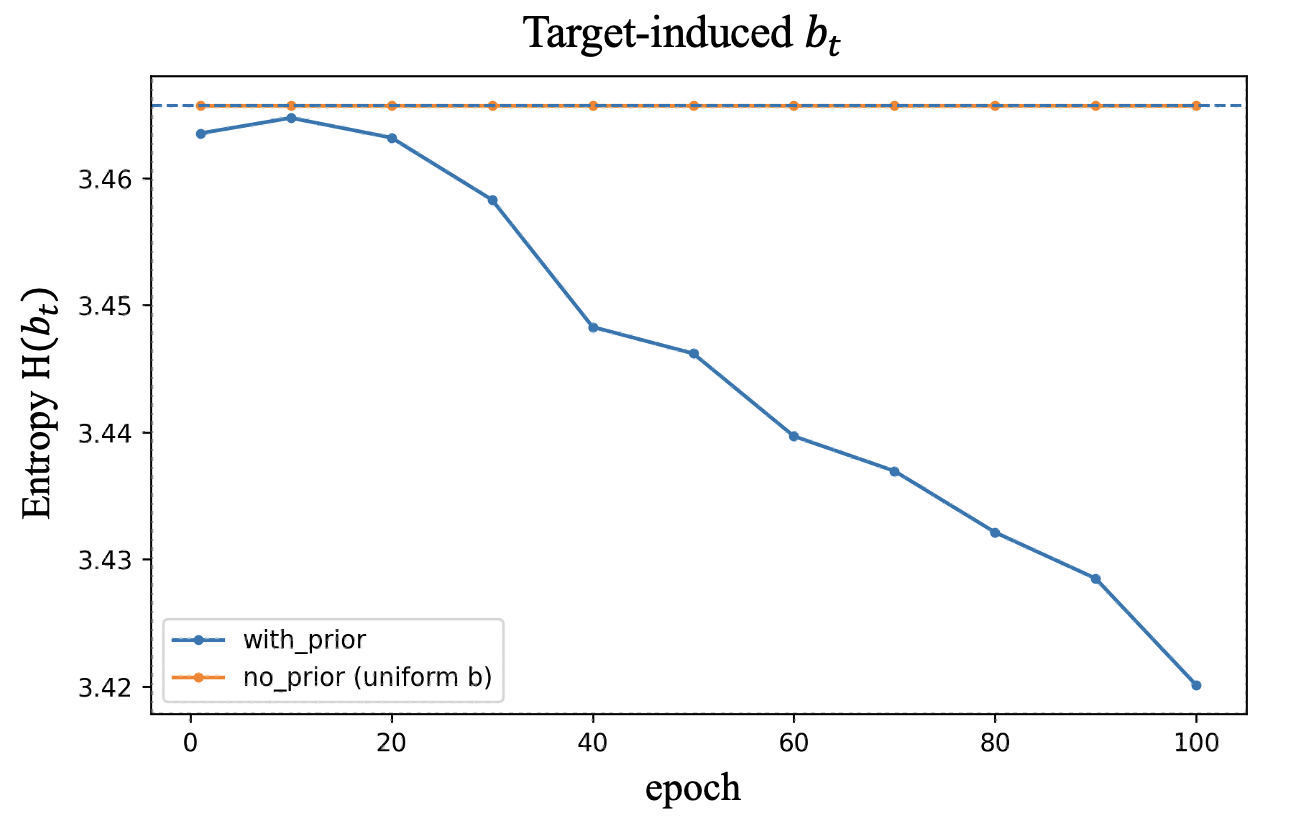}
        \caption{Entropy of $\mathbf{b}$}
        \label{fig:ablationA_bt_entropy}
    \end{subfigure}\hfill
    \begin{subfigure}[b]{0.32\linewidth}
        \centering
        \includegraphics[width=\linewidth]{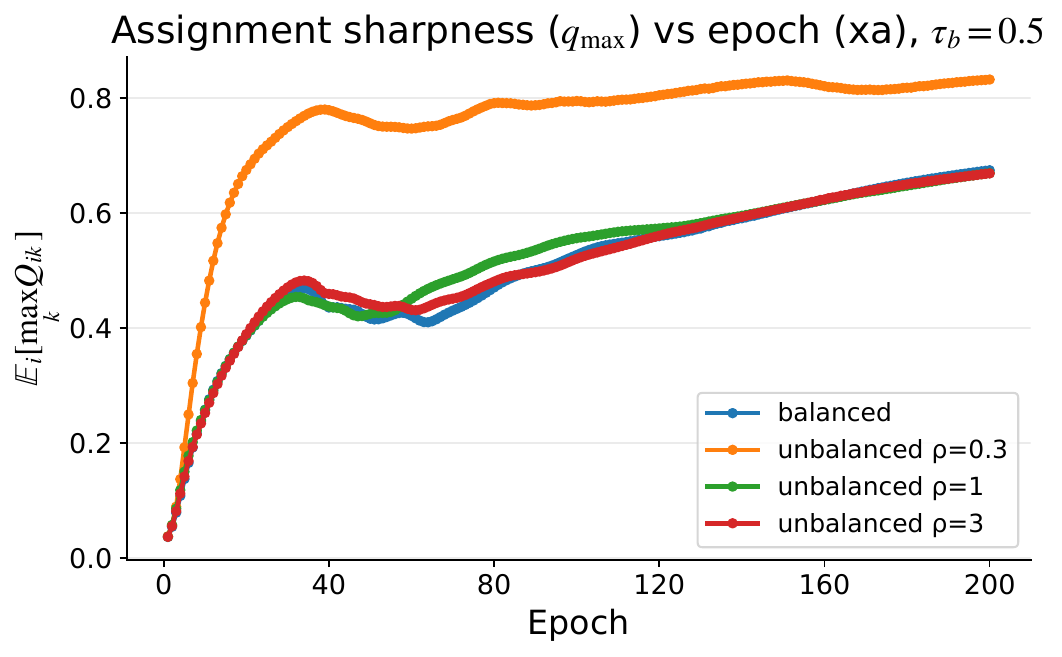}
        \caption{Assignment sharpness $q_{\max}$}
        \label{fig:uot_qmax}
    \end{subfigure}\hfill
    \begin{subfigure}[b]{0.32\linewidth}
        \centering
        \includegraphics[width=\linewidth]{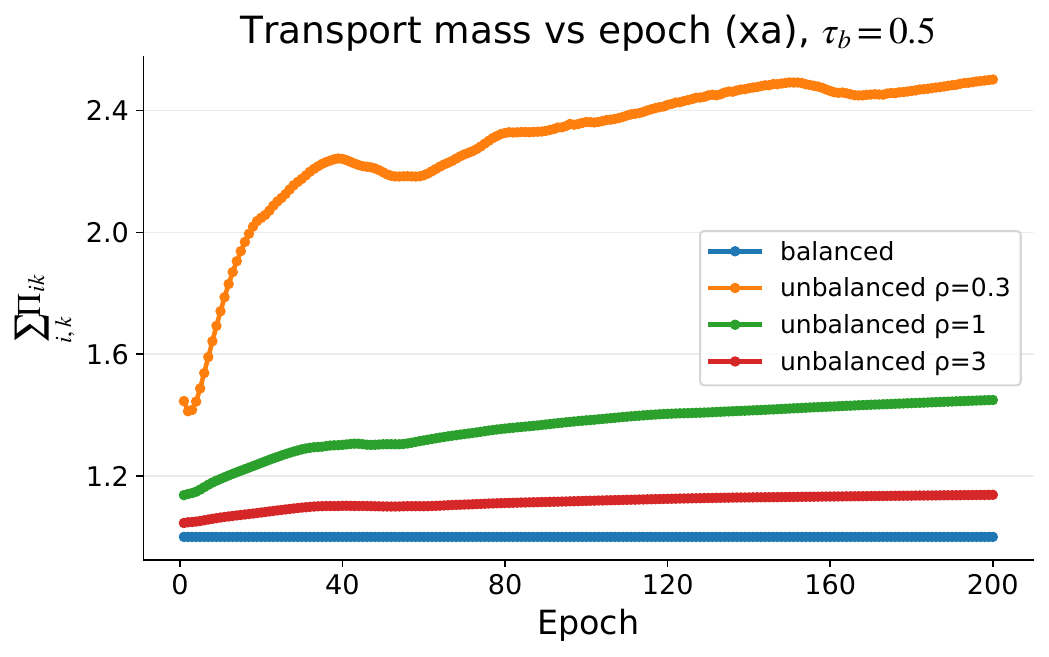}
        \caption{Transport mass $\sum_{i,k}\Pi_{ik}$}
        \label{fig:uot_mass}
    \end{subfigure}
    \caption{Hub diagnostics on CD,BJ$\to$XA. (a) Entropy of the prototype
    marginal $\mathbf{b}$ during training: it stays at $\log K$ under the
    uniform prior and decreases under the adaptive prior.
    (b, c) Unbalanced OT with small $\rho$ gives sharper assignments only by
    inflating the total mass above one ($\tau_b=0.5$).}
    \label{fig:hub_diagnostics}
\end{figure}

The adaptive prior is best on every metric (Table~\ref{tab:ablationA_bt}).
Fig.~\ref{fig:ablationA_bt_entropy} shows why: the uniform prior's entropy
stays at $\log K$, while the adaptive prior's entropy decreases steadily as
prototypes specialize toward target semantics. Fixing $\mathbf{b}$ early
lowers MAE relative to Uniform but raises MAPE, so the prior must evolve
with the target encoder.

\subsection{OT Formulation}
\label{app:ot_formulation}

We test whether relaxing balanced OT helps, first in single-source alignment
and then in the multi-source hub.

\paragraph{Single source.}
\label{app:ot_variants}
We replace balanced OT with unbalanced OT~\citep{chizat2018scaling},
partial OT~\citep{chapel2020partial}, or the Sinkhorn
divergence~\citep{genevay2018learning,feydy2019interpolating}, keeping
everything else fixed (Table~\ref{tab:ot_variants}). Balanced OT gives the
lowest MAE in every case, and unbalanced OT, the only variant that relaxes
the marginals, is worst on most metrics.

\begin{table}[h]
\centering
\footnotesize
\caption{OT variant comparison on XA$\leftrightarrow$BJ. Balanced entropic
OT (our default) gives the lowest MAE on all tasks and directions.
\best{Red}: best per column. Lower is better.}
\vspace{1em}
\label{tab:ot_variants}
\setlength{\tabcolsep}{4pt}
\renewcommand{\arraystretch}{1.2}
\resizebox{\textwidth}{!}{%
\begin{tabular}{@{}l|cccccc|cccccc@{}}
\toprule
\multirow{3}{*}{\textbf{OT variant in \scot}}
& \multicolumn{6}{c|}{\textbf{XA $\to$ BJ}}
& \multicolumn{6}{c}{\textbf{BJ $\to$ XA}} \\
\cmidrule(lr){2-7} \cmidrule(lr){8-13}
& \multicolumn{2}{c}{GDP} & \multicolumn{2}{c}{Pop} & \multicolumn{2}{c|}{CO$_2$}
& \multicolumn{2}{c}{GDP} & \multicolumn{2}{c}{Pop} & \multicolumn{2}{c}{CO$_2$} \\
\cmidrule(lr){2-3} \cmidrule(lr){4-5} \cmidrule(lr){6-7}
\cmidrule(lr){8-9} \cmidrule(lr){10-11} \cmidrule(lr){12-13}
& MAE & MAPE & MAE & MAPE & MAE & MAPE
& MAE & MAPE & MAE & MAPE & MAE & MAPE \\
\midrule

\rowcolor{red!8}
\textbf{Balanced (Ours)}
& \best{115.33} & \best{3.17} & \best{528.50} & \best{2.13} 
& \best{149.42} & \best{1.79}
& \best{154.92} & \best{1.60} & \best{452.67} & 1.58 
& \best{128.74} & 1.63 \\

\rowcolor{gray!06}
Partial (frac$=0.7$)
& 137.18 & 4.54 & 556.49 & 3.11 & 156.70 & 2.34
& 172.49 & 2.01 & 462.15 & \best{1.56} & 131.25 & \best{1.60} \\

Sinkhorn divergence
& 141.78 & 3.90 & 562.76 & 2.23 & 160.57 & 2.06
& 167.83 & 2.62 & 476.30 & 2.54 & 136.74 & 2.50 \\

\rowcolor{gray!06}
Unbalanced (KL marginal)
& 145.04 & 5.77 & 627.52 & 5.15 & 177.08 & 3.67
& 172.09 & 2.42 & 487.21 & 2.62 & 139.69 & 2.52 \\

\bottomrule
\end{tabular}}
\end{table}


\paragraph{Multi-source hub.}
\label{app:uot_vs_bot}
In the city-to-hub coupling, we vary the KL penalty $\rho$ of unbalanced OT.
Small $\rho$ gives sharper early assignments (Fig.~\ref{fig:uot_qmax}), but
only by inflating the total mass far above one (Fig.~\ref{fig:uot_mass}),
i.e., by duplicating mass rather than matching better. Balanced OT gives the
lowest MAE on all three tasks, while unbalanced OT is sensitive to $\rho$ and
only approaches it at $\rho=3$ (Table~\ref{tab:ot_balanced_unbalanced_bj2xa}).


\begin{table}[H]
\centering
\footnotesize
\caption{Balanced vs.\ unbalanced OT in the hub on CD,BJ$\to$XA
(4 seeds, mean $\pm$ std). Balanced OT gives the lowest MAE on all tasks.
\best{Red}: best; \second{Blue}: runner-up. Lower is better.}
\label{tab:ot_balanced_unbalanced_bj2xa}
\setlength{\tabcolsep}{3pt}
\renewcommand{\arraystretch}{1.2}
\begin{tabular}{@{}l|cc|cc|cc@{}}
\toprule
\multirow{2}{*}{\textbf{OT Variant}}
& \multicolumn{2}{c|}{\textbf{GDP}}
& \multicolumn{2}{c|}{\textbf{Population}}
& \multicolumn{2}{c}{\textbf{CO$_2$}} \\
\cmidrule(lr){2-3}\cmidrule(lr){4-5}\cmidrule(lr){6-7}
& MAE$\downarrow$ & MAPE$\downarrow$
& MAE$\downarrow$ & MAPE$\downarrow$
& MAE$\downarrow$ & MAPE$\downarrow$ \\
\midrule
\rowcolor{orange!12}
Unbalanced OT {\scriptsize($\rho=0.3$)}
& 212.96 $\pm$ 27.46 & 7.22 $\pm$ 1.16
& 721.83 $\pm$ 95.63 & 4.80 $\pm$ 0.69
& 174.77 $\pm$ 41.48 & 3.17 $\pm$ 1.29 \\
\rowcolor{orange!8}
Unbalanced OT {\scriptsize($\rho=1$)}
& 173.59 $\pm$ 13.68 & 4.16 $\pm$ 1.54
& 563.30 $\pm$ 33.74 & 3.49 $\pm$ 1.09
& 149.69 $\pm$ 9.89 & 1.90 $\pm$ 0.48 \\
\rowcolor{orange!4}
Unbalanced OT {\scriptsize($\rho=3$)}
& 165.98 $\pm$ 11.66 & \best{2.06 $\pm$ 0.90}
& 530.68 $\pm$ 13.40 & \second{2.27 $\pm$ 0.47}
& 147.91 $\pm$ 8.99 & \second{1.81 $\pm$ 0.14} \\
\midrule
\rowcolor{red!8}
\textbf{Balanced OT (Ours)}
& \best{154.49 $\pm$ 2.10} & \second{2.12 $\pm$ 0.31}
& \best{467.54 $\pm$ 17.42} & \best{2.23 $\pm$ 0.25}
& \best{133.58 $\pm$ 4.64} & \best{1.51 $\pm$ 0.18} \\
\bottomrule
\end{tabular}
\end{table}

\subsection{Cycle Design}
\label{app:cycle_design}

We examine two choices in the cycle regularizer: whether to enforce one or
both cycle directions, and, if one, which.

\paragraph{One-sided vs.\ two-sided.}
\label{app:onesided_twosided}
The two-sided variant adds the reverse target$\to$source$\to$target cycle to
our source cycle. Because $n_s\neq n_t$, the two cities are not related by
an invertible map, so requiring both cycles over-constrains the matching.
The one-sided design is better on 23 of 24 metrics across four real-city
directions (Table~\ref{tab:onesided_twosided_combined}).

\begin{table}[H]
\centering
\footnotesize
\caption{One-sided vs.\ two-sided cycle reconstruction across four 
transfer directions (4 seeds, mean $\pm$ std). The one-sided design 
wins on 23 of 24 metrics. \best{Red}: better in each pair. Lower is better.}
\vspace{1em}
\label{tab:onesided_twosided_combined}
\setlength{\tabcolsep}{3pt}
\renewcommand{\arraystretch}{1.2}
\begin{tabular}{@{}l|cc|cc|cc@{}}
\toprule
\multirow{2}{*}{\textbf{Variant}}
& \multicolumn{2}{c|}{\textbf{GDP}}
& \multicolumn{2}{c|}{\textbf{Population}}
& \multicolumn{2}{c}{\textbf{CO$_2$}} \\
\cmidrule(lr){2-3}\cmidrule(lr){4-5}\cmidrule(lr){6-7}
& MAE$\downarrow$ & MAPE$\downarrow$
& MAE$\downarrow$ & MAPE$\downarrow$
& MAE$\downarrow$ & MAPE$\downarrow$ \\

\midrule
\rowcolor{olive!18}
\multicolumn{7}{@{}l}{\textcolor{olive!50!black}{$\blacktriangleright$}\ \textbf{\textsc{Direction: BJ $\to$ XA}}} \\
\midrule
\rowcolor{red!8}
\hspace{0.5em}\textbf{One-sided (Ours)}
& \best{160.21 $\pm$ 3.53} & 1.87 $\pm$ 0.18
& \best{450.14 $\pm$ 2.81} & \best{1.73 $\pm$ 0.12}
& \best{127.79 $\pm$ 1.08} & \best{1.78 $\pm$ 0.10} \\
\hspace{0.5em}Two-sided
& 160.92 $\pm$ 2.94 & \best{1.65 $\pm$ 0.20}
& 503.85 $\pm$ 14.49 & 2.51 $\pm$ 0.35
& 143.18 $\pm$ 2.82 & 2.47 $\pm$ 0.30 \\

\midrule
\rowcolor{violet!18}
\multicolumn{7}{@{}l}{\textcolor{violet!50!black}{$\blacktriangleright$}\ \textbf{\textsc{Direction: XA $\to$ BJ}}} \\
\midrule
\rowcolor{red!8}
\hspace{0.5em}\textbf{One-sided (Ours)}
& \best{120.25 $\pm$ 7.30} & \best{3.59 $\pm$ 0.48}
& \best{527.04 $\pm$ 6.38} & \best{2.17 $\pm$ 0.23}
& \best{149.20 $\pm$ 1.58} & \best{1.80 $\pm$ 0.17} \\
\hspace{0.5em}Two-sided
& 164.88 $\pm$ 12.64 & 6.38 $\pm$ 0.44
& 583.17 $\pm$ 18.17 & 3.92 $\pm$ 0.13
& 167.51 $\pm$ 5.62 & 2.98 $\pm$ 0.15 \\

\midrule
\rowcolor{teal!18}
\multicolumn{7}{@{}l}{\textcolor{teal!50!black}{$\blacktriangleright$}\ \textbf{\textsc{Direction: XA $\to$ CD}}} \\
\midrule
\rowcolor{red!8}
\hspace{0.5em}\textbf{One-sided (Ours)}
& \best{154.60 $\pm$ 2.33} & \best{4.88 $\pm$ 0.57}
& \best{558.56 $\pm$ 11.01} & \best{2.19 $\pm$ 0.21}
& \best{114.71 $\pm$ 4.22} & \best{6.48 $\pm$ 0.78} \\
\hspace{0.5em}Two-sided
& 172.93 $\pm$ 8.88 & 7.66 $\pm$ 1.58
& 582.95 $\pm$ 7.16 & 2.48 $\pm$ 0.08
& 130.86 $\pm$ 6.30 & 7.11 $\pm$ 0.35 \\

\midrule
\rowcolor{yellow!18}
\multicolumn{7}{@{}l}{\textcolor{purple!50!black}{$\blacktriangleright$}\ \textbf{\textsc{Direction: CD $\to$ XA}}} \\
\midrule
\rowcolor{red!8}
\hspace{0.5em}\textbf{One-sided (Ours)}
& \best{159.61 $\pm$ 0.71} & \best{3.17 $\pm$ 0.29}
& \best{531.00 $\pm$ 12.57} & \best{3.29 $\pm$ 0.10}
& \best{131.10 $\pm$ 2.00} & \best{1.24 $\pm$ 0.02} \\
\hspace{0.5em}Two-sided
& 172.65 $\pm$ 5.98 & 4.38 $\pm$ 0.15
& 572.33 $\pm$ 13.30 & 4.10 $\pm$ 0.31
& 135.88 $\pm$ 1.98 & 1.49 $\pm$ 0.13 \\

\bottomrule
\end{tabular}
\end{table}

\paragraph{Which direction.}
\label{app:cycle_direction}
We cycle back to the source because labels and the downstream predictor live
there. When $n_s>n_t$, however, the target-returning cycle is better posed:
$\operatorname{rank}(\mathbf{A}_{s\to t}\mathbf{A}_{t\to s})\le n_t<n_s$
prevents the source cycle from reaching the identity. To test whether this
matters, we use paired synthetic cities with $n_s=480$ and $n_t=240$ and
vary only the cycle direction (four seeds, 200 training epochs).

\begin{table}[H]
\centering
\small
\caption{Target MAE under source-returning and target-returning cycles at
$n_s=480$ and $n_t=240$ (mean $\pm$ std over four seeds). Lower is better.}
\vspace{0.8em}
\label{tab:cycle_direction}
\setlength{\tabcolsep}{18pt}
\renewcommand{\arraystretch}{1.15}
\begin{tabular}{@{}l|cc@{}}
\toprule
\textbf{Graph}
& \textbf{Source-returning}
& \textbf{Target-returning} \\
\midrule

Mobility
& \cellcolor{red!10}\textbf{$1.158\pm0.198$}
& $1.218\pm0.150$ \\

\rowcolor{gray!06}
Feature similarity
& \cellcolor{red!10}\textbf{$1.025\pm0.177$}
& $1.037\pm0.138$ \\

Spatial-functional
& \cellcolor{red!10}\textbf{$0.945\pm0.233$}
& $0.979\pm0.213$ \\

\bottomrule
\end{tabular}
\end{table}

The better-posed target-returning cycle does not improve prediction: all
differences are within one standard deviation, and the source-returning
cycle has slightly lower mean MAE in every case. Since the cycle acts as a
regularizer rather than an exact inverse, we keep the source-returning cycle
as the default.

\subsection{Comparison to Structure-Aware Cost Formulations}
\label{app:cost_variants}

We compare our cost ($\ell_2$ distance between normalized embeddings) with
three structure-aware alternatives:
a \textbf{graph-aware cost},
$C_{ij}=\alpha\|\tilde u_i-\tilde v_j\|_2+(1-\alpha)|\sigma_s(i)-\sigma_t(j)|$
with $\sigma(\cdot)=\|L\tilde u(\cdot)\|_2$ and $\alpha=0.7$;
a \textbf{mobility-profile cost}, which adds OD volume and entropy
statistics ($\alpha=0.5$); and
\textbf{Gromov--Wasserstein}~\citep{peyre2016gromov,memoli2011gromov},
which matches intra-city distance structure.
Our cost is best on all 12 metrics (Table~\ref{tab:cost_variants}). Since
the GAT encoder already encodes graph and mobility structure, adding it
again through the cost brings no benefit.

\begin{table}[H]
\centering
\footnotesize
\caption{Cost formulation comparison on XA$\,\leftrightarrow\,$BJ. 
$\ell_2$ on $\ell_2$-normalized embeddings (\scot's default) achieves 
the best result on \emph{all} $12$ metrics. \best{Red}: best per cell. 
Lower is better.}
\label{tab:cost_variants}
\setlength{\tabcolsep}{4pt}
\renewcommand{\arraystretch}{1.2}
\resizebox{\textwidth}{!}{%
\begin{tabular}{@{}l|cccccc|cccccc@{}}
\toprule
\multirow{3}{*}{\textbf{Cost variant in \scot}}
& \multicolumn{6}{c|}{\textbf{XA $\to$ BJ}}
& \multicolumn{6}{c}{\textbf{BJ $\to$ XA}} \\
\cmidrule(lr){2-7} \cmidrule(lr){8-13}
& \multicolumn{2}{c}{GDP} & \multicolumn{2}{c}{Pop} & \multicolumn{2}{c|}{CO$_2$}
& \multicolumn{2}{c}{GDP} & \multicolumn{2}{c}{Pop} & \multicolumn{2}{c}{CO$_2$} \\
\cmidrule(lr){2-3} \cmidrule(lr){4-5} \cmidrule(lr){6-7}
\cmidrule(lr){8-9} \cmidrule(lr){10-11} \cmidrule(lr){12-13}
& MAE & MAPE & MAE & MAPE & MAE & MAPE
& MAE & MAPE & MAE & MAPE & MAE & MAPE \\
\midrule

\rowcolor{red!8}
\textbf{$\ell_2$ on embeddings (Ours)}
& \best{115.33} & \best{3.17} & \best{528.50} & \best{2.13} 
& \best{149.42} & \best{1.79}
& \best{154.92} & \best{1.60} & \best{452.67} & \best{1.58} 
& \best{128.74} & \best{1.63} \\

\rowcolor{gray!06}
Graph-aware ($\alpha=0.7$)
& 152.93 & 5.73 & 588.56 & 3.74 & 168.17 & 2.96
& 166.51 & 2.43 & 458.53 & 2.27 & 131.21 & 2.24 \\

Mobility-profile ($\alpha=0.5$)
& 175.20 & 7.49 & 598.49 & 4.58 & 176.77 & 3.67
& 172.54 & 2.15 & 468.85 & 1.93 & 132.97 & 1.88 \\

\rowcolor{gray!06}
Gromov-Wasserstein
& 177.10 & 6.62 & 607.10 & 3.54 & 174.19 & 3.03
& 167.70 & 2.53 & 501.56 & 2.57 & 142.65 & 2.65 \\

\bottomrule
\end{tabular}}
\end{table}

\section{Hyperparameter Sensitivity}
\label{app:hyperpameter_sensitivity}

The main paper reports quantitative sensitivity for $\lambda_{\mathrm{align}}$ 
and $\varepsilon$. Here we provide the remaining sweeps and 
visualizations: $\tau$, hub size $K$, contrastive weight $\eta$, and 
target-prior temperature $\tau_b$, together with t-SNE visualizations 
that show how each hyperparameter shapes the cross-city embedding 
geometry. Across all six hyperparameters, \scot exhibits broad stable 
regions and only degrades at extreme values, supporting the main-paper 
claim that the gains stem from the OT-based alignment framework rather 
than fine-grained tuning.

\subsection{Sensitivity to Contrastive Temperature $\tau$}
\label{app:tau_sens}

We sweep $\tau\in\{0.03, 0.05, 0.1, 0.2, 0.5, 1\}$ on XA$\to$BJ 
(Fig.~\ref{fig:tau_sensitivity}), revealing a characteristic 
\emph{U-shape} in target performance and a matching geometric pattern 
in embedding space.

\tauI~\textbf{Quantitative U-shape.} 
Very small $\tau$ ($0.03$--$0.05$) over-sharpens similarity weighting 
and amplifies noise; large $\tau\!=\!1$ dilutes the discriminative 
signal. The \emph{sweet spot} $\tau\in[0.1, 0.5]$ balances sharpness 
against smoothness, and we adopt $\tau=0.1$ as the default 
(Fig.~\ref{fig:tau_sensitivity}\,(a)).
\tauII~\textbf{Geometric mirror in embedding space.} 
The t-SNE visualization (Fig.~\ref{fig:tau_sensitivity}\,(b)) reveals 
the same trade-off geometrically: small $\tau$ leaves the two cities 
\emph{weakly interleaved}, moderate $\tau$ produces \emph{clean 
interleaving with preserved clusters}, and large $\tau$ \emph{over-
smooths} the embeddings into indistinguishable mixtures. The 
embedding geometry tracks the metric U-shape, providing visual 
confirmation of the sweet spot.

\begin{figure}[H]
\centering
\begin{subfigure}[t]{0.46\textwidth}
    \centering
    \includegraphics[width=\linewidth]{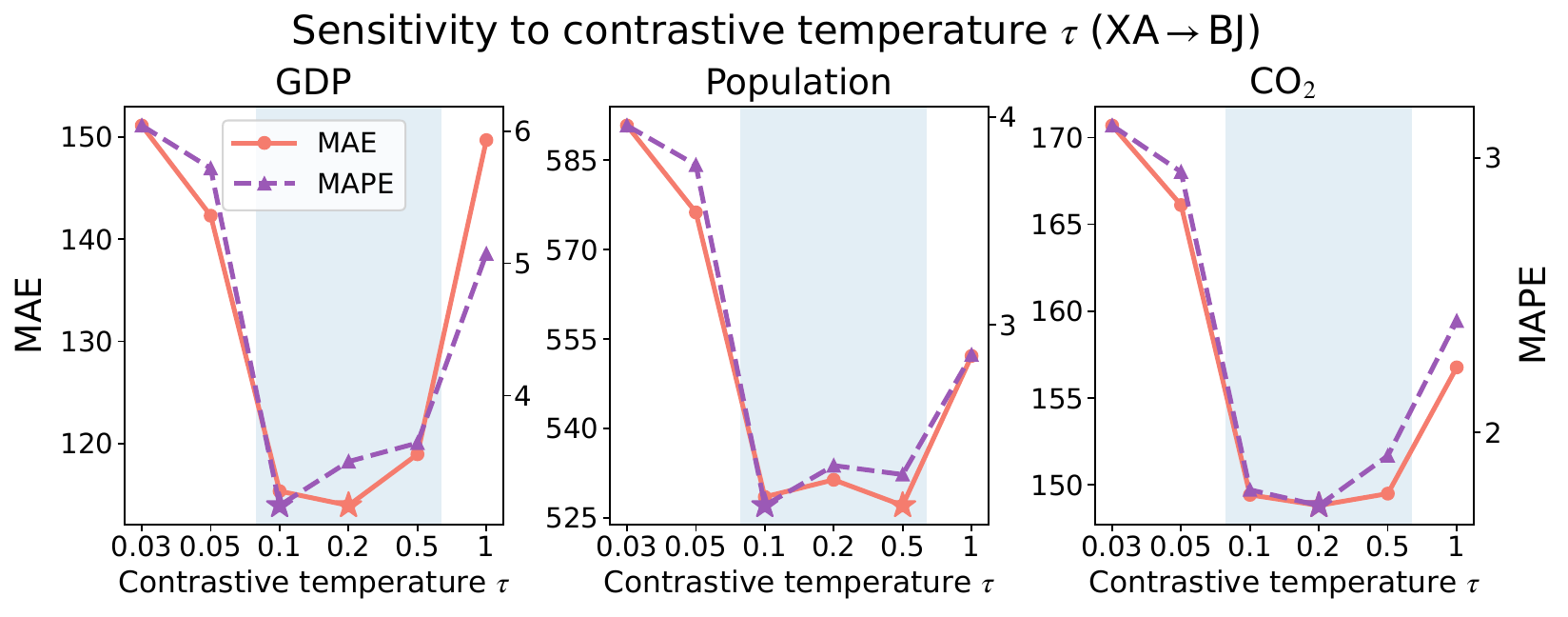}
    \caption{\textbf{Metrics.} MAE / MAPE vs.\ $\tau$.}
    \label{fig:tau_sensitivity_metrics}
\end{subfigure}
\hfill
\begin{subfigure}[t]{0.52\textwidth}
    \centering
    \includegraphics[width=\linewidth]{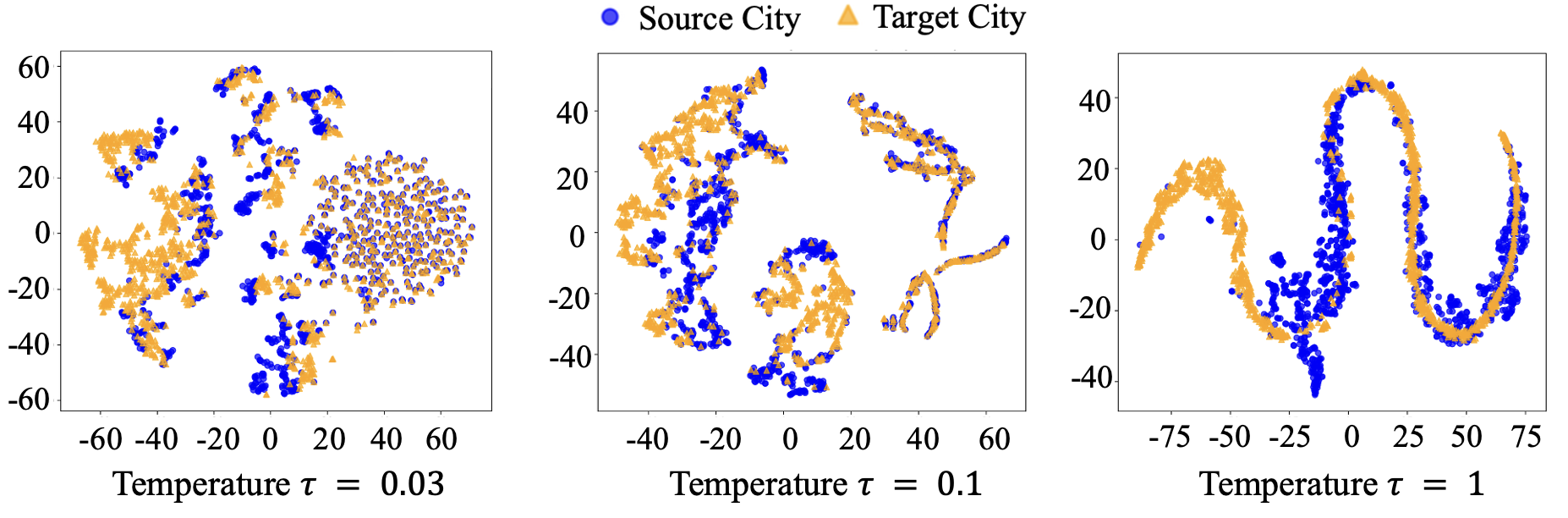}
    \caption{\textbf{Embeddings.} t-SNE at $\tau=0.03,\,0.1,\,1$.}
    \label{fig:tau_sensitivity_tsne}
\end{subfigure}
\caption{\textbf{Sensitivity to $\tau$ on XA$\to$BJ.} 
(a) U-shape with sweet spot $\tau\in[0.1, 0.5]$. 
(b) Geometry mirrors metrics: extreme $\tau$ yields weak interleaving 
or over-smoothing; $\tau=0.1$ aligns cleanly without collapsing clusters.}
\label{fig:tau_sensitivity}
\end{figure}

\subsection{Sensitivity to Contrastive Weight $\eta$}
\label{app:eta_sens}

The contrastive weight $\eta$ balances OT-based geometric 
correspondence against contrastive discriminative sharpening. We 
sweep $\eta\in\{0, 0.1, 0.2, 0.5, 1, 2, 5, 10\}$ on XA$\to$BJ 
(Fig.~\ref{fig:eta_sensitivity}).

\tauI~\textbf{Moderate $\eta$ wins; both extremes fail.} 
$\eta\in[0.1, 0.5]$ consistently yields the best or near-best 
performance (Fig.~\ref{fig:eta_sensitivity}\,(a)). Too small $\eta$ 
under-uses contrastive sharpening; too large $\eta\!\ge\!2$ causes 
sharp error increases.
\tauII~\textbf{The geometry confirms the failure modes.} 
t-SNE (Fig.~\ref{fig:eta_sensitivity}\,(b)) shows that $\eta=0$ 
leaves the two cities \emph{barely interleaved}, $\eta=0.5$ produces 
\emph{clean alignment with preserved clusters}, and $\eta=5$ drives 
\emph{near-complete feature mixing}. OT and contrastive alignment thus 
act complementarily, with \scot robust over a practical mid-range.

\begin{figure}[H]
\centering
\begin{subfigure}[t]{0.46\textwidth}
    \centering
    \includegraphics[width=\linewidth]{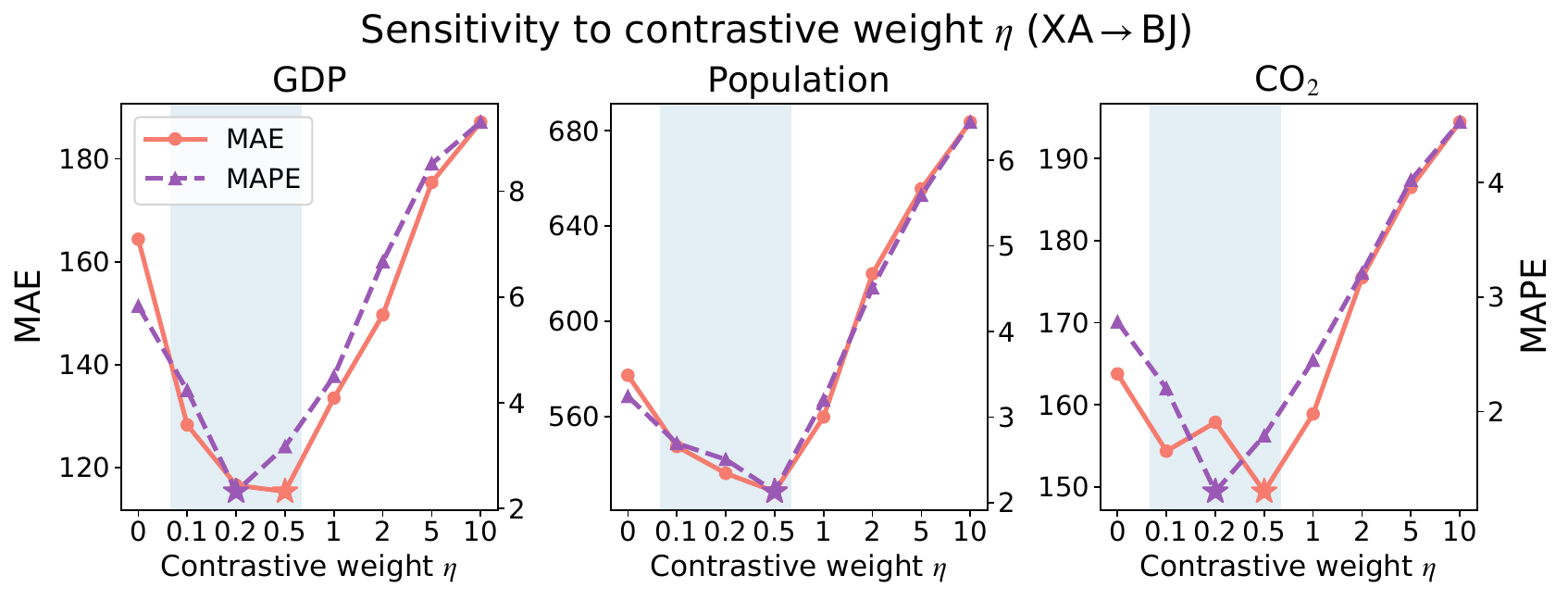}
    \caption{\textbf{Metrics.} MAE / MAPE vs.\ $\eta$.}
    \label{fig:eta_sensitivity_metrics}
\end{subfigure}
\hfill
\begin{subfigure}[t]{0.52\textwidth}
    \centering
    \includegraphics[width=\linewidth]{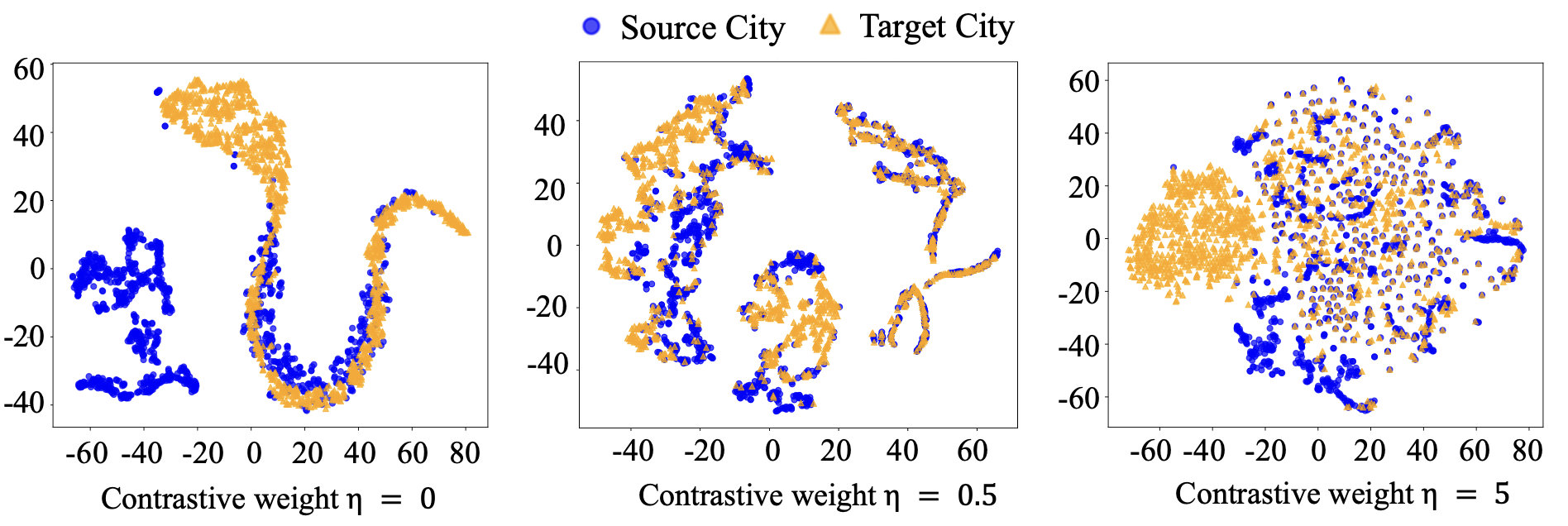}
    \caption{\textbf{Embeddings.} t-SNE at $\eta=0,\,0.5,\,5$. 
    Blue $\circ$: source; orange $\triangle$: target.}
    \label{fig:eta_sensitivity_tsne}
\end{subfigure}
\caption{\textbf{Sensitivity to $\eta$ on XA$\to$BJ.} 
(a) Best in $\eta\in[0.1, 0.5]$; large $\eta$ degrades sharply. 
(b) $\eta=0$ under-aligns; $\eta=0.5$ aligns cleanly; 
$\eta=5$ over-mixes the two cities.}
\label{fig:eta_sensitivity}
\end{figure}

\subsection{Sensitivity to Hub Size $K$ and $\lambda_{\mathrm{align}}$}
\label{app:hubK_sens}
\label{app:lambda_align_sens}

\paragraph{Hub size $K$.}
We sweep $K\in\{2,4,8,16,32,64,128,256,1000\}$ on CD,BJ$\to$XA
(Fig.~\ref{fig:hubK_sensitivity}). Performance is stable for
$K\in[4,32]$. At $K=2$, heterogeneous regions must share too few
prototypes and performance drops sharply; for $K\ge64$, the hub becomes too
fine to act as a compact bridge. We use $K=32$.

\paragraph{Effect of $\lambda_{\mathrm{align}}$ on geometry.}
Fig.~\ref{fig:tsne_lambda_align_xa2bj} complements the sweep in
Fig.~\ref{fig:lambda_align_sens_xa2bj}. At $\lambda_{\mathrm{align}}=1$,
the two cities interleave while keeping their cluster structure. At $0.05$
they remain largely separate (alignment too weak), and at $3$ they are
over-mixed (alignment overrides intra-city structure).

\begin{figure}[H]
    \centering
    \begin{subfigure}[b]{0.48\linewidth}
        \centering
        \includegraphics[width=\linewidth]{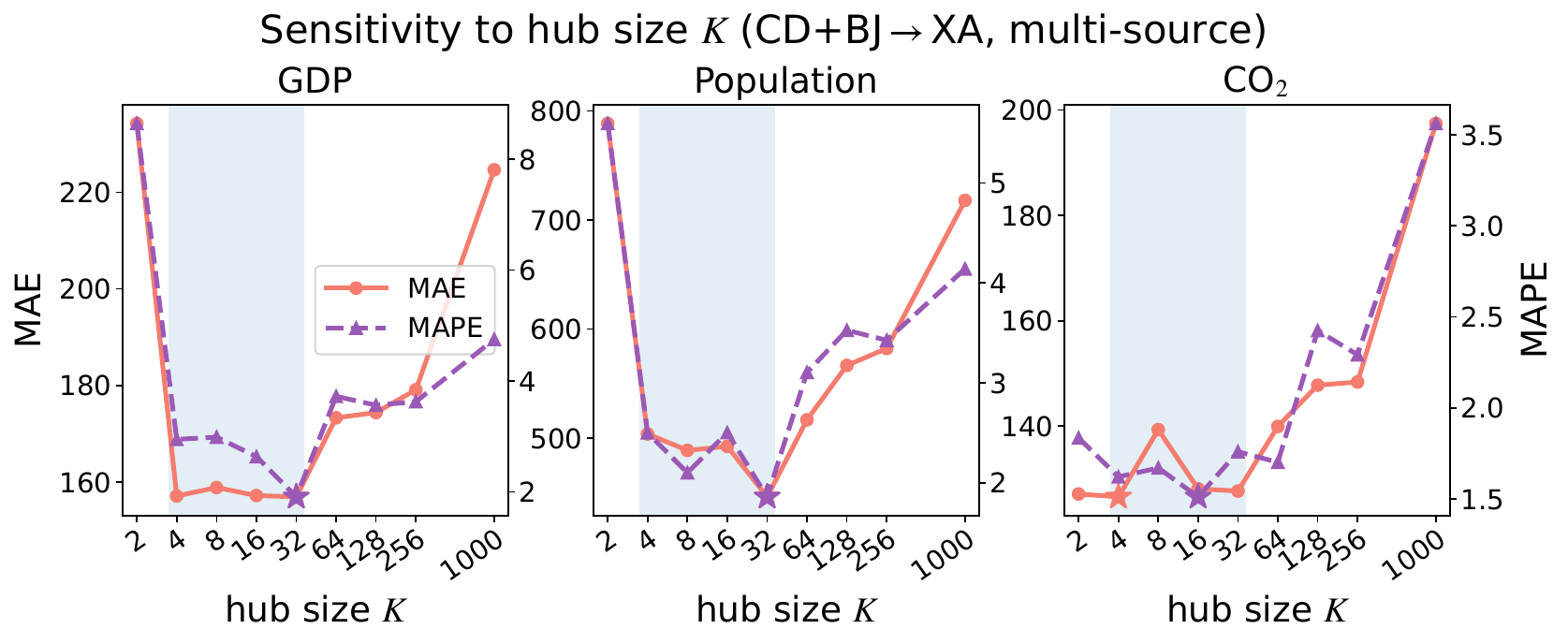}
        \caption{Hub size $K$ (CD,BJ$\to$XA)}
        \label{fig:hubK_sensitivity}
    \end{subfigure}\hfill
    \begin{subfigure}[b]{0.48\linewidth}
        \centering
        \includegraphics[width=\linewidth]{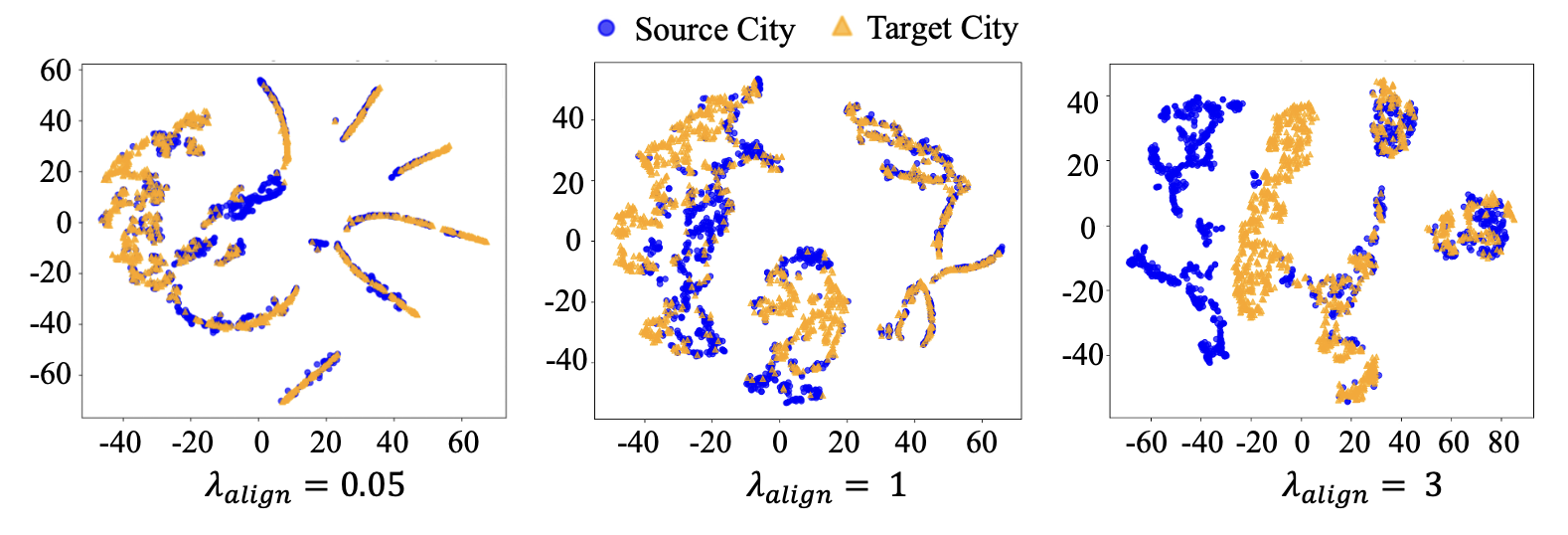}
        \caption{t-SNE at $\lambda_{\mathrm{align}}=0.05,\,1,\,3$ (XA$\to$BJ)}
        \label{fig:tsne_lambda_align_xa2bj}
    \end{subfigure}
    \caption{(a) MAE/MAPE vs.\ hub size $K$; performance is stable for
    $K\in[4,32]$. (b) Embeddings under varying $\lambda_{\mathrm{align}}$;
    blue $\circ$: source, orange $\triangle$: target.}
    \label{fig:hubK_lambda}
\end{figure}

\subsection{Sensitivity to Target-Prior Temperature $\tau_b$}
\label{subsec:taub_sensitivity}

The target-prior temperature $\tau_b$ controls the sharpness of the 
target-induced prototype marginal $\mathbf{b}$: smaller $\tau_b$ 
peaks $\mathbf{b}$, larger $\tau_b$ flattens it. We sweep 
$\tau_b\in\{0.1, 0.2, 0.3, 0.5, 0.8, 1.0, 2.0\}$ on multi-source 
CD,BJ$\to$XA, jointly with hub-usage diagnostics that reveal the 
underlying mechanism.

\tauI~\textbf{The recommended range $\tau_b\in[0.3, 0.8]$ delivers 
the best performance.} Within this range, MAE and MAPE are stably 
minimized across GDP, Population, and CO$_2$, with $\tau_b=0.5$ as 
the operating midpoint (Fig.~\ref{fig:taub_combined}\,(a)). The 
plateau spans nearly an order of magnitude, eliminating the need 
for per-target tuning.

\tauII~\textbf{Hub diagnostics explain why.} The OT column marginal 
$p_k=\sum_i\Pi_{ik}$ reports how the target city distributes mass 
across prototypes. Within the recommended range, the normalized 
entropy $H(p)/\log K$ stabilizes around $0.4$, corresponding to 
\emph{$\exp(H(p))\approx 4$ effective prototypes} per region---a 
selective, non-collapsed hub usage that aligns with strong transfer 
(Fig.~\ref{fig:taub_combined}\,(b)).

\tauIII~\textbf{Boundary behavior is interpretable.} Departures from 
the recommended range deviate predictably: small $\tau_b\le 0.2$ 
peaks $\mathbf{b}$ onto few prototypes (low entropy, prototype 
\emph{specialization without diversity}), while large $\tau_b\ge 1.0$ 
flattens $\mathbf{b}$ toward uniform ($\exp(H(p))\!\to\!K$, 
\emph{diffuse usage with weakened target guidance}). Both extremes 
serve only to bracket the boundary behavior of the target-induced 
prior.

\begin{figure}[H]
\centering
\begin{subfigure}[t]{0.62\textwidth}
    \centering
    \includegraphics[width=\linewidth]{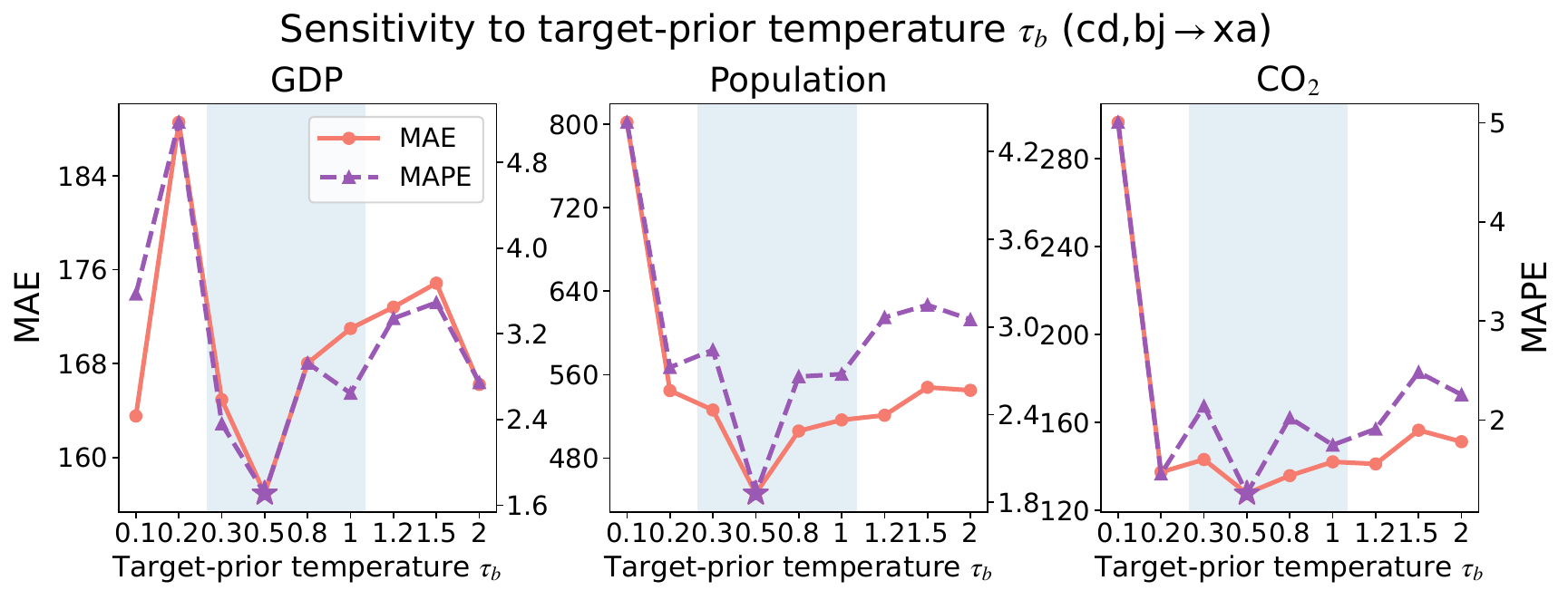}
     \caption{\textbf{Metrics.} MAE (solid) and MAPE (dashed) vs.\ $\tau_b$.}
    \label{fig:taub_metrics}
\end{subfigure}
\hfill
\begin{subfigure}[t]{0.36\textwidth}
    \centering
    \includegraphics[width=\linewidth]{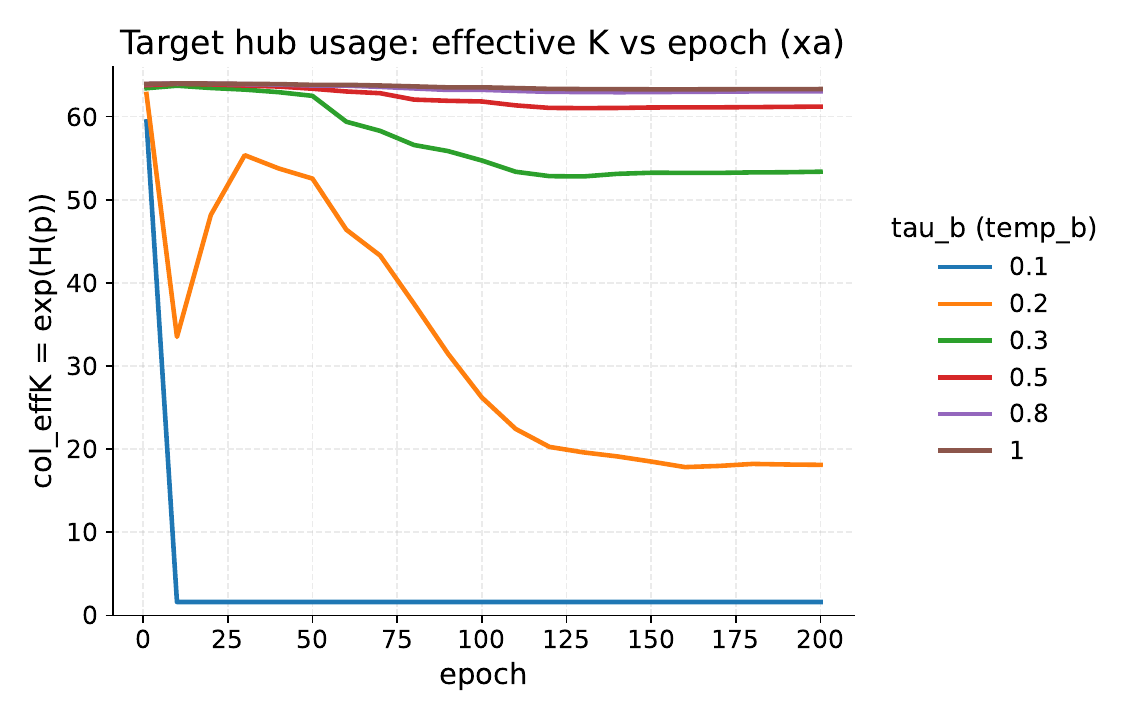}
    \caption{\textbf{Hub diagnostics.}}
    \label{fig:taub_hubdiag}
\end{subfigure}
\caption{\textbf{Sensitivity to $\tau_b$ on multi-source CD,BJ$\to$XA.} 
(a) Performance plateau across $\tau_b\in[0.3, 0.8]$. 
(b) Normalized entropy $H(p)/\log K$ and effective prototype count 
$\exp(H(p))$ on XA: hub usage stabilizes at $\sim 4$ effective 
prototypes within the same range, providing direct mechanistic 
evidence for the metric plateau.}
\label{fig:taub_combined}
\end{figure}





\section{Semantic Analysis of the Learned Transport}
\label{app:transport_semantics}

We test whether high-mass patterns in the learned coupling reflect semantic
correspondence rather than arbitrary block concentration, at three levels:
a single source region, blocks of the full coupling, and synthetic cities
with known correspondence. For real cities, we use two attributes never seen
during training as coarse semantic proxies: \textsc{Function} (dominant urban
function) and \textsc{Form\_Type} (building height and layout). The real-city
results should therefore be read as evidence of semantic enrichment, not
exact matching.

\subsection{Single Source Region}
\label{app:transport_semantics_local}

For the source region in Fig.~\ref{fig:alignment_diagnostics}, we compare
the mass-weighted purity of its top-$k$ targets against size-matched random
subsets. We also compare against a row-shuffled control, which keeps every
transport profile but assigns it to a different source region
(Table~\ref{tab:local_transport_semantics}).

\begin{table}[H]
\centering
\caption{Semantic diagnostics for the high-mass target regions in
Fig.~\ref{fig:alignment_diagnostics}. Purity gains are measured against size-matched random
subsets. The row-shuffled control preserves each transport profile but breaks
its association with the original source region.}
\label{tab:local_transport_semantics}
\setlength{\tabcolsep}{9pt}
\renewcommand{\arraystretch}{1.22}
\resizebox{\textwidth}{!}{%
\begin{tabular}{@{}l|c|cc@{}}
\toprule
\multirow{2}{*}{\textbf{Analysis}}
& \multirow{2}{*}{\textbf{Subset / diagnostic}}
& \multicolumn{2}{c}{\textbf{Semantic consistency}} \\
\cmidrule(lr){3-4}
& & \textbf{\textsc{Function}} & \textbf{\textsc{Form\_Type}} \\
\midrule

Top-$k$ purity gain
& Top-10
& \cellcolor{blue!12}\textbf{$+10.2\%$}
& \cellcolor{blue!18}\textbf{$+20.9\%$} \\

\rowcolor{gray!06}
Top-$k$ purity gain
& Top-20
& \cellcolor{blue!14}\textbf{$+13.6\%$}
& \cellcolor{blue!20}\textbf{$+22.1\%$} \\

Top-$k$ purity gain
& Top-50
& \cellcolor{blue!22}\textbf{$+26.9\%$}
& \cellcolor{blue!12}\textbf{$+10.4\%$} \\

\midrule

\rowcolor{red!7}
\textbf{\scot vs.\ row shuffle}
& \textbf{Semantic consistency}
& \cellcolor{red!14}\textbf{Same-\textsc{Function} mass: $+18.1\%$}
& \cellcolor{red!14}\textbf{Weighted distance: $-8.4\%$} \\

\bottomrule
\end{tabular}}
\end{table}

The top targets are purer than random subsets, and agreement drops after
row shuffling. The selected targets thus depend on the source region rather
than being reused globally.

\subsection{Blocks of the Full Coupling}
\label{app:transport_semantics_global}

For the full XA$\to$BJ coupling, we represent each source region by its
outgoing profile $P_{i,:}$ and each target region by its normalized incoming
profile $\widetilde{P}_{:,j}=P_{:,j}/\sum_i P_{ij}$. We build cosine kNN
graphs over these profiles and cluster each side independently with spectral
clustering ($k\in\{20,50,80\}$, $C\in\{8,10,12\}$, four seeds; 36 partitions
per side). \textsc{Function} is used only for evaluation.

\paragraph{Within-side enrichment.}
The shuffled control keeps profiles, block sizes, and citywide
\textsc{Function} frequencies, but breaks the region--block association
(Table~\ref{tab:block_semantic_structure}).

\begin{table}[H]
\centering
\caption{\textsc{Function} structure of transport-profile blocks.
Purity is averaged over all neighborhood sizes, cluster counts, and seeds.
Representative enriched blocks use $k=50$, $C=10$, and seed 42.
Higher purity and enrichment indicate stronger semantic structure.}
\label{tab:block_semantic_structure}
\setlength{\tabcolsep}{6pt}
\renewcommand{\arraystretch}{1.22}
\resizebox{\textwidth}{!}{%
\begin{tabular}{@{}l|ccc|ccccc@{}}
\toprule
\multirow{2}{*}{\textbf{Side}}
& \multicolumn{3}{c|}{\textbf{Average block purity}}
& \multicolumn{5}{c}{\textbf{Representative enriched block}} \\
\cmidrule(lr){2-4}\cmidrule(lr){5-9}
& \textbf{SCOT}
& \textbf{Shuffled}
& \textbf{Relative gain}
& \textbf{Block}
& \textbf{Size}
& \textbf{Dominant category}
& \textbf{Block share}
& \textbf{Enrichment} \\
\midrule

\rowcolor{red!7}
\textbf{Source}
& \cellcolor{red!14}\textbf{0.370}
& 0.276
& \cellcolor{blue!20}\textbf{$+34\%$}
& S3 & 107 & Category 0 & 0.607
& \cellcolor{blue!18}\textbf{$2.26\times$} \\

\rowcolor{gray!06}
Source
& \multicolumn{3}{c|}{}
& S4 & 45 & Category 3 & 0.378
& \cellcolor{blue!12}\textbf{$1.82\times$} \\

Source
& \multicolumn{3}{c|}{}
& S2 & 96 & Category 3 & 0.354
& \cellcolor{blue!10}\textbf{$1.71\times$} \\

\midrule

\rowcolor{red!7}
\textbf{Target}
& \cellcolor{red!14}\textbf{0.398}
& 0.338
& \cellcolor{blue!14}\textbf{$+18\%$}
& T1 & 49 & Category 0 & 0.633
& \cellcolor{blue!14}\textbf{$1.87\times$} \\

\rowcolor{gray!06}
Target
& \multicolumn{3}{c|}{}
& T8 & 161 & Category 3 & 0.224
& \cellcolor{blue!17}\textbf{$2.04\times$} \\

Target
& \multicolumn{3}{c|}{}
& T0 & 302 & Category 4 & 0.219
& \cellcolor{blue!13}\textbf{$1.79\times$} \\

\bottomrule
\end{tabular}}
\end{table}

Both sides are purer than the control. For a representative partition, none
of 5,000 label permutations reaches the observed purity ($p<0.001$); NMI and
ARI show the same pattern. Enriched blocks include the less frequent
Categories 3 and 4, so the blocks are functionally enriched rather than
dominated by the majority category. Partitions are stable across seeds
(ARI/NMI: $0.84/0.93$ source, $0.79/0.90$ target).

\paragraph{Cross-city block alignment.}
We aggregate the coupling into block-level mass
$M_{ab}=\sum_{i\in S_a}\sum_{j\in T_b}P_{ij}$ and test whether mass
concentrates between blocks with similar \textsc{Function} compositions.
The null shuffles target \textsc{Function} labels with blocks fixed; a lift
of one denotes the null expectation. Overlap compares which functions two
blocks contain, while JS similarity also compares their proportions
(Table~\ref{tab:block_cross_city_semantics}).

\begin{table}[h]
\centering
\caption{Cross-city semantic consistency over 36 independently extracted
source--target partitions. Higher lift indicates that transport concentrates
between blocks with more similar \textsc{Function} compositions.}
\label{tab:block_cross_city_semantics}
\setlength{\tabcolsep}{14pt}
\renewcommand{\arraystretch}{1.22}
\resizebox{0.86\textwidth}{!}{%
\begin{tabular}{@{}l|ccc@{}}
\toprule
\textbf{Metric}
& \textbf{Observed lift}
& \textbf{Permuted null}
& \textbf{$p$} \\
\midrule

\rowcolor{red!7}
\textsc{Function} distribution overlap
& \cellcolor{red!16}\textbf{1.089}
& $1.001\pm0.010$
& \cellcolor{blue!14}\textbf{$<0.001$} \\

\rowcolor{gray!06}
\textsc{Function} JS similarity
& \cellcolor{red!13}\textbf{1.048}
& $1.000\pm0.005$
& \cellcolor{blue!14}\textbf{$<0.001$} \\

\midrule

\rowcolor{blue!7}
\textit{Same-dominant-\textsc{Function} transport mass}
& \cellcolor{blue!18}\textbf{63.1\%}
& \cellcolor{blue!10}\textbf{49.3\%}
& --- \\

\bottomrule
\end{tabular}}
\end{table}

Both lifts are significantly above the null, and $63.1\%$ of transport mass
falls between blocks with the same dominant \textsc{Function}, versus
$49.3\%$ under permutation. Because source and target blocks are extracted
independently, this agreement is not imposed by the clustering.

\subsection{Synthetic Cities with Known Correspondence}
\label{app:transport_semantics_synthetic}

Real cities lack ground-truth correspondence, so we use paired synthetic
cities in which each region's latent type is known. We compare \scot with
marginal mixing (no selectivity), MMD+OT, and Non-aligned+OT (OT on
embeddings trained without a cross-city objective). Lifts are normalized by
the agreement expected from source and target type frequencies alone
(lift $=1$: no selectivity). \emph{Fixed-type} lift matches type indices
directly; \emph{permutation-invariant} lift allows labels to be reordered
across cities; \emph{overlap} and \emph{JS} lifts compare the full type
mixtures of transported blocks.

\begin{table}[h]
\centering
\caption{Recovery of ground-truth semantic correspondence in controlled
synthetic cities. A lift of $1$ indicates no selectivity beyond the independent
marginal expectation. \best{Red}: best.}
\label{tab:synthetic_semantic_correspondence}
\setlength{\tabcolsep}{9pt}
\renewcommand{\arraystretch}{1.2}
\resizebox{\textwidth}{!}{%
\begin{tabular}{@{}l|cccc@{}}
\toprule
\textbf{Method}
& \textbf{Fixed-type lift}
& \textbf{Permutation-invariant lift}
& \textbf{Shared overlap lift}
& \textbf{Shared JS lift} \\
\midrule
Marginal mixing
& 1.000 & 1.000 & 1.000 & 1.000 \\
\rowcolor{gray!06}
MMD+OT
& 0.969 & 1.040 & 0.964 & 0.963 \\
Non-aligned+OT
& 0.980 & 1.007 & 0.980 & 0.981 \\
\midrule
\rowcolor{red!8}
\textbf{\scot (Ours)}
& \best{1.331}
& \best{1.494}
& \best{1.317}
& \best{1.330} \\
\bottomrule
\end{tabular}}
\end{table}

\scot places $33.1\%$ more mass than expected between regions of the same
type, rising to $49.4\%$ when labels may be reordered. All OT-based controls
stay near one, so the recovered correspondence comes from the \scot
alignment objective rather than from OT alone.

\section{Controlled Generality Across Graph Constructions and City Heterogeneity}
\label{app:controlled_generality}

Public urban data rarely offer consistent partitions, graphs, and labels
across cities and modalities. We therefore use synthetic cities in which the
latent regions, labels, and task are fixed, and vary only (i) the graph
construction or (ii) the source--target heterogeneity.

\paragraph{Setup.}
We generate $K=10$ latent region types with centers
$\boldsymbol{\mu}_k\in\mathbb{R}^{16}$, $75\%$ of which are shared by the two
cities. Source features are $\mathbf{x}_i^s=\boldsymbol{\mu}_{c_i^s}+\boldsymbol{\epsilon}_i^s$,
where $c_i$ is the latent type of region $i$. Shared types are shifted in the
target as $\boldsymbol{\mu}_k^t=\boldsymbol{\mu}_k\mathbf{R}+0.45h\,\boldsymbol{\delta}_k$,
where $\mathbf{R}$ is a rotation, $\boldsymbol{\delta}_k$ a random
perturbation, and $h$ the heterogeneity level. Labels follow
$y_i=\boldsymbol{\mu}_{c_i}^{\top}\mathbf{w}+\alpha_{c_i}+\epsilon_i$.
We use $n_s=300$, $n_t=450$, top-12 row-normalized graphs, 100 training
epochs, and four seeds.

\paragraph{Graph constructions.}
From the same regions and labels, we build three graphs that differ only in
the edge rule: mobility interaction, functional similarity, and
spatial-functional.

\begin{table}[H]
\centering
\small
\caption{Target MAE under different graph constructions
(mean $\pm$ std over four seeds). Lower is better.
Gain is \scot's relative improvement over the strongest baseline.}
\label{tab:controlled_graph_types}
\setlength{\tabcolsep}{7pt}
\renewcommand{\arraystretch}{1.18}

\begin{tabular}{@{}l|ccc|c@{}}
\toprule
\textbf{Graph construction}
& \textbf{Non-align}
& \textbf{MMD}
& \textbf{\scot}
& \textbf{Gain} \\
\midrule

\rowcolor{purple!7}
Mobility interaction
& $4.051\pm3.228$
& $1.289\pm0.515$
& $\mathbf{1.198\pm0.189}$
& \cellcolor{purple!16}$\mathbf{+7.1\%}$ \\

\rowcolor{cyan!7}
Functional similarity
& $1.536\pm0.654$
& $1.222\pm0.159$
& $\mathbf{1.156\pm0.167}$
& \cellcolor{cyan!16}$\mathbf{+5.4\%}$ \\

\rowcolor{yellow!10}
Spatial-functional
& $1.659\pm0.394$
& $1.426\pm0.126$
& $\mathbf{1.311\pm0.126}$
& \cellcolor{yellow!22}$\mathbf{+8.1\%}$ \\

\bottomrule
\end{tabular}
\end{table}

\scot is best under all three graphs, so the alignment is not tied to
mobility-derived edges.

\paragraph{Heterogeneity.}
With the functional-similarity graph fixed, we vary
$h\in\{0.2,0.5,0.8\}$, which scales the target-side shift and noise
($h=0.5$ is the setting above).

\begin{table}[H]
\centering
\small
\caption{Target MAE under increasing source--target heterogeneity
(mean $\pm$ std over four seeds). Lower is better.
Gain is \scot's relative improvement over the strongest baseline.}
\label{tab:controlled_heterogeneity}
\setlength{\tabcolsep}{7pt}
\renewcommand{\arraystretch}{1.18}

\begin{tabular}{@{}c|ccc|c@{}}
\toprule
\textbf{$h$}
& \textbf{Non-align}
& \textbf{MMD}
& \textbf{\scot}
& \textbf{Gain} \\
\midrule

\rowcolor{teal!8}
0.2
& $1.415\pm0.627$
& $1.164\pm0.247$
& $\mathbf{1.057\pm0.235}$
& \cellcolor{teal!18}$\mathbf{+9.2\%}$ \\

\rowcolor{orange!8}
0.5
& $1.536\pm0.654$
& $1.222\pm0.159$
& $\mathbf{1.156\pm0.167}$
& \cellcolor{orange!18}$\mathbf{+5.4\%}$ \\

\rowcolor{violet!8}
0.8
& $1.511\pm0.619$
& $1.227\pm0.071$
& $\mathbf{1.202\pm0.155}$
& \cellcolor{violet!18}$\mathbf{+2.0\%}$ \\

\bottomrule
\end{tabular}
\end{table}

\scot stays best at every level. Its MAE rises gradually from $1.057$ to
$1.202$, while its margin over MMD narrows from $9.2\%$ to $2.0\%$ as the
cities become less alike.

\section{Complexity Analysis}
\label{app:complexity}

Table~\ref{tab:complexity} lists the cost of computing the OT couplings
with $T$ Sinkhorn iterations; the contrastive term reuses the same matrices
at no extra asymptotic cost. With $M$ sources and cities of about $n$
regions, the hub replaces $M$ problems of size $n\times n$ with $M+1$
problems of size $n\times K$, a speedup of $\frac{M}{M+1}\cdot\frac{n}{K}$
(about $20\times$ for $M=2$, $n\approx1000$, $K=32$). The cycle term adds
$O(n_s^2 n_t)$ in the single-source case and $O(n^2K)$ per city with the hub.

\begin{table}[H]
\centering
\footnotesize
\caption{Cost of the OT couplings. $n$: city size; $T$: Sinkhorn
iterations; $M$: number of sources; $K$: hub size.}
\label{tab:complexity}
\setlength{\tabcolsep}{8pt}
\renewcommand{\arraystretch}{1.25}
\begin{tabular}{@{}l|cc@{}}
\toprule
\textbf{Variant} & \textbf{Time} & \textbf{Memory} \\
\midrule
Single-source \scot
& $O(T n_s n_t)$ & $O(n_s n_t)$ \\
\rowcolor{gray!06}
Multi-source, pairwise OT
& $O(TMn^2)$ & $O(Mn^2)$ \\
\rowcolor{red!8}
\textbf{Multi-source \scot (hub)}
& $O(T(M\!+\!1)nK)$ & $O((M\!+\!1)nK)$ \\
\bottomrule
\end{tabular}
\end{table}

\section{Multi-Source vs.\ Best Single Source}
\label{app:source_quality}

Choosing the best single source normally requires target labels, which are
scarce in our setting~\citep{mansour2008domain,sun2015survey,pan2009survey}.
\scot avoids explicit selection by aggregating all sources through the hub.
Table~\ref{tab:ss_vs_ms_scot} compares it with an \emph{oracle} that picks
the best single source separately for each target, task, and metric.

\begin{table}[H]
\centering
\footnotesize
\caption{Multi-source \scot vs.\ the best single-source \scot (oracle,
chosen per target, task, and metric). \best{Red}: better in each pair.
Lower is better.}
\vspace{1em}
\label{tab:ss_vs_ms_scot}
\setlength{\tabcolsep}{3pt}
\renewcommand{\arraystretch}{1.2}
\begin{tabular}{@{}l|cc|cc|cc@{}}
\toprule
\multirow{2}{*}{\textbf{Target / Variant}}
& \multicolumn{2}{c|}{\textbf{GDP}}
& \multicolumn{2}{c|}{\textbf{Population}}
& \multicolumn{2}{c}{\textbf{CO$_2$}} \\
\cmidrule(lr){2-3}\cmidrule(lr){4-5}\cmidrule(lr){6-7}
& MAE$\downarrow$ & MAPE$\downarrow$
& MAE$\downarrow$ & MAPE$\downarrow$
& MAE$\downarrow$ & MAPE$\downarrow$ \\

\midrule
\rowcolor{cyan!18}
\multicolumn{7}{@{}l}{\textcolor{cyan!50!black}{$\blacktriangleright$}\ \textbf{\textsc{Target: Beijing (BJ)}}} \\
\midrule
\hspace{0.5em}Best single-source \textit{(oracle)}
& 115.33 & 3.17
& 528.50 & 2.13
& 148.50 & 1.54 \\
\rowcolor{red!8}
\hspace{0.5em}\textbf{Multi-source (Ours)}
& \best{104.16} & \best{2.57}
& \best{525.10} & \best{1.87}
& \best{143.53} & \best{1.16} \\

\midrule
\rowcolor{orange!18}
\multicolumn{7}{@{}l}{\textcolor{orange!60!black}{$\blacktriangleright$}\ \textbf{\textsc{Target: Xi'an (XA)}}} \\
\midrule
\hspace{0.5em}Best single-source \textit{(oracle)}
& \best{154.92} & \best{1.60}
& 452.67 & \best{1.58}
& 128.74 & \best{1.24} \\
\rowcolor{red!8}
\hspace{0.5em}\textbf{Multi-source (Ours)}
& 156.94 & 1.71
& \best{446.13} & 1.86
& \best{127.66} & 1.26 \\

\midrule
\rowcolor{purple!18}
\multicolumn{7}{@{}l}{\textcolor{purple!60!black}{$\blacktriangleright$}\ \textbf{\textsc{Target: Chengdu (CD)}}} \\
\midrule
\hspace{0.5em}Best single-source \textit{(oracle)}
& 135.63 & \best{3.55}
& 575.43 & \best{2.35}
& 114.68 & 7.83 \\
\rowcolor{red!8}
\hspace{0.5em}\textbf{Multi-source (Ours)}
& \best{133.94} & 3.82
& \best{546.82} & 2.43
& \best{98.43} & \best{5.10} \\

\bottomrule
\end{tabular}
\end{table}

Multi-source \scot beats the oracle on 8 of 9 MAE entries and on 4 of 9
MAPE entries. It wins every metric on Beijing, while its shortfalls lie on
Xi'an and on Chengdu's GDP and population MAPE.

\paragraph{Label-free conflict signal.}
Table~\ref{tab:hub_conflict} reports, for each target, the gap $\Delta$
between the two sources on three hub statistics, all computed without
target labels. Beijing has near-zero gaps, Chengdu intermediate ones, and
Xi'an the largest on all three. This ordering matches where aggregation
helps least, suggesting that source disagreement in the hub can flag cases
where a single source may be preferable. Using this signal for
conflict-aware gating is left to future work.

\begin{table}[H]
\centering
\footnotesize
\caption{Hub statistics near convergence for the two sources (S1, S2) of
each target. $\Delta$: absolute gap between sources; larger $\Delta$
indicates stronger source disagreement. No target labels are used.}
\vspace{0.2em}
\label{tab:hub_conflict}
\setlength{\tabcolsep}{4pt}
\renewcommand{\arraystretch}{1.2}
\begin{tabular}{@{}l|ccc|ccc|ccc@{}}
\toprule
\multirow{2}{*}{\textbf{Target}}
& \multicolumn{3}{c|}{$\mathcal{L}_{\mathrm{OT}}$}
& \multicolumn{3}{c|}{$\mathcal{L}_{\mathrm{Con}}$}
& \multicolumn{3}{c}{Transport mass} \\
\cmidrule(lr){2-4}\cmidrule(lr){5-7}\cmidrule(lr){8-10}
& S1 & S2 & $\Delta$
& S1 & S2 & $\Delta$
& S1 & S2 & $\Delta$ \\
\midrule
Beijing (BJ)
& 0.558 & 0.557 & 0.001
& 0.999 & 0.987 & 0.012
& 0.431 & 0.447 & 0.016 \\
\rowcolor{gray!06}
Chengdu (CD)
& 0.435 & 0.403 & 0.032
& 1.075 & 0.839 & 0.237
& 0.489 & 0.515 & 0.026 \\
\rowcolor{orange!12}
Xi'an (XA)
& 0.470 & 0.420 & \cellcolor{orange!22}\textbf{0.050}
& 1.080 & 0.660 & \cellcolor{orange!22}\textbf{0.420}
& 0.480 & 0.580 & \cellcolor{orange!22}\textbf{0.100} \\
\bottomrule
\end{tabular}
\end{table}

\section{Cross-Country Generalization: Transfer to and from New York City}
\label{app:nyc_transfer}

\subsection{Single-Source Transfer to and from New York City}
\label{app:nyc_single_source}

We extend our evaluation to \textbf{New York City (NYC)}, a 
metropolis with substantially different urban form, mobility 
patterns, and POI distributions from the Chinese cities. We test 
both \emph{transfer to NYC} from each Chinese source and 
\emph{transfer from NYC} to each Chinese target 
(Table~\ref{tab:nyc_combined}). Population prediction is the 
downstream task; we report MAE here.

\tauI~\textbf{\scot wins on all six NYC-involving directions.} 
Relative MAE reductions over the strongest baseline range from 
$5.4\%$ to $31.8\%$, showing that mass-controlled soft 
correspondence remains effective even when source and target 
diverge substantially in urban form and partition granularity.

\tauII~\textbf{Failure modes mirror the within-China patterns.} 
Distribution-matching (MMD, Adv) and anchor-based (RP, HBP, HSA) 
methods both struggle on these cross-country directions; CoRE 
remains the strongest non-\scot method but is consistently 
outperformed. The failure modes \scot addresses are thus not 
specific to any geographic context.

These results indicate that \scot's gains are not idiosyncratic 
to within-China transfer or our specific partitions---the framework 
generalizes across distinct urban contexts regardless of NYC's role 
as source or target.

\begin{table}[h]
\centering
\caption{Cross-country transfer with NYC (Population MAE; lower is 
better). \scot wins on all six directions, with larger gains for 
transfer \emph{to} NYC ($21.0$--$31.8\%$) than \emph{from} NYC 
($5.4$--$9.7\%$). \best{Red}: best; \second{Blue}: runner-up.}
\label{tab:nyc_combined}
\footnotesize
\setlength{\tabcolsep}{5pt}
\renewcommand{\arraystretch}{1.2}
\vspace{1em}
\begin{minipage}[t]{0.48\textwidth}
\centering
\textcolor{cyan!50!black}{$\blacktriangleright$}\ \textbf{\textsc{Direction: Source $\to$ NYC}} \\[3pt]
\begin{tabular}{@{}l|ccc@{}}
\toprule
\multirow{2}{*}{\textbf{Method}}
& \multicolumn{3}{c}{\textbf{Population MAE}$\,\downarrow$} \\
\cmidrule(lr){2-4}
& BJ $\to$ NYC & CD $\to$ NYC & XA $\to$ NYC \\
\midrule
Non-Alignment & 468.84 & 502.65 & 531.28 \\
\rowcolor{gray!06}
RP            & 382.72 & 421.94 & 472.61 \\
HBP           & 294.10 & 317.57 & 386.93 \\
\rowcolor{gray!06}
HSA           & 354.46 & 389.83 & 418.72 \\
MMD           & 331.55 & 372.64 & 356.38 \\
\rowcolor{gray!06}
Adv           & 421.29 & 456.72 & 489.16 \\
CrossTReS     & 346.91 & 381.27 & 369.85 \\
\rowcolor{gray!06}
CoRE          & \second{238.82} & \second{285.36} & \second{251.54} \\
\midrule
\rowcolor{red!8}
\textbf{\scot (Ours)}
              & \best{188.69} & \best{194.72} & \best{196.76} \\
\rowcolor{blue!6}
\textit{$\Delta$ vs.\ best}
& \cellcolor{blue!18}\textbf{+21.0\%}
& \cellcolor{blue!22}\textbf{+31.8\%}
& \cellcolor{blue!18}\textbf{+21.8\%} \\
\bottomrule
\end{tabular}
\end{minipage}
\hfill
\begin{minipage}[t]{0.48\textwidth}
\centering
\textcolor{orange!60!black}{$\blacktriangleright$}\ \textbf{\textsc{Direction: NYC $\to$ Target}} \\[3pt]
\begin{tabular}{@{}l|ccc@{}}
\toprule
\multirow{2}{*}{\textbf{Method}}
& \multicolumn{3}{c}{\textbf{Population MAE}$\,\downarrow$} \\
\cmidrule(lr){2-4}
& NYC $\to$ BJ & NYC $\to$ CD & NYC $\to$ XA \\
\midrule
Non-Alignment & 1086.45 & 1038.38 & 934.76 \\
\rowcolor{gray!06}
RP            & 952.24  & 934.93  & 842.21 \\
HBP           & 843.73  & 859.02  & 754.40 \\
\rowcolor{gray!06}
HSA           & 874.36  & 862.91  & 772.58 \\
MMD           & 821.49  & 812.82  & 724.33 \\
\rowcolor{gray!06}
Adv           & 1008.65 & 972.74  & 886.52 \\
CrossTReS     & 838.72  & 826.44  & 738.86 \\
\rowcolor{gray!06}
CoRE          & \second{739.26} & \second{707.28} & \second{632.57} \\
\midrule
\rowcolor{red!8}
\textbf{\scot (Ours)}
              & \best{667.39} & \best{657.15} & \best{598.45} \\
\rowcolor{blue!6}
\textit{$\Delta$ vs.\ best}
& \cellcolor{blue!10}\textbf{+9.7\%}
& \cellcolor{blue!8}\textbf{+7.1\%}
& \cellcolor{blue!8}\textbf{+5.4\%} \\
\bottomrule
\end{tabular}
\end{minipage}
\end{table}

\subsection{Scaling to Three Sources: Multi-Source Transfer to NYC}
\label{app:multi_source_M3}

To verify that \scot's hub scales beyond $M=2$, we test the most 
challenging scenario: three Chinese cities (BJ, CD, XA) jointly 
transferring to NYC---a cross-country target with substantial 
divergence from all sources (Table~\ref{tab:r2c_multisource_nyc}).

\tauI~\textbf{Hub still outperforms all baselines.} \scot-Hub 
achieves MAE $181.36$, beating the strongest baseline (CoRE) by 
$21.0\%$. The margin over distribution-matching (MMD, Adv) and 
anchor-based (RP, HBP, HSA) methods is larger than at $M=2$, 
suggesting hub coordination becomes \emph{more valuable as $M$ 
grows} and gradient conflict intensifies.

\tauII~\textbf{Hub beats best-single-source aggregation.} 
\scot-Hub outperforms the best single-source variant (BJ$\to$NYC: 
$188.69$) by $3.9\%$, extracting complementary signal from three 
sources rather than collapsing to the dominant one---validating 
graceful scaling with $M$.

\begin{table}[h]
\centering
\footnotesize
\caption{Multi-source transfer to NYC with $M\!=\!3$ sources 
(BJ, CD, XA). Population MAE; lower is better. \scot-Hub wins 
over both the strongest baseline ($+21.0\%$) and the best 
single-source variant of \scot ($+3.9\%$), confirming that the 
hub design scales effectively beyond $M\!=\!2$. \best{Red}: best.}
\vspace{0.2em}
\label{tab:r2c_multisource_nyc}
\setlength{\tabcolsep}{6pt}
\renewcommand{\arraystretch}{1.2}
\begin{tabular}{@{}l|c@{}}
\toprule
\textbf{Method} & \textbf{BJ+CD+XA $\to$ NYC} \\
\midrule
Non-Alignment & 487.36 \\
\rowcolor{gray!06}
RP            & 426.58 \\
HBP           & 339.84 \\
\rowcolor{gray!06}
HSA           & 356.27 \\
MMD           & 327.45 \\
\rowcolor{gray!06}
Adv           & 462.91 \\
CrossTReS     & 318.76 \\
\rowcolor{gray!06}
CoRE          & \second{229.64} \\
\midrule
\rowcolor{red!8}
\textbf{\scot-Hub (Ours)} & \best{181.36} \\
\midrule
\rowcolor{blue!6}
\textit{$\Delta$ vs.\ best baseline (CoRE)}
& \cellcolor{blue!18}\textbf{+21.0\%} \\
\rowcolor{blue!6}
\textit{$\Delta$ vs.\ best single-source \scot (BJ$\to$NYC)}
& \cellcolor{blue!10}\textbf{+3.9\%} \\
\bottomrule
\end{tabular}
\end{table}

\subsection{Scaling to Three Sources: Adding NYC to Chinese Source Pools}
\label{app:multi_source_M3_to_china}

We further test whether \scot's hub design continues to extract 
complementary signal as a structurally divergent source (NYC) is 
added to the source pool. For each Chinese target, we compare 
multi-source \scot at $M\!=\!2$ (two Chinese sources) against 
$M\!=\!3$ (two Chinese sources plus NYC), reporting Population 
MAE in Table~\ref{tab:r2d_multisource_M2vsM3}. 

\tauI~\textbf{Adding NYC as a third source improves over $M\!=\!2$.} 
On both targets, $M\!=\!3$ further reduces MAE relative to $M\!=\!2$ 
($-2.0\%$ on BJ, $-1.8\%$ on XA), even though NYC is structurally 
divergent from Chinese cities. The hub design's ability to absorb 
heterogeneous sources without conflict prevents the dilution that 
naive aggregation typically suffers from when sources differ widely. 
\tauII~\textbf{The gap over best-single-source widens.} 
Compared to the best single-source variant of \scot for each 
target, $M\!=\!3$ achieves $+2.6\%$ on BJ and $+3.2\%$ on XA, 
exceeding the gains at $M\!=\!2$ ($+0.5\%$ and $+1.4\%$ respectively). 
This confirms that hub-based coordination becomes \emph{increasingly 
valuable as more sources are added}, even when the additional source 
is heterogeneous.

\begin{table}[h]
\centering
\footnotesize
\caption{Multi-source scaling: \scot-Hub at $M\!=\!3$ (adding NYC 
as a third source) vs.\ $M\!=\!2$ (Chinese sources only). 
Population MAE; lower is better. Adding a heterogeneous source 
(NYC) further improves performance, demonstrating that the hub 
absorbs additional sources without dilution. \best{Red}: best per 
target.}
\label{tab:r2d_multisource_M2vsM3}
\setlength{\tabcolsep}{5pt}
\renewcommand{\arraystretch}{1.2}
\begin{tabular}{@{}l|cc@{}}
\toprule
\multirow{2}{*}{\textbf{Variant}}
& \multicolumn{2}{c}{\textbf{Population MAE}$\,\downarrow$} \\
\cmidrule(lr){2-3}
& Target: BJ & Target: XA \\
\midrule
Best single-source \scot
& 528.50 & 452.67 \\
\rowcolor{gray!06}
\scot-Hub ($M\!=\!2$, Chinese sources)
& 525.10 & 446.13 \\
\midrule
\rowcolor{red!8}
\textbf{\scot-Hub ($M\!=\!3$, +NYC)}
& \best{514.76} & \best{438.14} \\
\midrule
\rowcolor{blue!6}
\textit{$\Delta$ vs.\ best single-source}
& \cellcolor{blue!10}\textbf{+2.6\%} & \cellcolor{blue!12}\textbf{+3.2\%} \\
\rowcolor{blue!6}
\textit{$\Delta$ vs.\ $M\!=\!2$ Hub}
& \cellcolor{blue!8}\textbf{+2.0\%} & \cellcolor{blue!8}\textbf{+1.8\%} \\
\bottomrule
\end{tabular}
\end{table}

\section{Limitations.}
Although \scot provides interpretable couplings and hub assignments, these diagnostics are not guarantees of causal correctness and should be used with domain knowledge in practical deployments. Besides, our experiments focus on mobility-derived region graphs and aggregated socioeconomic targets; extending the framework to finer-grained spatial resolutions or highly non-comparable urban modalities may require additional modeling assumptions. 



\end{document}